\documentclass{article}

\usepackage[final]{nips_2017}
\setcitestyle{round}
\usepackage{amsmath}
\allowdisplaybreaks
\usepackage{color}

\usepackage[utf8]{inputenc} 
\usepackage[T1]{fontenc}    
\usepackage{hyperref}       
\usepackage{url}            
\usepackage{booktabs}       
\usepackage{amsfonts}       
\usepackage{nicefrac}       
\usepackage{microtype}      
\usepackage{times}
\usepackage{graphicx} 
\usepackage{subfigure}

\usepackage{algorithm}
\usepackage{algorithmic}
\usepackage{amsmath}
\usepackage{amsthm}
\usepackage{amssymb}
\usepackage{amsthm}
\usepackage{comment}

\usepackage{parskip}
\usepackage{bbm}

\newtheorem{theorem}{Theorem}[section]
\newtheorem{lemma}[theorem]{Lemma}
\newtheorem{remark}[theorem]{Remark}
\newtheorem{definition}[theorem]{Definition}
\newtheorem{corollary}[theorem]{Corollary}
\newtheorem{proposition}[theorem]{Proposition}

\newcommand{\expectation}[1]{\mathbb{E}\left\lbrack #1 \right\rbrack}
\newcommand{\innerprod}[2]{\left\langle #1, #2 \right\rangle}
\newcommand{\takesign}[1]{\mathrm{sign}\left( #1 \right)}
\newcommand{\probability}[1]{\mathbb{P}\left[ #1 \right]}
\newcommand{\Var}[1]{\mathrm{Var}\left( #1 \right)}
\newcommand{\Cov}[2]{\mathrm{Cov}\left( #1 , #2 \right)}
\newcommand{\calcd}{\mathrm{d}}
\usepackage{multirow}
\usepackage{verbatim}

\usepackage{enumitem}

\title{The Unreasonable Effectiveness of Structured Random Orthogonal Embeddings}

\author{
  Krzysztof Choromanski \thanks{equal contribution} \\
  Google Brain Robotics\\
  \texttt{kchoro@google.com} \\ \And
  Mark Rowland \footnotemark[1] \\
  University of Cambridge\\
  \texttt{mr504@cam.ac.uk} \\ \And
  Adrian Weller \\
  University of Cambridge and Alan Turing Institute\\
  \texttt{aw665@cam.ac.uk} \\    
}

\begin{document}

\maketitle

\begin{abstract}
		We examine a class of embeddings based on structured random matrices with orthogonal 
		rows which can be applied in many machine learning applications including dimensionality
		reduction and kernel approximation.
		For both the Johnson-Lindenstrauss transform and the angular kernel, we show that we can select matrices yielding guaranteed improved performance in accuracy and/or speed compared to earlier methods.
We introduce matrices with complex entries which give significant further accuracy improvement. 
We provide geometric and Markov chain-based perspectives to help understand the benefits, and empirical results which suggest that the approach is helpful in a wider range of applications.  
\end{abstract}

	\section{Introduction}
\label{sec:shortintro}

Embedding methods play a central role in 
many machine learning applications by 
projecting   
feature vectors into a new space (often nonlinearly),  
allowing the original task to be solved more efficiently. 
The new space might have more or fewer dimensions depending on the goal. 
Applications include the Johnson-Lindenstrauss Transform for dimensionality reduction (JLT, \citealp{johnson84extensionslipschitz}) and kernel methods with random feature maps \citep{rahimi}. 
The embedding can be costly hence many fast methods have been developed, see \S \ref{sec:related} for background and related work.

We present a general class of random embeddings based on particular structured random matrices with orthogonal rows, which we call \textit{random ortho-matrices} (ROMs); see \S \ref{sec:model}. We show that ROMs may be used for the applications above, in each case demonstrating improvements over previous 
methods in statistical accuracy (measured by mean squared error, MSE), 
in computational efficiency (while providing similar accuracy), or both. We highlight the following contributions:
\begin{itemize}[leftmargin=*]
\item In \S \ref{sec:ojlt}: The \textit{Orthogonal Johnson-Lindenstrauss Transform} (OJLT) for dimensionality reduction. We prove this has strictly smaller MSE than the previous unstructured JLT mechanisms. Further, OJLT is as fast as the fastest previous JLT variants (which are structured).
\item In \S \ref{sec:kernels}: Estimators for the \emph{angular kernel} \citep{soft_cosine} 
which guarantee better MSE. 
The \emph{angular kernel} is important for many applications, including natural language processing \citep{soft_cosine}, image analysis \citep{Jeg11}, speaker representations \citep{Schm14} and tf-idf data sets \citep{Sun13}.
\item In \S \ref{sec:why}: Two perspectives on the effectiveness of ROMs to help build intuitive understanding.
\end{itemize}

In \S \ref{sec:experiments} we provide empirical results which support our analysis, and show that ROMs are effective for a still broader set of applications. Full details and proofs of all results are in the Appendix.

\subsection{Background and related work}\label{sec:related}

Our ROMs can have two forms (see \S \ref{sec:model} for details): (i) a $\mathbf{G}_{\mathrm{ort}}$ is a random Gaussian matrix conditioned on rows being orthogonal; or (ii) an $\mathbf{SD}$-\emph{product} matrix is formed by multiplying some number $k$ of $\mathbf{SD}$ blocks, each of which is highly structured, typically leading to fast computation of products. Here $\mathbf{S}$ is a particular structured matrix, and $\mathbf{D}$ is a random diagonal matrix; see \S 2 for full details. Our $\mathbf{SD}$ block generalizes an $\mathbf{HD}$ block, where $\mathbf{H}$ is a \emph{Hadamard} matrix, which received previous attention. Earlier approaches to embeddings have explored using various structured matrices, including particular versions of one or other of our two forms, though in different contexts.

For dimensionality reduction, \cite{ailon2006approximate} used a single $\mathbf{HD}$ block as a way to spread out the mass of a vector over all dimensions before applying a sparse Gaussian matrix. \cite{chor_sind_2016} also used just one $\mathbf{HD}$ block as part of a larger structure. \citet{choromanski_tri} discussed using $k=3$ $\mathbf{HD}$ blocks for locality-sensitive hashing methods but gave no concrete results for their application to dimensionality reduction or kernel approximation.
All these works, and other earlier approaches \citep{hinrichs2011johnson, vybiral2011variant, zhang2013new,le2013fastfood,choromanska2015binary}, provided computational benefits by using structured matrices with less randomness than unstructured iid Gaussian matrices, but none demonstrated accuracy gains.

\citet{choromanski_ort} were the first to show that $\mathbf{G}_{\mathrm{ort}}$-type matrices can yield improved accuracy, but their theoretical result applies only asymptotically for many dimensions, only for the Gaussian kernel and for just one specific orthogonal transformation, which is one instance of the larger class 
we consider. Their theoretical result does not yield computational benefits. \citet{choromanski_ort} did explore using a number $k$ of $\mathbf{HD}$ blocks empirically, observing good computational and statistical performance for $k=3$, 
but without any theoretical accuracy guarantees. It was left as an open question why 
matrices formed by a small number of $\mathbf{HD}$ blocks can outperform non-discrete transforms.

In contrast, we are able to prove that ROMs yield improved MSE in several settings and for many of them for any number of dimensions. In addition, $\mathbf{SD}$-product matrices can deliver computational speed benefits. We provide initial analysis to understand why $k=3$ can outperform the state-of-the-art, why odd $k$ yields better results than even $k$, and why higher values of $k$ deliver decreasing additional benefits (see \S \ref{sec:ojlt} and \S \ref{sec:why}). 

	\section{The family of Random Ortho-Matrices (ROMs)}
\label{sec:model}

Random ortho-matrices (ROMs) are taken from two main classes of distributions defined below that require the rows of sampled matrices to be orthogonal. A central theme of the paper is that this orthogonal structure can yield improved statistical performance.
We shall use bold uppercase (e.g. $\mathbf{M}$) to denote matrices and bold lowercase (e.g. $\mathbf{x}$) for vectors.

\textbf{Gaussian orthogonal matrices.} Let $\mathbf{G}$ be a random matrix taking values in $\mathbb{R}^{m \times n}$ with iid $\mathcal{N}(0,1)$ elements, which we refer to as an \emph{unstructured} Gaussian matrix. The first ROM distribution we consider yields the random matrix $\mathbf{G}_{\mathrm{ort}}$, which is defined as a random $\mathbb{R}^{n \times n}$ matrix given by first taking the rows of the matrix to be a uniformly random orthonormal basis, and then independently scaling each row, so that the rows marginally have multivariate Gaussian $\mathcal{N}(0, I)$ distributions. The random variable $\mathbf{G}_{\mathrm{ort}}$ can then be extended to non-square matrices by either stacking independent copies of the $\mathbb{R}^{n \times n}$ random matrices, and deleting superfluous rows if necessary. The orthogonality of the rows of this matrix has been observed to yield improved statistical properties for randomized algorithms built from the matrix in a variety of applications.

\textbf{$\mathbf{SD}$-product matrices.} 
Our second class of distributions is motivated by the desire to obtain similar statistical benefits of orthogonality to $\mathbf{G}_{\mathrm{ort}}$, whilst gaining computational efficiency by employing more structured matrices. 
We call this second class $\mathbf{SD}$-\emph{product} matrices. These take the more structured form $\prod_{i=1}^k \mathbf{S} \mathbf{D}_i$,
where $\mathbf{S}=\{s_{i,j}\} \in \mathbb{R}^{n \times n}$ has orthogonal rows, $|s_{i,j}| = \frac{1}{\sqrt{n}} \; \forall i,j \!\in\! \{1,\dots,n\}$; and the $(\mathbf{D}_{i})_{i=1}^k$ are independent diagonal matrices described below.  
By $\prod_{i=1}^k \mathbf{S} \mathbf{D}_i$, we mean the matrix product $(\mathbf{S} \mathbf{D}_k)\dots (\mathbf{S} \mathbf{D}_1)$. 
This class includes as particular cases several recently introduced random matrices (e.g. \citealp{indyk2015,choromanski_ort}), where good \emph{empirical} performance was observed. We go further to establish strong theoretical guarantees, 
see \S \ref{sec:ojlt} and \S \ref{sec:kernels}.

A prominent example of an $\mathbf{S}$ matrix is the normalized \emph{Hadamard} matrix $\mathbf{H}$, defined recursively by $\mathbf{H}_{1}=(1)$, and then for $i > 1$,
$\mathbf{H}_{i} = \frac{1}{\sqrt{2}} 
\begin{pmatrix} 
\mathbf{H}_{i-1} & \mathbf{H}_{i-1} \\
\mathbf{H}_{i-1} & -\mathbf{H}_{i-1} 
\end{pmatrix}.$  
Importantly, 
matrix-vector products with $\mathbf{H}$ 
are computable in $O(n \log n)$ time via the fast Walsh-Hadamard transform, yielding large computational savings. In addition, $\mathbf{H}$ matrices enable a significant space advantage:  since the fast Walsh-Hadamard transform can be computed without explicitly storing $\mathbf{H}$, only $O(n)$ space is required to store the diagonal elements of $(\mathbf{D}_i)_{i=1}^k$. 
Note that these $\textbf{H}_n$ matrices are defined only for $n$ a power of 2, 
but if needed, one can always adjust data by padding with $0$s to enable the use of `the next larger' $\textbf{H}$, doubling the number of dimensions in the worst case. 

Matrices $\mathbf{H}$ are representatives of a much larger family in $\mathbf{S}$ which also attains computational savings. 
These are $L_{2}$-normalized versions of Kronecker-product matrices of the form $\mathbf{A}_{1} \otimes ... \otimes \mathbf{A}_{l} \in \mathbb{R}^{n \times n}$ 
for $l \in \mathbb{N}$, where $\otimes$ stands for a Kronecker product and blocks $\mathbf{A}_{i} \in \mathbb{R}^{d \times d}$ have entries of the same magnitude and 
pairwise orthogonal rows each. For these matrices, matrix-vector products are computable in $O(n(2d-1)\log_{d}(n))$ time \citep{fastkronecker}.

$\textbf{S}$ includes also the \textit{Walsh matrices}
$\mathbf{W} = \{w_{i,j}\} \in \mathbb{R}^{n \times n}$, 
where $w_{i,j} = \frac{1}{\sqrt{n}} (-1)^{i_{N-1}j_{0}+...+i_{0}j_{N-1}}$ and $i_{N-1}...i_{0}$,
$j_{N-1}...j_{0}$ are binary representations of $i$ and $j$ respectively.

For diagonal $(\mathbf{D}_i)_{i=1}^k$, we mainly consider Rademacher entries leading to the following 
matrices. \\ 

\begin{definition}\label{def:ROMs}
The $\mathbf{S}$-\textit{Rademacher} random matrix with $k \in \mathbb{N}$ blocks is below, where $(\mathbf{D}^{(\mathcal{R})}_i)_{i=1}^k$ are diagonal 
with iid Rademacher random variables [i.e. $\mathrm{Unif}(\{\pm 1\})$] 
on the diagonals:
\begin{equation}
\mathbf{M}_\mathrm{\mathbf{S}\mathcal{R}}^{(k)} = \prod_{i=1}^k \mathbf{S}\mathbf{D}_i^{(\mathcal{R})} \, .
\end{equation}
\end{definition}
Having established the two classes of ROMs, we next apply them to dimensionality reduction. 

	\section{The Orthogonal Johnson-Lindenstrauss Transform (OJLT)}\label{sec:ojlt}

Let $\mathcal{X} \subset \mathbb{R}^n$ be a dataset of $n$-dimensional real vectors.
The goal of dimensionality reduction via random projections is to transform linearly each $\mathbf{x} \in \mathcal{X}$ by a random mapping 
$\mathbf{x} \overset{F}{\mapsto} \mathbf{x}^{\prime}$, 
where: $F: \mathbb{R}^{n} \rightarrow \mathbb{R}^{m}$ for $m<n$, such that for any $\mathbf{x},\mathbf{y} \in \mathcal{X}$ the following holds: 
$(\mathbf{x}^{\prime})^\top\mathbf{y}^{\prime} \approx \mathbf{x}^\top \mathbf{y}$. If we furthermore have
$\mathbb{E}[(\mathbf{x}^{\prime})^\top\mathbf{y}^{\prime}] = \mathbf{x}^\top \mathbf{y}$ then the dot-product estimator is \emph{unbiased}.
In particular, this dimensionality reduction mechanism should in expectation preserve information
about vectors' norms, i.e. we should have: $\mathbb{E}[\|\mathbf{x}^{\prime}\|_{2}^{2}] = \|\mathbf{x}\|_{2}^{2}$ for any $\mathbf{x} \in \mathcal{X}$.

The standard JLT mechanism uses the randomized linear map 
$F = \frac{1}{\sqrt{m}}\mathbf{G}$, where $\mathbf{G} \in \mathbb{R}^{m \times n}$ is as in \S \ref{sec:model}, 
requiring $mn$ multiplications to evaluate.
Several fast variants (FJLTs) have been proposed by replacing $\mathbf{G}$ with random structured matrices, such as sparse or circulant Gaussian matrices \citep{ailon2006approximate, hinrichs2011johnson, vybiral2011variant, zhang2013new}. The fastest of these variants has $O(n\log n)$ time complexity, but at a cost of 
higher MSE for dot-products.

Our Orthogonal Johnson-Lindenstrauss Transform (OJLT) is obtained by replacing the unstructured random matrix $\mathbf{G}$ with a sub-sampled ROM 
from \S \ref{sec:model}: either $\mathbf{G}_{\mathrm{ort}}$, or a sub-sampled version $\mathbf{M}_{\mathrm{\mathbf{S}\mathcal{R}}}^{(k), \mathrm{sub}}$ of the $\mathbf{S}$-Rademacher ROM, given by sub-sampling rows from the left-most $\mathbf{S}$ matrix in the product. We sub-sample since $m<n$. We typically assume 
uniform sub-sampling \textit{without} replacement. 
The resulting dot-product estimators for 
vectors $\mathbf{x}, \mathbf{y} \in \mathcal{X}$ are given by:
\begin{align}
\widehat{K}&^{\mathrm{base}}_m(\mathbf{x}, \mathbf{y})  = \frac{1}{m}(\mathbf{G}\mathbf{x})^\top (\mathbf{G}\mathbf{y}) \quad \text{[unstructured iid baseline, previous state-of-the-art accuracy]}, \notag\\ 
&\widehat{K}^{\mathrm{ort}}_m(\mathbf{x}, \mathbf{y})  = \frac{1}{m}(\mathbf{G}_{\mathrm{ort}}\mathbf{x})^\top (\mathbf{G}_{\mathrm{ort}}\mathbf{y}), \qquad 
\widehat{K}^{(k)}_m(\mathbf{x}, \mathbf{y})  = \frac{1}{m}\left(\mathbf{M}_{\mathrm{\mathbf{S}\mathcal{R}}}^{(k),\mathrm{sub}}\mathbf{x}\right)^\top \left(\mathbf{M}_{\mathrm{\mathbf{S}\mathcal{R}}}^{(k),\mathrm{sub}}\mathbf{y}\right). 
\end{align}
We contribute the following closed-form expressions, which exactly quantify the mean-squared error (MSE) for these three estimators. Precisely, the MSE of an estimator $\widehat{K}(\mathbf{x}, \mathbf{y})$ of the inner product $\langle \mathbf{x}, \mathbf{y} \rangle$ for $\mathbf{x}, \mathbf{y} \in \mathcal{X}$ is defined to be $\mathrm{MSE}(\widehat{K}(\mathbf{x}, \mathbf{y})) = \mathbb{E}\left\lbrack (\widehat{K}(\mathbf{x}, \mathbf{y}) - \langle \mathbf{x}, \mathbf{y} \rangle^2) \right\rbrack$. See the Appendix for detailed proofs of these results and all others in this paper.

\begin{lemma}
	\label{easy_lemma}
	The MSE of the unstructured JLT dot-product estimator $\widehat{K}^{\mathrm{base}}_{m}$ of $\mathbf{x},\mathbf{y} \in \mathbb{R}^{n}$ using $m$-dimensional random feature maps is unbiased, with
	$\mathrm{MSE}(\widehat{K}^{\mathrm{base}}_{m}(\mathbf{x},\mathbf{y})) = \frac{1}{m} ((\mathbf{x} ^\top \mathbf{y})^2 + \|\mathbf{x}\|_2^2 \|\mathbf{y}\|_2^2).$
\end{lemma}

\begin{theorem}
	\label{gaussian_ojlt}
	The estimator $\widehat{K}^{\mathrm{ort}}_{m}$ is unbiased and satisfies, for $n\geq 4$:
	\begin{align}
	&\mathrm{MSE}(\widehat{K}^{\mathrm{ort}}_{m}(\mathbf{x},\mathbf{y})) \nonumber \\
	=&  \mathrm{MSE}(\widehat{K}^{\mathrm{base}}_{m}(\mathbf{x},\mathbf{y})) \; +  \nonumber \\ & \frac{m-1}{m}\Bigg\lbrack\frac{\|\mathbf{x}\|_2^2 \|\mathbf{y}\|_2^2 n(n - 2)}{2} \Bigg( \frac{1}{(n + 2)(n-1)} \left\lbrack \cos^2(\theta) + \frac{1}{2} \right\rbrack
	+ \nonumber \\
	& \qquad\qquad\qquad\qquad\qquad\qquad \frac{1}{(n-1)(n-2)} \left\lbrack \cos^2(\theta) - \frac{1}{2} \right\rbrack \Bigg) - \langle \mathbf{x}, \mathbf{y} \rangle^2  \Bigg\rbrack \, ,
	\end{align}
	where $\theta$ is the angle between $\mathbf{x}$ and $\mathbf{y}$.
\end{theorem}

\begin{theorem}[Key result]
	\label{hd_theorem}
	The OJLT estimator $\widehat{K}^{(k)}_m(\mathbf{x}, \mathbf{y})$ with $k$ blocks, using $m$-dimensional random feature maps and 
	uniform sub-sampling policy without replacement, is unbiased with
	\begin{align}\label{eq:hd_mse_formula}
	\mathrm{MSE}(\widehat{K}^{(k)}_m(\mathbf{x}, \mathbf{y})) \!=\! \frac{1}{m}\!\left(\frac{n-m}{n-1}\right)\! \bigg(&((\mathbf{x}^\top \mathbf{y})^2 + 
	\|\mathbf{x}\|^2\|\mathbf{y}\|^2) \; + \\ 
	&\sum_{r=1}^{k-1}  \frac{(-1)^r 2^{r}}{n^r} (2(\mathbf{x}^\top \mathbf{y})^2  +  \|\mathbf{x}\|^2\|\mathbf{y}\|^2 )  +  \frac{(-1)^{k}2^k}{n^{k-1}}  \sum_{i=1}^n x_i^2 y_i^2  \bigg). \nonumber 
	\end{align}
\end{theorem}
\begin{proof}[Proof (Sketch)]
	For $k=1$, the random projection matrix is given by sub-sampling rows from $\mathbf{SD}_1$, and the computation can be carried out directly. For $k \geq 1$, the proof proceeds by induction. The random projection matrix in the general case is given by sub-sampling rows of the matrix $\mathbf{SD}_k \cdots \mathbf{SD}_1$. By writing the MSE as an expectation and using the law of conditional expectations conditioning on the value of the first $k-1$ random matrices $\mathbf{D}_{k-1},\ldots, \mathbf{D}_1$, the statement of the theorem for $1$ $\mathbf{SD}$ block and for $k-1$ $\mathbf{SD}$ blocks can be neatly combined to yield the result.
\end{proof}

To our knowledge, it has not previously been possible to provide theoretical guarantees that $\mathbf{SD}$-product matrices outperform iid matrices.
Combining Lemma \ref{easy_lemma} with Theorem \ref{hd_theorem} yields the following important result.\\
\begin{corollary}[Theoretical guarantee of improved performance]\label{cor:key}
Estimators $\widehat{K}_m^{(k)}$ (subsampling without replacement) yield guaranteed lower MSE 
than $\widehat{K}_m^\mathrm{base}$.
\end{corollary}

It is not yet clear when $\widehat{K}_m^{\mathrm{ort}}$ is better or worse than $\widehat{K}_m^{(k)}$; we explore this empirically in \S \ref{sec:experiments}. 
Theorem \ref{hd_theorem} shows that there are diminishing MSE benefits to using a large number $k$ of $\mathbf{SD}$ blocks. 
Interestingly, odd $k$ is better than even: it is easy to observe that 
$\mathrm{MSE}(\widehat{K}^{(2k-1)}_{m}(\mathbf{x},\mathbf{y})) < \mathrm{MSE}(\widehat{K}^{(2k)}_{m}(\mathbf{x},\mathbf{y})) >
\mathrm{MSE}(\widehat{K}^{(2k+1)}_{m}(\mathbf{x},\mathbf{y}))$. These observations, and those in \S \ref{sec:why}, help to understand why empirically $k=3$ was previously observed to work well \citep{choromanski_ort}.

If we take $\mathbf{S}$ to be a normalized Hadamard matrix $\mathbf{H}$, then even though we are using sub-sampling, and hence the full computational benefits of the Walsh-Hadamard transform are not available, still $\widehat{K}_m^{(k)}$ achieves improved MSE compared to the base method with \textit{less} computational effort, as follows. \\

\begin{lemma}
	\label{complexity_lemma}
	There exists an algorithm (see Appendix for details)
	which computes an embedding for a given datapoint $\mathbf{x}$ using $\widehat{K}_m^{(k)}$ with $\mathbf{S}$ set to $\mathbf{H}$ and
	uniform sub-sampling policy 
	in expected time $\min\{O((k-1)n\log(n)+nm-\frac{(m-1)m}{2}, kn\log(n)\}$. 
\end{lemma}

Note that for $m=\omega(k\log(n))$ or if $k=1$, the time complexity is smaller than the brute force 
$\Theta(nm)$. 
The algorithm uses a simple observation that one can reuse calculations conducted for the upper half of the Hadamard matrix while 
performing computations involving rows from its other half, instead of running these calculations from scratch (details in the Appendix). 
 
An alternative to sampling without replacement is deterministically to choose the first $m$ rows. In our experiments in \S \ref{sec:experiments}, these two approaches yield the same empirical performance, though we expect that the deterministic method could perform poorly on adversarially chosen data. 
The first $m$ rows approach can be realized in time $O(n\log(m)+(k-1)n\log(n))$ per datapoint.

Theorem \ref{hd_theorem} is a key result in this paper, demonstrating that $\mathbf{SD}$-product matrices yield both statistical and computational improvements compared to the base iid procedure, which is widely used in practice.
We next show how to obtain further gains in accuracy.

\subsection{Complex variants of the OJLT}
 
We show that the MSE benefits of Theorem \ref{hd_theorem} may be markedly improved 
by using $\mathbf{SD}$-product matrices with complex entries $\mathbf{M}_\mathrm{\mathbf{S}\mathcal{H}}^{(k)}$. Specifically, we consider the variant 
$\mathbf{S}$-\textit{Hybrid} random matrix below, 
where $\mathbf{D}_k^{(\mathcal{U})}$ is a diagonal matrix with iid $\mathrm{Unif}(S^1)$ random variables on the diagonal, independent of $(\mathbf{D}_i^{(\mathcal{R})})_{i=1}^{k-1}$, and $S^1$ is the unit circle of $\mathbb{C}$. 
We use the real part of the Hermitian product between projections as a dot-product estimator; recalling the definitions of \S2, we use: 
\begin{equation}
\mathbf{M}_\mathrm{\mathbf{S}\mathcal{H}}^{(k)} = \mathbf{S}\mathbf{D}_k^{(\mathcal{U})} \prod_{i=1}^{k-1} \mathbf{S}\mathbf{D}_i^{(\mathcal{R})} \, ,
\qquad
\widehat{K}^{\mathcal{H}, (k)}_m (\mathbf{x}, \mathbf{y}) = \frac{1}{m}\: \mathrm{Re}\left\lbrack\left(\overline{\mathbf{M}_{\mathrm{\mathbf{S}\mathcal{H}}}^{(k),\mathrm{sub}}\mathbf{x}}\right)^\top  \left(\mathbf{M}_{\mathrm{\mathbf{S}\mathcal{H}}}^{(k),\mathrm{sub}}\mathbf{y}\right) \right\rbrack. \label{eq:hybrid-estimator}
\end{equation}

Remarkably, this complex variant yields exactly half the MSE of the OJLT estimator. \\

\begin{theorem}\label{thm:hybrid-mse}
The estimator $\widehat{K}^{\mathcal{H}, (k)}_m(\mathbf{x},\mathbf{y})$, applying uniform sub-sampling without replacement, is unbiased and satisfies:
$\mathrm{MSE}(\widehat{K}^{\mathcal{H}, (k)}_m(\mathbf{x}, \mathbf{y})) = \frac{1}{2} \mathrm{MSE}(\widehat{K}^{(k)}_m(\mathbf{x}, \mathbf{y}))$. 
\end{theorem}

This large factor of $2$ improvement 
could instead be obtained by doubling $m$ for $\widehat{K}_m^{(k)}$. However, this would require 
doubling the number of parameters for the transform, whereas the $\mathbf{S}$-Hybrid estimator requires additional storage only for the complex parameters in the matrix $\mathbf{D}^\mathcal{(U)}_k$. 
Strikingly, 
it is straightforward to extend the proof of Theorem \ref{thm:hybrid-mse} (see Appendix) to show that rather than taking the complex random variables in $\mathbf{M}_{\mathbf{S}\mathcal{H}}^{(k), \mathrm{sub}}$ to be $\mathrm{Unif}(S^1)$, it is possible to take them to be $\mathrm{Unif}(\{1, -1, i, -i\})$ and still obtain exactly the same benefit in MSE.\\

\begin{theorem}\label{corr:4-pt-complex}
For the estimator $\widehat{K}_m^{\mathcal{H}, (k)}$ defined in Equation \eqref{eq:hybrid-estimator}: replacing the random matrix $\mathbf{D}^{(\mathcal{U})}_k$ (which has iid $\mathrm{Unif}(S^1)$ elements on the diagonal) 
with instead a random diagonal matrix having iid $\mathrm{Unif}(\{1,-1,i,-i\})$ elements on the diagonal, does not affect the MSE of the estimator.
\end{theorem}

It is natural to wonder if using an $\mathbf{SD}$-product matrix with more complex random variables (for all $\mathbf{SD}$ blocks) would improve performance still further. However, interestingly, this appears not to be the case; 
details are provided in the Appendix \S \ref{sec:proof-full-complex}.

\subsection{Sub-sampling with replacement}
Our results above focus on $\mathbf{SD}$-product matrices where rows have been sub-sampled 
without replacement. 
Sometimes (e.g. for parallelization) it can be convenient instead to sub-sample \emph{with} replacement. 
As might be expected, this leads to worse MSE, which we can quantify precisely. \\ 

\begin{theorem}\label{thm:no-replacement-subsampling}
For each of the estimators $\widehat{K}^{(k)}_m$ and $\widehat{K}^{\mathcal{H}, (k)}_m$, if uniform sub-sampling \emph{with} (rather than without) replacement is used 
then the MSE 
is worsened by a multiplicative constant of $\frac{n-1}{n-m}$.
\end{theorem}

	\section{Kernel methods with ROMs}\label{sec:kernels}

ROMs can also be used to construct high-quality random feature maps for non-linear kernel approximation. We analyze here the \emph{angular kernel}, an important example of a 
\textit{Pointwise Nonlinear Gaussian kernel} (PNG), discussed in more detail at the end of this section. \\

\begin{definition}
	The angular kernel $K^\mathrm{ang}$ is defined on $\mathbb{R}^n$ by $K^\mathrm{ang}(\mathbf{x}, \mathbf{y}) = 1 - \frac{2\theta_{\mathbf{x},\mathbf{y}}}{\pi}$, where $\theta_{\mathbf{x},\mathbf{y}}$ is the angle between $\mathbf{x}$ and $\mathbf{y}$.
\end{definition}

To employ random feature style approximations to this kernel, we first observe it may be rewritten as
\[
K^\mathrm{ang}(\mathbf{x}, \mathbf{y}) = \mathbb{E}\left\lbrack \mathrm{sign}(\mathbf{G} \mathbf{x}) \mathrm{sign}(\mathbf{G} \mathbf{y}) \right\rbrack \, ,
\]
where $\mathbf{G} \in \mathbb{R}^{1 \times n}$ is an unstructured isotropic Gaussian vector. This motivates approximations of the form:
\begin{equation}\label{eq:png-approx}
\widehat{K}^\mathrm{ang}{m}(\mathbf{x},\mathbf{y}) = \frac{1}{m}\mathrm{sign}(\mathbf{Mx})^{\top}\mathrm{sign}(\mathbf{My}),
\end{equation}
where $\mathbf{M} \in \mathbb{R}^{m \times n}$ is a random matrix, and the $\mathrm{sign}$ function is applied coordinate-wise. Such kernel estimation procedures are heavily used in practice \citep{rahimi}, as they allow fast approximate linear methods to be used \citep{Joachims2006} for inference tasks. If $\mathbf{M} = \mathbf{G}$, the unstructured Gaussian matrix, then we obtain the standard random feature estimator. We shall contrast this approach against the use of matrices from the ROMs family.

When constructing random feature maps for kernels, very often $m > n$. In this case, our structured mechanism can be applied by concatenating some number of independent structured blocks. Our theoretical guarantees will be given just for one block, but can easily be extended to a  larger number of blocks  since different blocks are independent.

The standard random feature approximation $\widehat{K}^{\mathrm{ang},\mathrm{base}}_{m}$ for approximating the angular kernel is defined by taking $\mathbf{M}$ to be $\mathbf{G}$, the unstructured Gaussian matrix, 
in Equation \eqref{eq:png-approx}, and satisfies the following.\\
\begin{lemma}
	\label{simple_lemma}
	The estimator $\widehat{K}^{\mathrm{ang},\mathrm{base}}_{m}$ is unbiased and 
	$\mathrm{MSE}(\widehat{K}^{\mathrm{ang},\mathrm{base}}_{m}(\mathbf{x},\mathbf{y})) = \frac{4\theta_{\mathbf{x},\mathbf{y}}(\pi-\theta_{\mathbf{x},\mathbf{y}})}{m \pi^{2}}$. 
\end{lemma}

The MSE of an estimator $\widehat{K}^{\mathrm{ang}}(\mathbf{x}, \mathbf{y})$ of the true angular kernel $K^\mathrm{ang}(\mathbf{x}, \mathbf{y})$ is defined analogously to the MSE of an estimator of the dot product, given in \S \ref{sec:ojlt}.
Our main result regarding angular kernels states that if we instead take $\mathbf{M} = \mathbf{G}_\mathrm{ort}$ in Equation \eqref{eq:png-approx}, then we obtain an estimator $\widehat{K}^{\mathrm{ang},\mathrm{ort}}_{m}$ with strictly smaller MSE, as follows. \\

\begin{theorem}
	\label{angle_theorem}
	Estimator $\widehat{K}^{\mathrm{ang},\mathrm{ort}}_{m}$ is unbiased and satisfies:
	$$\mathrm{MSE}(\widehat{K}^{\mathrm{ang},\mathrm{ort}}_{m}(\mathbf{x},\mathbf{y})) < \mathrm{MSE}(\widehat{K}^{\mathrm{ang},\mathrm{base}}_{m}(\mathbf{x},\mathbf{y})).$$
\end{theorem}

We also derive a formula for the MSE of an estimator $\widehat{K}_{m}^{\mathrm{ang},\mathbf{M}}$ of the angular kernel which replaces $\mathbf{G}$ with an arbitrary random matrix $\mathbf{M}$ and uses $m$ random feature maps. The formula is helpful to see how the quality of the estimator depends on the probabilities that the projections of the rows of $\mathbf{M}$ are contained in some particular convex regions of the $2$-dimensional space $\mathcal{L}_{\mathbf{x},\mathbf{y}}$ spanned by datapoints $\mathbf{x}$ and $\mathbf{y}$. For an illustration of the geometric definitions introduced in this Section, see 
Figure \ref{fig:intuition}.
The formula depends on probabilities involving events $\mathcal{A}^{i} = 
\{\mathrm{sgn}((\mathbf{r}^{i})^{T}\mathbf{x}) \neq \mathrm{sgn}((\mathbf{r}^{i})^{T}\mathbf{y})\}$,
where $\mathbf{r}^{i}$ stands for the $i^{th}$ row of the structured matrix.
Notice that $\mathcal{A}^{i} = \{\mathbf{r}^{i}_{proj} \in \mathcal{C}_{\mathbf{x},\mathbf{y}}\}$, where
$\mathbf{r}^{i}_{proj}$ stands for the projection of $\mathbf{r}^{i}$ into $\mathcal{L}_{\mathbf{x},\mathbf{y}}$ and $\mathcal{C}_{\mathbf{x},\mathbf{y}}$ is the union of two cones in $\mathcal{L}_{\mathbf{x},\mathbf{y}}$, each of angle $\theta_{\mathbf{x},\mathbf{y}}$. \\

\begin{theorem}
	\label{gen_theorem}
	Estimator $\widehat{K}_{m}^{\mathrm{ang},\mathbf{M}}$ satisfies the following, where: $\delta_{i,j} = \mathbb{P}[\mathcal{A}^{i} \cap \mathcal{A}^{j}] - \mathbb{P}[\mathcal{A}^{i}]\mathbb{P}[\mathcal{A}^{j}]$:
	\begin{align}
	\mathrm{MSE}(\widehat{K}^{\mathrm{ang}, \mathbf{M}}_{m}(\mathbf{x},\mathbf{y})) = \frac{1}{m^{2}}\left[m - \sum_{i=1}^{m}(1-2\mathbb{P}[\mathcal{A}^{i}])^{2}\right] 
	+ \frac{4}{m^{2}}\left[\sum_{i=1}^{m}(\mathbb{P}[\mathcal{A}^{i}]-\frac{\theta_{\mathbf{x},\mathbf{y}}}{\pi})^{2}+\sum_{i \neq j}\delta_{i,j}\right]. \notag 
	\end{align}
\end{theorem}
Note that probabilities $\mathbb{P}[\mathcal{A}^{i}]$ and $\delta_{i,j}$ depend on the choice of $\mathbf{M}$.
It is easy to prove that for unstructured $\mathbf{G}$ and $\mathbf{G}_{\mathrm{ort}}$ we have: 
$\mathbb{P}[\mathcal{A}^{i}] = \frac{\theta_{\mathbf{x},\mathbf{y}}}{\pi}$. Further, from the independence of the rows of $\mathbf{G}$, $\delta_{i,j} = 0$ for $i \neq j$.  
For unstructured $\mathbf{G}$ we obtain 
Lemma \ref{simple_lemma}. 
Interestingly, we see that to prove Theorem \ref{angle_theorem}, it suffices to show $\delta_{i,j} < 0$, which is the approach we take (see Appendix). 
If we replace $\mathbf{G}$ with $\mathbf{M}^{(k)}_{\mathbf{S}\mathcal{R}}$, then the expression $\epsilon=\mathbb{P}[\mathcal{A}^{i}]-\frac{\theta_{\mathbf{x},\mathbf{y}}}{\pi}$ does not depend on $i$. 
Hence, the angular kernel estimator based on Hadamard matrices gives smaller MSE estimator if and only if $\sum_{i \neq j} \delta_{i,j} + m\epsilon^{2} < 0$. It is not yet clear if this holds in general. 

As alluded to at the beginning of this section, the angular kernel may be viewed as a member of a wie family of kernels known as Pointwise Nonlinear Gaussian kernels. \\

\begin{definition}
	For a given function $f$, the Pointwise Nonlinear Gaussian kernel (PNG) $K^{f}$
	is defined by 
	$K^{f}(\mathbf{x},\mathbf{y}) =
	\mathbb{E}\left[f(\mathbf{g}^{T}\mathbf{x})f(\mathbf{g}^{T}\mathbf{y})\right]$, 
	where $\mathbf{g}$ is a Gaussian vector with i.i.d $\mathcal{N}(0,1)$ entries.
\end{definition}

Many prominent examples of kernels \citep{williams, NIPS2009_3628} are PNGs. 
Wiener's tauberian theorem shows that all stationary kernels may be approximated arbitrarily well by sums of PNGs \citep{samo2015generalized}. In future work we hope to explore whether ROMs can be used to achieve statistical benefit in estimation tasks associated with a wider range of PNGs.

	\section{Understanding the effectiveness of orthogonality}\label{sec:why} 

Here we build intuitive understanding for the effectiveness of ROMs. 
We examine geometrically the angular kernel (see \S \ref{sec:kernels}), then discuss a connection to random walks 
over orthogonal matrices.

\paragraph{Angular kernel.} 
As noted above for the $\mathbf{G}_{\mathrm{ort}}$-mechanism, smaller MSE than that for unstructured $\mathbf{G}$ is implied by the inequality $\mathbb{P}[\mathcal{A}^{i} \cap \mathcal{A}^{j}] < \mathbb{P}[\mathcal{A}^{i}]\mathbb{P}[\mathcal{A}^{j}]$, which is equivalent to:
$\mathbb{P}[\mathcal{A}^{j} | \mathcal{A}^{i}] < \mathbb{P}[\mathcal{A}^{j}]$.
Now it becomes clear why orthogonality is crucial. 
Without loss of generality take: $i=1$, $j=2$, and let $\mathbf{g}^{1}$ and $\mathbf{g}^{2}$ be the first two rows of $\mathbf{G}_{\mathrm{ort}}$. 

\begin{figure}[t]
	\centering
	\includegraphics[width=0.64\columnwidth]{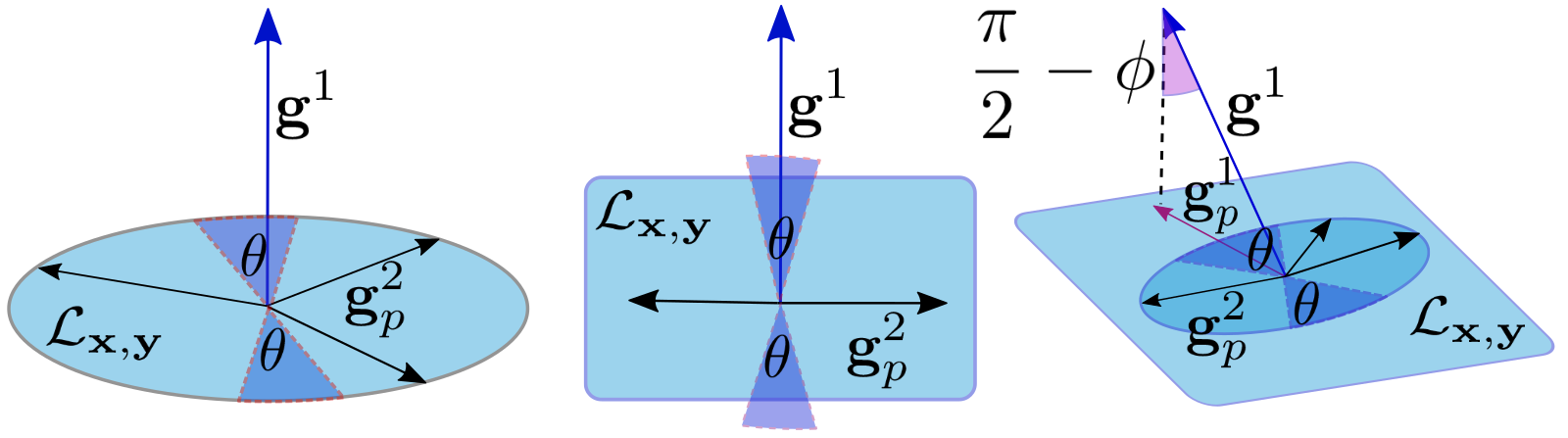} \quad
		\includegraphics[keepaspectratio, width=0.31\textwidth]{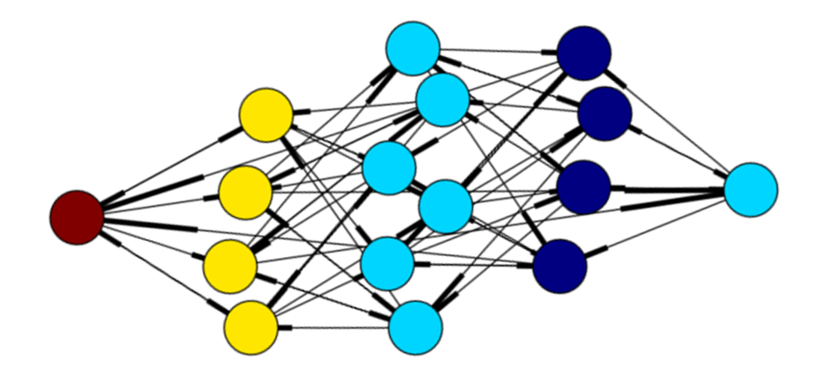}
	\caption{\small \textbf{Left part:} Left: $\mathbf{g}^{1}$ is orthogonal to $\mathcal{L}_{\mathbf{x},\mathbf{y}}$.
		Middle: $\mathbf{g}^{1} \in \mathcal{L}_{\mathbf{x},\mathbf{y}}$. Right: $\mathbf{g}^{1}$ is close to orthogonal to $\mathcal{L}_{\mathbf{x},\mathbf{y}}$.  
		\textbf{Right part:} Visualization of the Cayley graph explored by the Hadamard-Rademacher process in two dimensions. Nodes are colored red, yellow, light blue, dark blue, for Cayley distances of $0,1,2,3$ from the identity matrix respectively. 
		See text in \S \ref{sec:why}.}
	\label{fig:intuition}
\end{figure}

Consider first the extreme case (middle of left part of Figure \ref{fig:intuition}), where all vectors are $2$-dimensional. Recall definitions from just after Theorem \ref{angle_theorem}. 
If $\mathbf{g}^{1}$ is in $\mathcal{C}_{\mathbf{x},\mathbf{y}}$ then it is much less probable for $\mathbf{g}^{2}$  also to belong to $\mathcal{C}_{\mathbf{x},\mathbf{y}}$. In particular, if $\theta < \frac{\pi}{2}$ then the probability is zero. That implies the inequality. On the other hand, if $\mathbf{g}^{1}$ is perpendicular to $\mathcal{L}_{\mathbf{x},\mathbf{y}}$ then conditioning on $\mathcal{A}^{i}$ does not have any effect on the probability that $\mathbf{g}^{2}$ belongs to $\mathcal{C}_{\mathbf{x},\mathbf{y}}$ (left subfigure of Figure \ref{fig:intuition}).
In practice, with high probability the angle $\phi$ between $\mathbf{g}^{1}$ and $\mathcal{L}_{\mathbf{x},\mathbf{y}}$ is close to $\frac{\pi}{2}$, but is not exactly $\frac{\pi}{2}$. That again implies that conditioned
on the projection $\mathbf{g}^{1}_{p}$ of $\mathbf{g}^{1}$ into $\mathcal{L}_{\mathbf{x},\mathbf{y}}$ to be in $\mathcal{C}_{\mathbf{x},\mathbf{y}}$, the more probable directions of $\mathbf{g}^{2}_{p}$ are perpendicular to  $\mathbf{g}^{1}_{p}$ (see: ellipsoid-like shape in the right subfigure of Figure \ref{fig:intuition} which is the projection of the sphere taken from the $(n-1)$-dimensional space orthogonal to $\mathbf{g}^{1}$ into $\mathcal{L}_{\mathbf{x},\mathbf{y}}$). This makes it less probable for $\mathbf{g}^{2}_{p}$ to be also in $\mathcal{C}_{\mathbf{x},\mathbf{y}}$. 
The effect is subtle since $\phi \approx \frac{\pi}{2}$, but this is what provides superiority of the orthogonal transformations over state-of-the-art ones in the angular kernel approximation setting.

\paragraph{Markov chain perspective.}
We focus on Hadamard-Rademacher random matrices $\mathbf{HD}_{k}...\mathbf{HD}_{1}$, a special case of the $\mathbf{SD}$-product matrices described in Section \ref{sec:model}. Our aim is to provide intuition for how the choice of $k$ affects the quality of the random matrix, following our earlier observations just after Corollary \ref{cor:key}, which indicated that for $\mathbf{SD}$-product matrices, odd values of $k$ yield greater benefits than even values, and that there are diminishing benefits from higher values of $k$. 
We proceed by casting the random matrices into the framework of Markov chains. \\

\begin{definition}
The Hadamard-Rademacher process in $n$ dimensions is the Markov chain $(\mathbf{X}_k)_{k=0}^\infty$ taking values in the orthogonal group $O(n)$, with $\mathbf{X}_0 = \mathbf{I}$ almost surely, and $\mathbf{X}_k = \mathbf{HD}_k \mathbf{X}_{k-1}$ almost surely, where $\mathbf{H}$ is the normalized Hadamard matrix in $n$ dimensions, and $(\mathbf{D}_k)_{k=1}^\infty$ are iid diagonal matrices with independent Rademacher random variables on their diagonals.
\end{definition}

Constructing  an estimator based on Hadamard-Rademacher matrices is equivalent to simulating several time steps from the Hadamard-Rademacher process.
The quality of estimators based on Hadamard-Rademacher random matrices comes from a quick mixing property of the corresponding Markov chain. The following 
demonstrates attractive properties of the chain in low dimensions. \\

\begin{proposition}\label{prop:HD-MC-low-dim}
The Hadamard-Rademacher process in two dimensions: explores a state-space of $16$ orthogonal matrices, is ergodic with respect to the uniform distribution on this set, has period $2$, the diameter of the Cayley graph of its state space is $3$, and the chain is fully mixed after $3$ time steps.
\end{proposition}

This proposition, and the Cayley graph corresponding to the Markov chain's state space 
(Figure 
\ref{fig:intuition} right), illustrate the fast mixing properties of the Hadamard-Rademacher process in low dimensions; this agrees with the observations in \S \ref{sec:ojlt} that there are diminishing returns associated with using a large number $k$ of $\mathbf{HD}$ blocks in an estimator. The observation in Proposition \ref{prop:HD-MC-low-dim} that the Markov chain has period 2 
indicates that we should expect different behavior for estimators based on odd and even numbers of blocks of $\mathbf{HD}$ matrices, which is reflected in the analytic expressions for MSE derived in Theorems \ref{hd_theorem} and \ref{thm:hybrid-mse} for the dimensionality reduction setup.
	\section{Experiments}\label{sec:experiments}
\begin{figure*}[t]
	\centering
	\resizebox{.99\textwidth}{!}{%
		\subfigure[\texttt{g50c} - pointwise evaluation MSE for inner product estimation]{
			\includegraphics[keepaspectratio, width=0.25\textwidth]{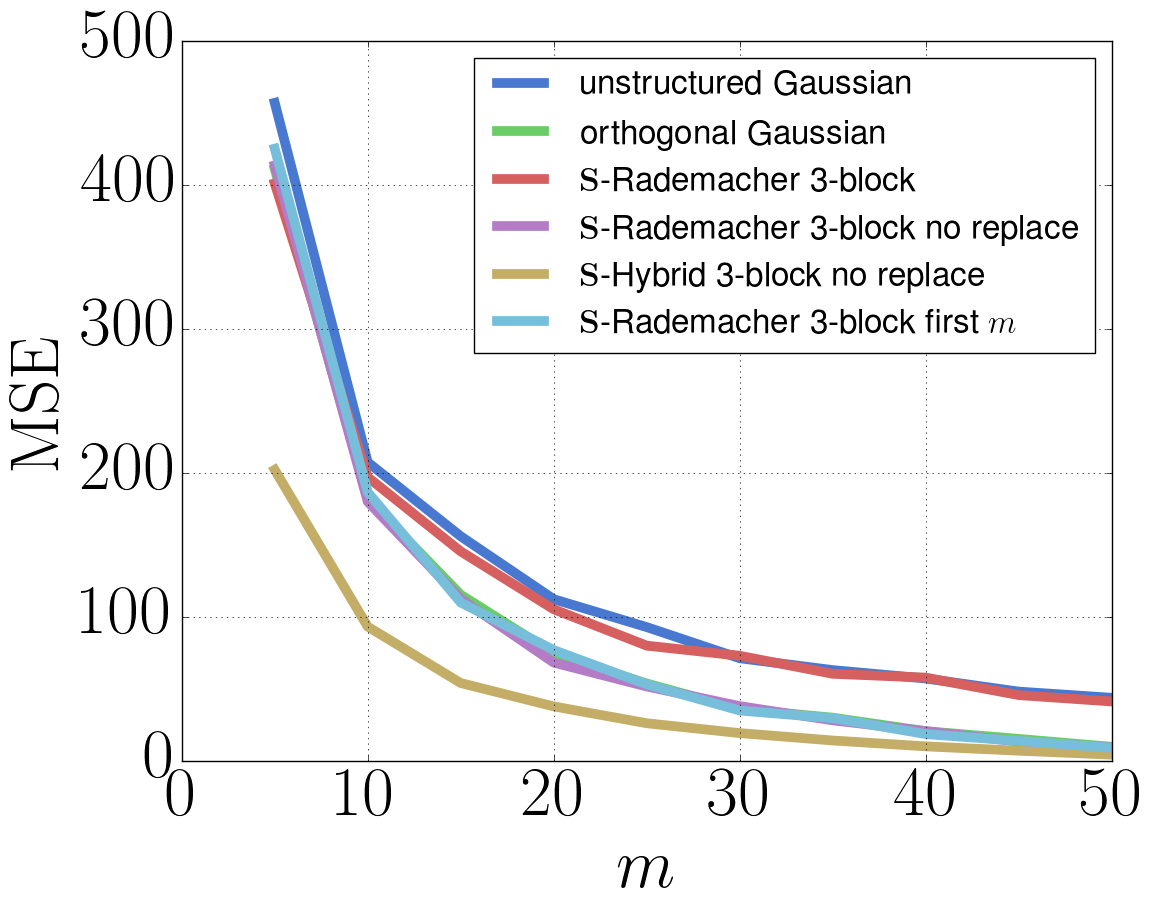}}
		\subfigure[\texttt{random} - angular kernel]{
			\includegraphics[keepaspectratio, width=0.25\textwidth]{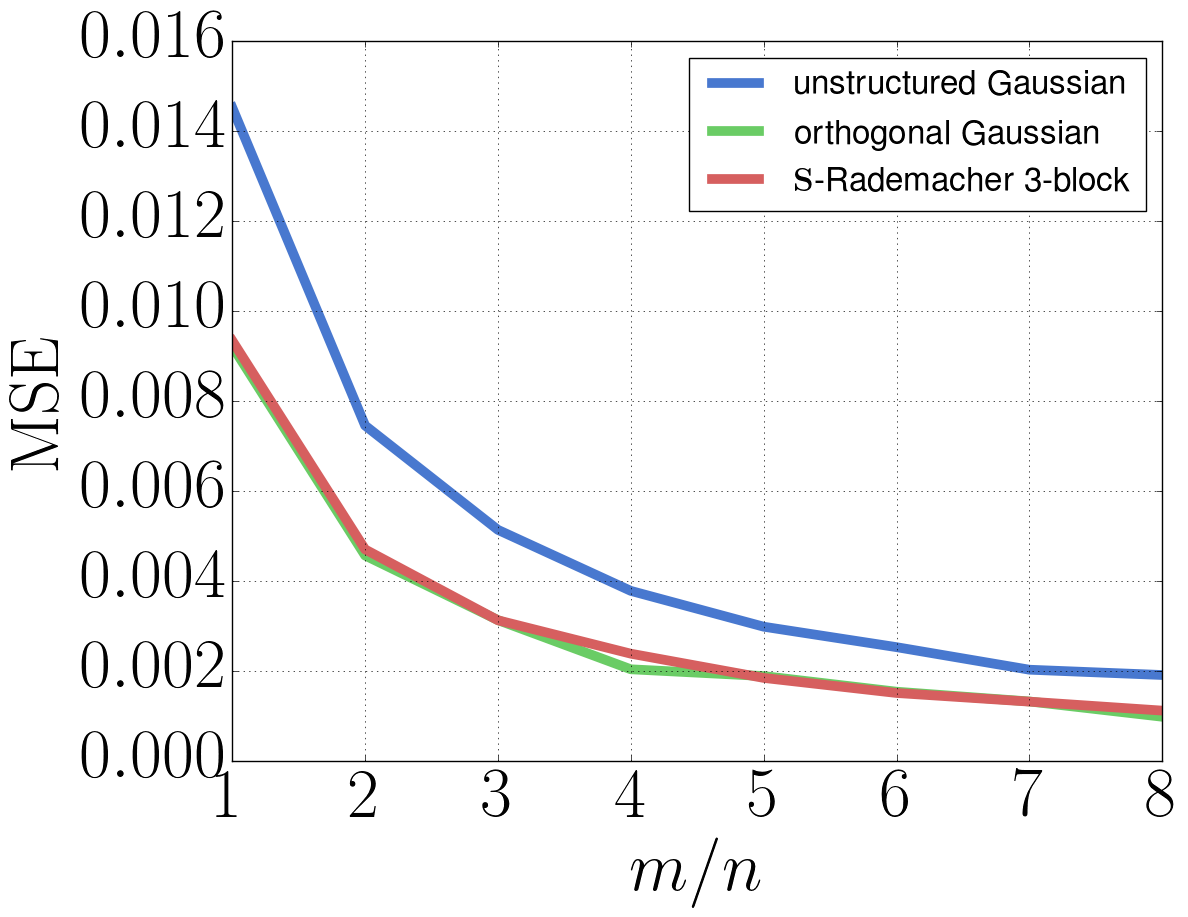}}
		\subfigure[\texttt{random} - angular kernel with true angle $\pi/4$]{
			\includegraphics[keepaspectratio, width=0.25\textwidth]{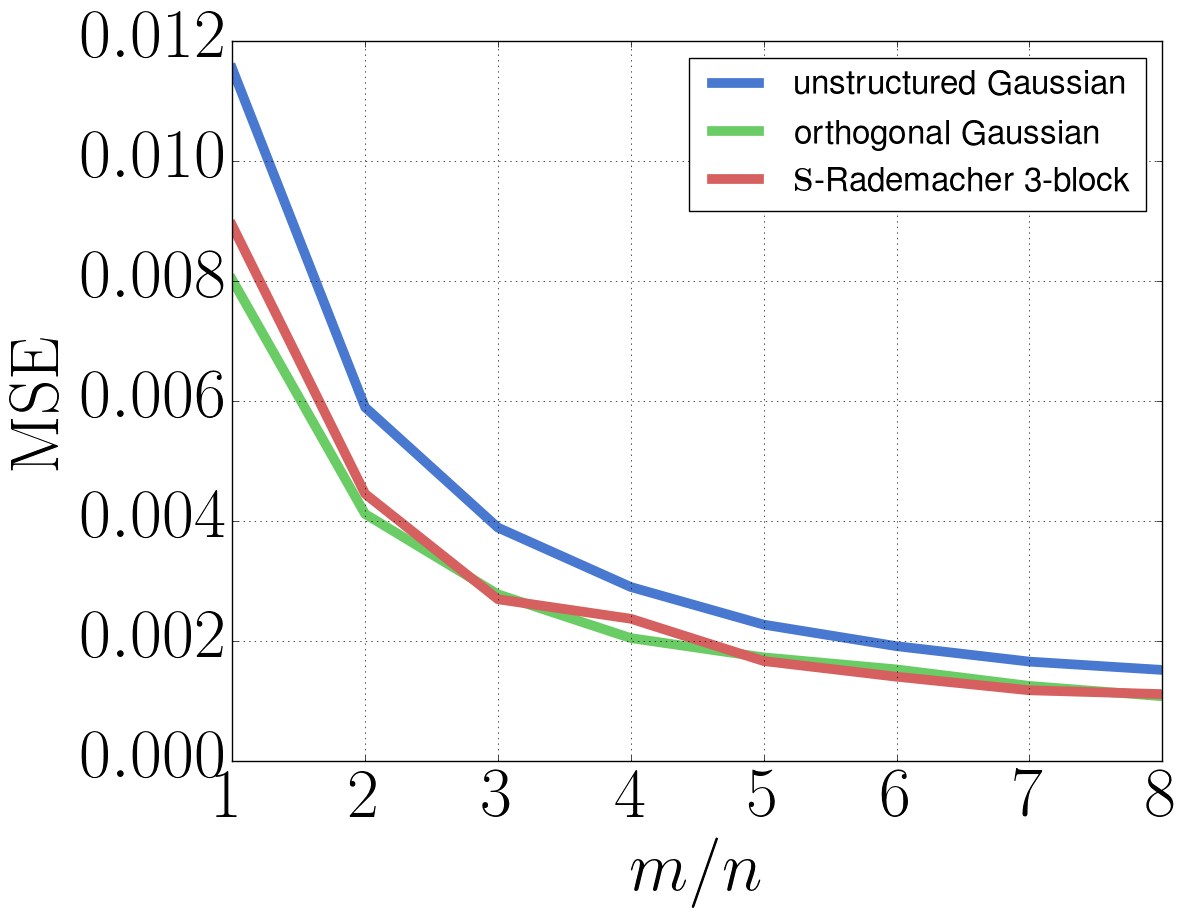}}
		\subfigure[\texttt{g50c} - inner product estimation MSE for variants of $3$-block $\mathbf{SD}$-product matrices.
		]{
			\includegraphics[keepaspectratio, width=0.25\textwidth]{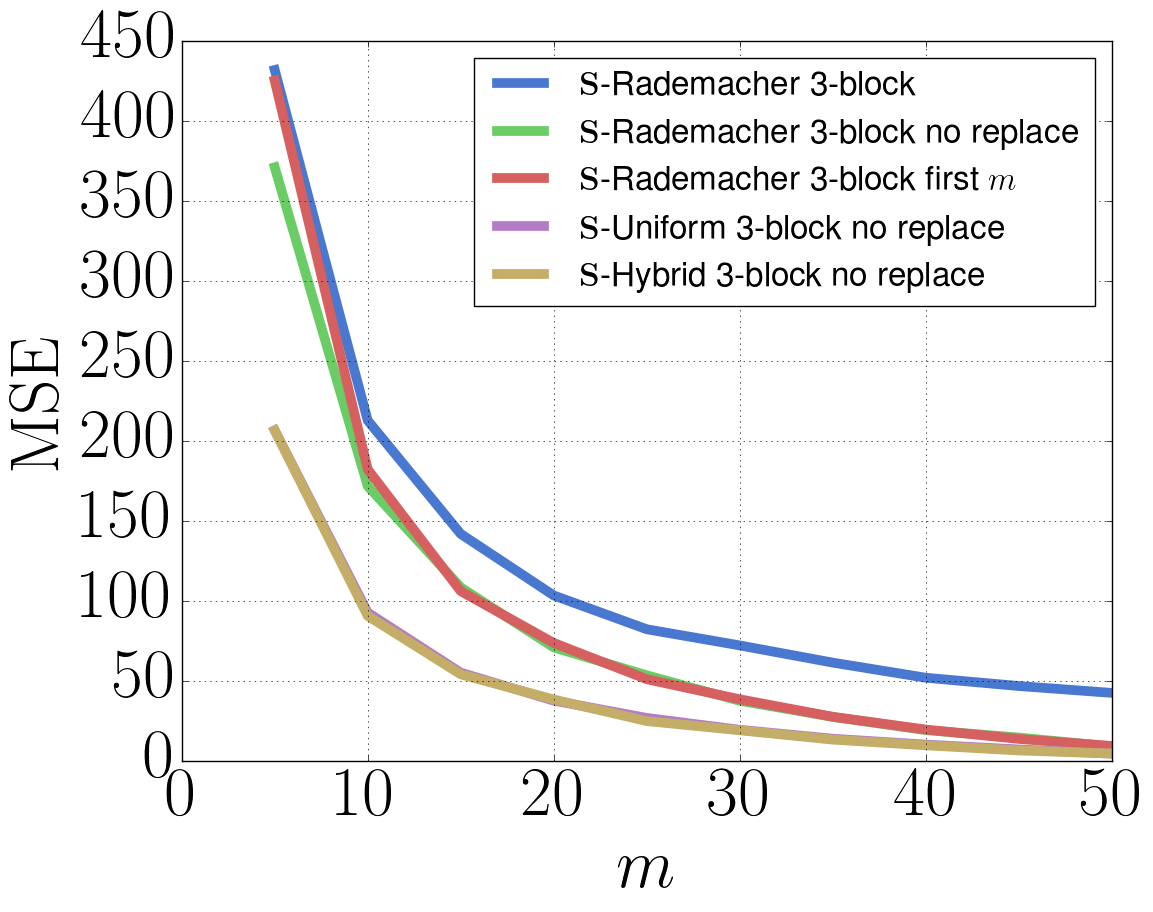}}
	}
	\resizebox{.99\textwidth}{!}{%
		\subfigure[\texttt{LETTER} - dot-product]{
			\includegraphics[keepaspectratio, width=0.25\textwidth]{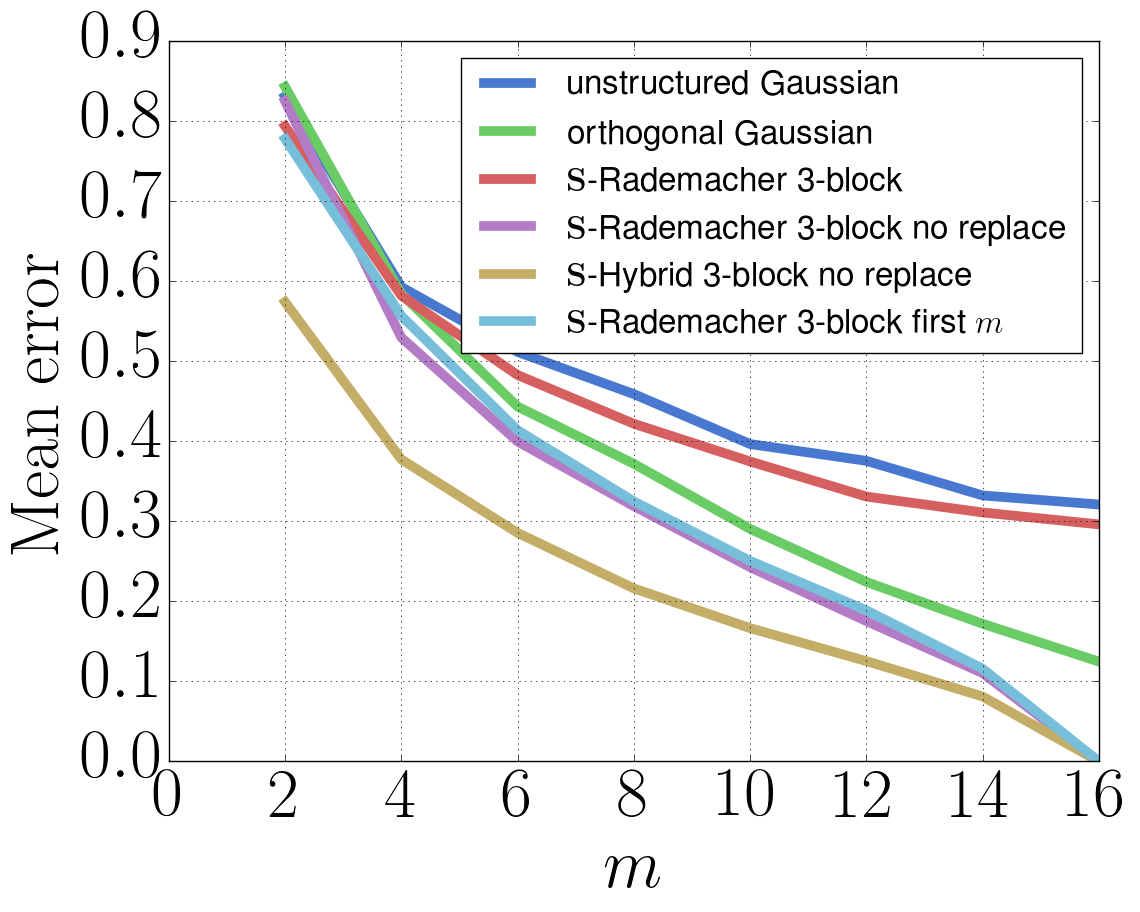}}
		\subfigure[\texttt{USPS} - dot-product]{
			\includegraphics[keepaspectratio, width=0.25\textwidth]{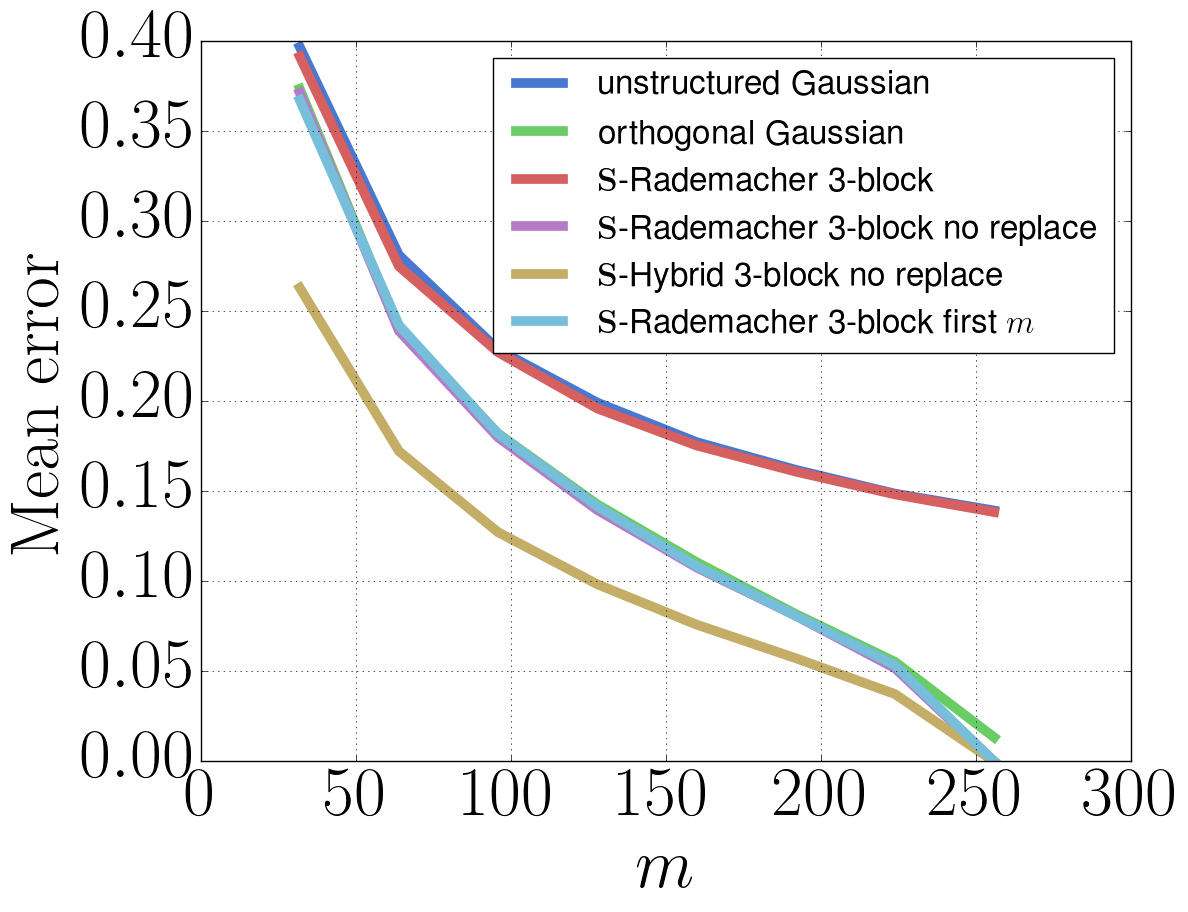}}
		\subfigure[\texttt{LETTER} - angular kernel]{
			\includegraphics[keepaspectratio, width=0.25\textwidth]{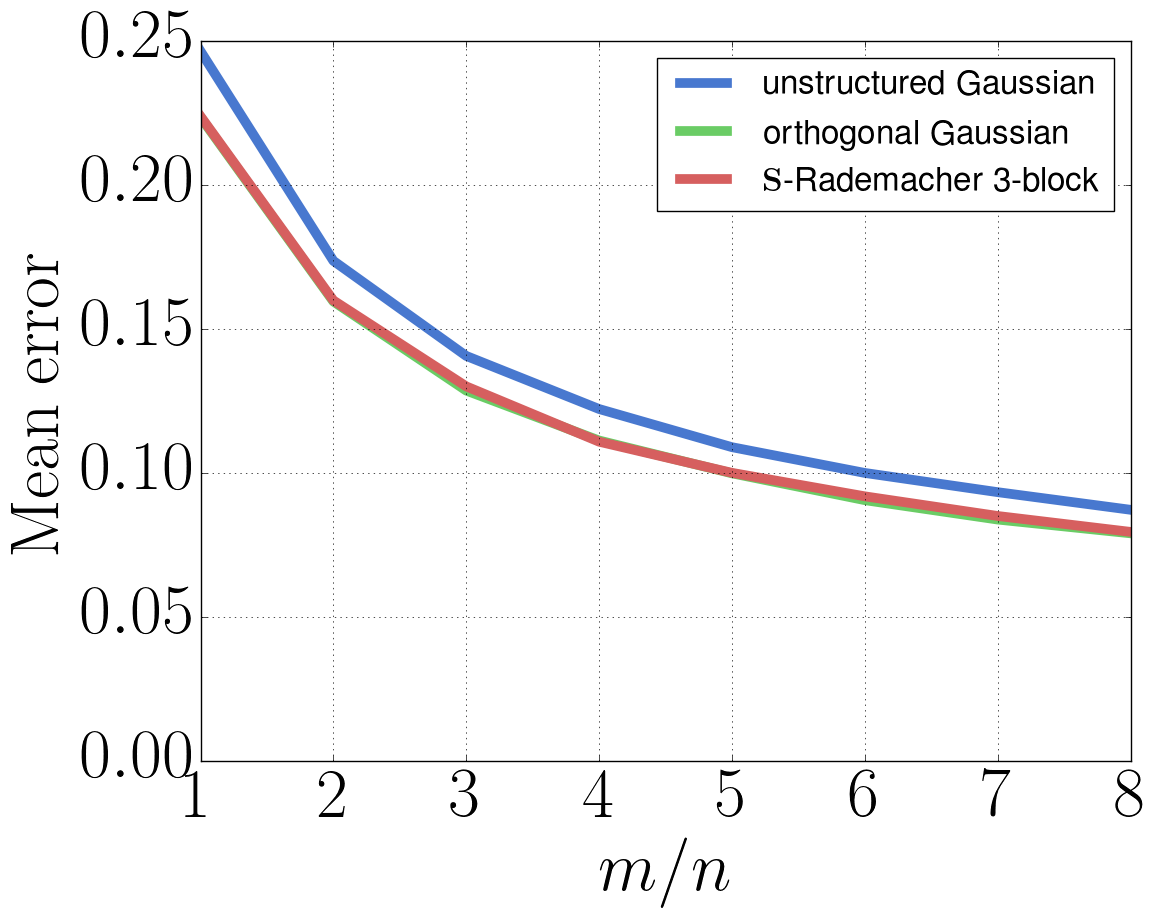}}
		\subfigure[\texttt{USPS} - angular kernel]{
			\includegraphics[keepaspectratio, width=0.25\textwidth]{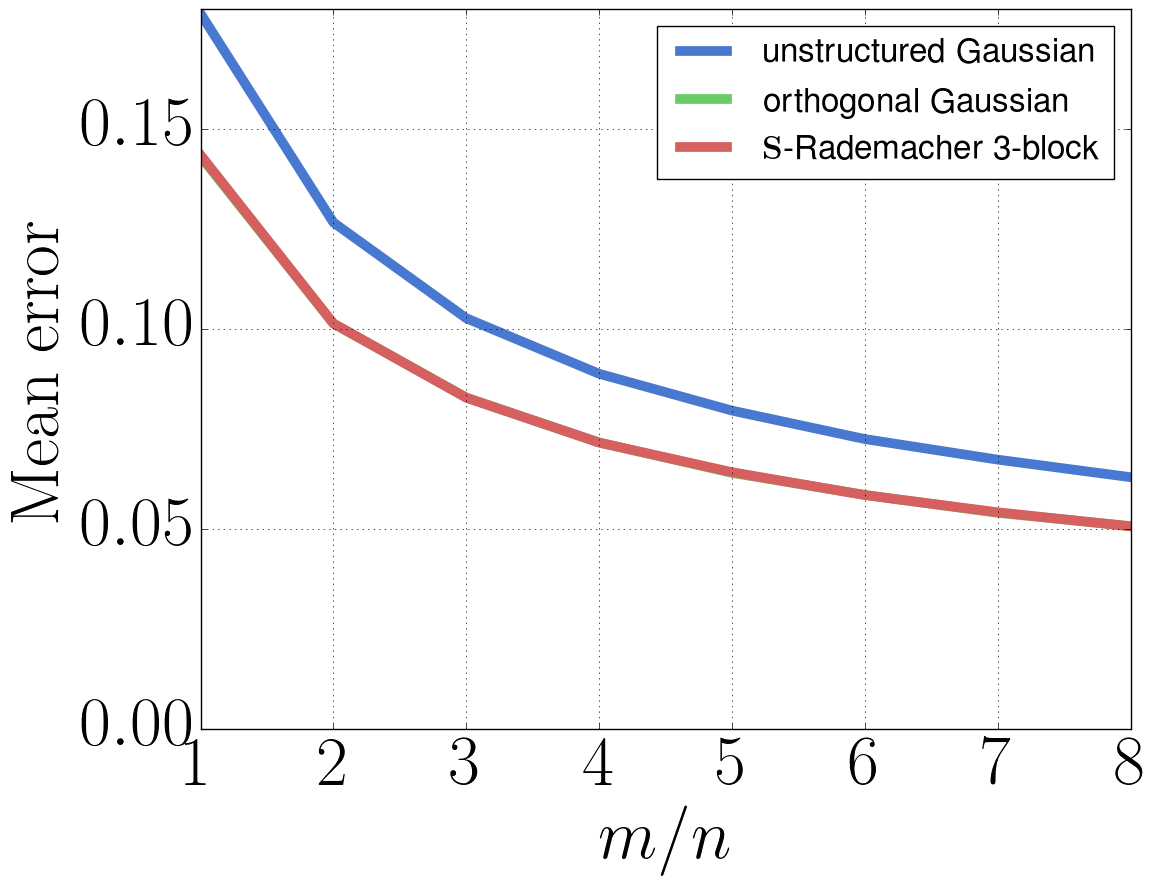}}
	}
	\caption{\small \textbf{Top row:} MSE curves for pointwise approximation of inner product and angular kernels on the \texttt{g50c} dataset, and randomly chosen vectors. \textbf{Bottom row:} Gram matrix approximation error for a variety of data sets, projection ranks, transforms, and kernels. Note that the error scaling is dependent on the application.}
	\label{fig:pointwise} 
\end{figure*}

We present comparisons of estimators introduced in \S \ref{sec:ojlt} and \S \ref{sec:kernels}, 
illustrating our theoretical results, and further demonstrating the 
empirical success of ROM-based estimators at the level of Gram matrix approximation. 
We compare estimators based on: unstructured Gaussian matrices $\mathbf{G}$, matrices $\mathbf{G}_{\mathrm{ort}}$,
$\mathbf{S}$-\textit{Rademacher} and $\mathbf{S}$-\textit{Hybrid} matrices with $k=3$ and different sub-sampling strategies.
Results for $k>3$ do not show additional statistical gains empirically. Additional experimental results, including a comparison of estimators using different numbers of  $\mathbf{SD}$ blocks, are in the Appendix \S \ref{sec:comparison}.
Throughout, we use the normalized Hadamard matrix $\mathbf{H}$ for the structured matrix $\mathbf{S}$.

\subsection{Pointwise kernel approximation}
Complementing the theoretical results of \S \ref{sec:ojlt} and \S \ref{sec:kernels}, we provide several salient comparisons of 
the various methods introduced - see Figure \ref{fig:pointwise} top.
Plots presented here (and in the Appendix) compare MSE for dot-product and angular and  kernel. They show that estimators based on $\mathbf{G}_{\mathrm{ort}}$, 
$\mathbf{S}$-\textit{Hybrid} and $\mathbf{S}$-\textit{Rademacher} matrices without replacement, or using the first $m$ rows, beat the state-of-the-art unstructured $\textbf{G}$ approach on accuracy 
for all our different datasets in the JLT setup. Interestingly, the latter two approaches give also smaller MSE than $\mathbf{G}_{\mathrm{ort}}$-estimators.
For angular kernel estimation, where sampling is not relevant, we see that $\mathbf{G}_{\mathrm{ort}}$ and $\mathbf{S}$-\textit{Rademacher} approaches again
outperform the ones based on matrices $\mathbf{G}$.

\subsection{Gram matrix approximation}\label{sec:gram-matrix-experiments}
Moving beyond the theoretical guarantees established in \S \ref{sec:ojlt} 
and \S \ref{sec:kernels}, we show empirically that the superiority of estimators 
based on ROMs is maintained at the level of Gram matrix approximation. 
We compute Gram matrix  approximations (with respect to both standard dot-product, and angular kernel)
for a variety of datasets. We use the normalized Frobenius norm error ${\|\mathbf{K} - \mathbf{\widehat{K}}\|_2}/\|\mathbf{K}\|_2$ 
as our metric (as used by \citealp{chor_sind_2016}), and plot the mean error based on 1,000 repetitions of each random transform - see Figure \ref{fig:pointwise} bottom. 
The Gram matrices are computed on a randomly selected subset of $550$ data points from each dataset. As can be seen, the $\mathbf{S}$-\textit{Hybrid} 
estimators using the ``no-replacement'' 
or ``first $m$ rows'' sub-sampling strategies outperform even the orthogonal Gaussian ones in the dot-product case. For the angular case, the $\mathbf{G}_{\mathrm{ort}}$-approach and $\mathbf{S}$-\textit{Rademacher} approach are practically indistinguishable.

	\section{Conclusion}\label{sec:conclusions}

We 
defined the family of random ortho-matrices (ROMs). This contains 
the $\mathbf{SD}$-product matrices, which include a number of recently proposed structured random matrices.  
We showed theoretically and empirically that ROMs have strong statistical and 
computational properties (in several cases outperforming previous state-of-the-art) for algorithms performing 
dimensionality reduction and random feature approximations of kernels. 
We highlight Corollary \ref{cor:key}, which provides a theoretical guarantee that $\mathbf{SD}$-product matrices yield better accuracy than iid matrices in an important dimensionality reduction application (we believe the first result of this kind).   
Intriguingly, for dimensionality reduction, using just one complex structured matrix yields random features of much better quality. We provided perspectives to help understand the benefits of ROMs, and to help explain the behavior of $\mathbf{SD}$-product matrices for various numbers of blocks. Our empirical findings suggest that our theoretical results might be further strengthened, particularly in the kernel setting.

\section*{Acknowledgements} 
We thank Vikas Sindhwani at Google Brain Robotics and Tamas Sarlos at Google Research for inspiring conversations that led to this work. We thank Matej Balog, Maria Lomeli, Jiri Hron and Dave Janz for helpful comments. MR acknowledges support by the UK Engineering and Physical Sciences Research Council (EPSRC) grant EP/L016516/1 for the University of Cambridge Centre for Doctoral Training, the Cambridge Centre for Analysis. AW acknowledges support by the Alan Turing Institute under the EPSRC grant
EP/N510129/1, and by the Leverhulme Trust via the CFI.

	%
	%
	%
	%
	
\newpage
{\small
\bibliography{orthogonal_transforms}
}
\newpage
\section*{APPENDIX: \\The Unreasonable Effectiveness of Random Orthogonal Embeddings}

We present here details and proofs of all the theoretical results presented in the main body of the paper. We also provide further experimental results in \S \ref{sec:comparison}. 

We highlight proofs of several key results that may be of particular interest to the reader:

\begin{itemize}
	\item The proof of Theorem \ref{hd_theorem}; see \S \ref{sec:S-rademacher-proof}.
	\item The proof of Theorem \ref{thm:hybrid-mse}; see \S \ref{sec:S-hybrid-proof}.
	\item The proof of Theorem \ref{angle_theorem}; see \S \ref{sec:ang-ort-proof}.
\end{itemize}
In the Appendix we will use interchangeably two notations for the dot product between vectors $\mathbf{x}$ and $\mathbf{y}$, namely: $\mathbf{x}^{\top}\mathbf{y}$ and $\langle \mathbf{x}, \mathbf{y}\rangle$.

\section{Proofs of results in \texorpdfstring{\S\ref{sec:ojlt}}{Section \ref{sec:ojlt}}}

\subsection{Proof of Lemma \ref{easy_lemma}}

\begin{proof}
Denote $X_{i} = (\mathbf{g}^{i})^{\top}\mathbf{x} \cdot (\mathbf{g}^{i})^{\top}\mathbf{y}$, where $\mathbf{g}^{i}$ stands for the $i^{th}$ row of the
unstructured Gaussian matrix $\mathbf{G} \in \mathbb{R}^{m \times n}$.
Note that we have:
\begin{equation}
 \widehat{K}^{\mathrm{base}}_{m}(\mathbf{x},\mathbf{y}) = \frac{1}{m}\sum_{i=1}^{m} X_{i}.
\end{equation}

Denote $\mathbf{g}^{i}=(g^{i}_{1},...,g^{i}_{n})^{\top}$.
Notice that from the independence of $g^{i}_{j}$s and the fact that: $\mathbb{E}[g^{i}_{j}]=0$, $\mathbb{E}[(g^{i}_{j})^{2}]=1$, we get: $\mathbb{E}[X_{i}] = \sum_{i=1}^{n} x_{i}y_{i}=\mathbf{x}^{\top}\mathbf{y}$, thus the estimator is unbiased.
Since the estimator is unbiased, we have: $\mathrm{MSE}(\widehat{K}^{\mathrm{base}}_{m}(\mathbf{x},\mathbf{y})) = Var(\widehat{K}^{\mathrm{base}}_{m}(\mathbf{x},\mathbf{y}))$.
Thus we get:
\begin{equation}
\mathrm{MSE}(\widehat{K}^{\mathrm{base}}_{m}(\mathbf{x},\mathbf{y})) = \frac{1}{m^{2}} \sum_{i,j} (\mathbb{E}[X_{i}X_{j}] - \mathbb{E}[X_{i}]\mathbb{E}[X_{j}]). 
\end{equation}

From the independence of different $X_{i}$s, we get:

\begin{equation}
\mathrm{MSE}(\widehat{K}^{\mathrm{base}}_{m}(\mathbf{x},\mathbf{y})) = \frac{1}{m^{2}} \sum_{i} (\mathbb{E}[X_{i}^{2}] - (\mathbb{E}[X_{i}])^{2}). 
\end{equation}

Now notice that different $X_{i}$s have the same distribution, thus we get:

\begin{equation}
\mathrm{MSE}(\widehat{K}^{\mathrm{base}}_{m}(\mathbf{x},\mathbf{y})) = \frac{1}{m} (\mathbb{E}[X_{1}^{2}] - (\mathbb{E}[X_{1}])^{2}). 
\end{equation}

From the unbiasedness of the estimator, we have: $\mathbb{E}[X_{1}] = \mathbf{x}^{\top}\mathbf{y}$.
Therefore we obtain:
\begin{equation}
\mathrm{MSE}(\widehat{K}^{\mathrm{base}}_{m}(\mathbf{x},\mathbf{y})) = \frac{1}{m} (\mathbb{E}[X_{1}^{2}] - (\mathbf{x}^{\top}\mathbf{y})^{2}). 
\end{equation}

Now notice that
\begin{equation}
\mathbb{E}[X_{1}^{2}] = \mathbb{E}[\sum_{i_{1},j_{1},i_{2},j_{2}} g_{i_{1}}g_{j_{1}}g_{i_{2}}g_{j_{2}} x_{i_{1}}y_{j_{1}}x_{i_{2}}y_{j_{2}}]=
\sum_{i_{1},j_{1},i_{2},j_{2}} x_{i_{1}}y_{j_{1}}x_{i_{2}}y_{j_{2}}\mathbb{E}[g_{i_{1}}g_{j_{1}}g_{i_{2}}g_{j_{2}}], 
\end{equation}

where $(g_{1},...,g_{n})$ stands for the first row of $\mathbf{G}$. In the expression above the only nonzero terms corresponds to
quadruples $(i_{1},j_{1},i_{2},j_{2})$, where no index appears odd number of times.
Therefore, from the inclusion-exclusion principle and the fact that $\mathbb{E}[g_{i}^{2}]=1$ and $\mathbb{E}[g_{i}^{4}]=3$, we obtain
\begin{align} 
\mathbb{E}[X_{1}^{2}] &= \sum_{i_{1}=j_{1},i_{2}=j_{2}} x_{i_{1}}y_{j_{1}}x_{i_{2}}y_{j_{2}}\mathbb{E}[g_{i_{1}}g_{j_{1}}g_{i_{2}}g_{j_{2}}] +
\sum_{i_{1}=i_{2},j_{1}=j_{2}} x_{i_{1}}y_{j_{1}}x_{i_{2}}y_{j_{2}}\mathbb{E}[g_{i_{1}}g_{j_{1}}g_{i_{2}}g_{j_{2}}]  \\
& \- \qquad + \sum_{i_{1}=j_{2},i_{2}=j_{1}} x_{i_{1}}y_{j_{1}}x_{i_{2}}y_{j_{2}}\mathbb{E}[g_{i_{1}}g_{j_{1}}g_{i_{2}}g_{j_{2}}]
- \sum_{i_{1}=j_{1}=i_{2}=j_{2}} x_{i_{1}}y_{j_{1}}x_{i_{2}}y_{j_{2}}\mathbb{E}[g_{i_{1}}g_{j_{1}}g_{i_{2}}g_{j_{2}}]\\
&=((\mathbf{x}^{\top}\mathbf{y})^{2}-\sum_{i=1}^{n}x_{i}^{2}y_{i}^{2} + 3\sum_{i=1}^{n}x_{i}^{2}y_{i}^{2}) + 
((\|\mathbf{x}\|_{2}\|\mathbf{y}\|_{2})^{2}-\sum_{i=1}^{n}x_{i}^{2}y_{i}^{2} + 3\sum_{i=1}^{n}x_{i}^{2}y_{i}^{2})  \\
& \- \qquad + ((\mathbf{x}^{\top}\mathbf{y})^{2}-\sum_{i=1}^{n}x_{i}^{2}y_{i}^{2} + 3\sum_{i=1}^{n}x_{i}^{2}y_{i}^{2})-
3 \cdot 2 \sum_{i=1}^{n}x_{i}^{2}y_{i}^{2}  \\
&= (\|\mathbf{x}\|_{2}\|\mathbf{y}\|_{2})^{2} + 2(\mathbf{x}^{\top}\mathbf{y})^{2}.
\end{align}
Therefore we obtain
\begin{equation}
\mathrm{MSE}(\widehat{K}^{\mathrm{base}}_{m}(\mathbf{x},\mathbf{y})) = 
\frac{1}{m} ((\|\mathbf{x}\|_{2}\|\mathbf{y}\|_{2})^{2} + 2(\mathbf{x}^{\top}\mathbf{y})^{2} - (\mathbf{x}^{\top}\mathbf{y})^{2})=
\frac{1}{m}(\|\mathbf{x}\|_{2}^{2}\|\mathbf{y}\|_{2}^{2} + (\mathbf{x}^{\top}\mathbf{y})^{2}),
\end{equation}
which completes the proof.
\end{proof}

\subsection{Proof of Theorem \ref{gaussian_ojlt}}

\begin{proof}

	The unbiasedness of the Gaussian orthogonal estimator comes from the fact that every row of the Gaussian orthogonal matrix is sampled from multivariate Gaussian distribution with entries taken independently at random from $\mathcal{N}(0,1)$.
	
          Note that:
	\begin{equation}
	\mathrm{Cov}(X_{i}, X_{j}) = \mathbb{E}[X_{i}X_{j}] -   \mathbb{E}[X_{i}] \mathbb{E}[X_{j}],
	\end{equation}
	where: $X_{i} =(\mathbf{r}_{i}^{\top}\mathbf{x})( \mathbf{r}_{i}^{\top}\mathbf{y})$, $X_{j} = (\mathbf{r}_{j}^{\top}\mathbf{x})(\mathbf{r}_{j}^{\top}\mathbf{y})$ and $\mathbf{r}_{i},\mathbf{r}_{j}$ stand for the $i^{th}$ and $j^{th}$ row of the Gaussian orthogonal matrix respectively.
	From the fact that  Gaussian orthogonal estimator is unbiased, we get:
	\begin{equation}
	\mathbb{E}[X_{i}] = \mathbf{x}^{\top}\mathbf{y}.
	\end{equation}
	
	Let us now compute $\mathbb{E}[X_{i}X_{j}]$. 
	Writing $\mathbf{Z}_{1} = \mathbf{r}_{i}$, $\mathbf{Z}_{2} = \mathbf{r}_{j}$, we begin with some geometric observations:
	\begin{itemize}
		\item If $\phi \in [0, \pi/2]$ is the acute angle between $\mathbf{Z}_1$ and the $\mathbf{x}$-$\mathbf{y}$ plane, then $\phi$ has density $f(\phi) =(n-2) \cos(\phi)\sin^{n-3}(\phi)$.
		\item The squared norm of the projection of $\mathbf{Z}_1$ into the $\mathbf{x}$-$\mathbf{y}$ plane is therefore given by the product of a $\chi^2_n$ random variable (the norm of $\mathbf{Z}_2$), multiplied by $\cos^2(\phi)$, where $\phi$ is distributed as described above, independently from the $\chi^2_n$ random variable.
		\item The angle $\psi \in [0, 2\pi)$ between $\mathbf{x}$ and the projection of $\mathbf{Z}_1$ into the $\mathbf{x}$-$\mathbf{y}$ plane is distributed uniformly.
		\item Conditioned on the angle $\phi$, the direction of $\mathbf{Z}_2$ is distributed uniformly on the hyperplane of $\mathbb{R}^n$ orthogonal to $\mathbf{Z}_1$. Using hyperspherical coordinates for the unit hypersphere of this hyperplane, we may pick an orthonormal basis of the $\mathbf{x}$-$\mathbf{y}$ plane such that the first basis vector is the unit vector in the direction of the projection of $\mathbf{Z}_1$, and the coordinates of the projection of $\mathbf{Z}_2$ with respect to this basis are $(\sin(\phi) \cos(\varphi_1), \sin(\varphi_1) \cos(\varphi_2))$, where $\varphi_1, \varphi_2$ are random angles taking values in $[0, \pi]$, with densities given by $\sin^{n-3}(\varphi_1) I(n-3)^{-1}$ and $\sin^{n-4}(\varphi_2) I(n-4)^{-1}$ respectively. Here $I(k) = \int_0^\pi \sin^k(x) dx = \sqrt{\pi}\Gamma((k+1)/2)/\Gamma(k/2 + 1)$.
		\item The angle $t$ that the projection of $\mathbf{Z}_2$ into the $\mathbf{x}$-$\mathbf{y}$ plane makes with the projection of $\mathbf{Z}_1$ then satisfies $\tan(t) = \sin(\varphi_1) \cos(\varphi_2) / (\sin(\phi) \cos(\varphi_1)) = \cos(\varphi_1)/\sin(\phi) \times \tan(\varphi_1)$.
	\end{itemize}
	
	Applying these observations, we get:
	\begin{align}
	&\mathbb{E}[X_{i}X_{j}] \nonumber \\
	 =& \mathbb{E}[(\mathbf{r}_{i}^{\top}\mathbf{x})( \mathbf{r}_{i}^{\top}\mathbf{y})(\mathbf{r}_{j}^{\top}\mathbf{x})(\mathbf{r}_{j}^{\top}\mathbf{y})] \nonumber\\
	 =& \|\mathbf{x}\|_2^2 \|\mathbf{y}\|_2^2 n^2 \int_{0}^{\pi/2}\! \!\!\!\! d\phi f(\phi) \cos^2(\phi) \int_{0}^{\pi}\!\!\! d\varphi_1 \sin^{n-3}(\varphi_1) I(n-3)^{-1} \int_{0}^\pi\!\!\! d\varphi_2 \sin^{n-4}(\varphi_2) I(n-4)^{-1}  \times \nonumber \\
	& \int_0^{2\pi} \frac{d\psi}{2\pi}\left( \sin^2(\phi) \cos^2(\varphi_1) + \sin^2(\varphi_1) \cos^2(\varphi_2) \right) \cos(\psi) \cos(\psi + \theta) \cos(t - \psi) \cos(t - \theta - \psi).
	\end{align}
	We first apply the cosine product formula to the two adjacent pairs making up the final product of four cosines involving $\psi$ in the integrand above. The majority of these terms vanish upon integrating with respect to $\psi$, due to the periodicity of the integrands wrt $\psi$. We are thus left with:
	\begin{align}
	&\mathbb{E}[X_{i}X_{j}] \nonumber \\
	=& \|\mathbf{x}\|_2^2 \|\mathbf{y}\|_2^2 n^2 \int_{0}^{\pi/2}\! \!\!\!\! d\phi f(\phi) \cos^2(\phi) \int_{0}^{\pi} d\varphi_1 \sin^{n-3}(\varphi_1) I(n-3)^{-1} \int_{0}^\pi d\varphi_2 \sin^{n-4}(\varphi_2) I(n-4)^{-1}  \times \nonumber \\
	& \left( \sin^2(\phi) \cos^2(\varphi_1) + \sin^2(\varphi_1) \cos^2(\varphi_2) \right) \left( \frac{1}{4} \cos^2(\theta) + \frac{1}{8} \cos(2t)\right). \label{eq:twoparts}
	\end{align}
	We now consider two constituent parts of the integral above: one involving the term $\frac{1}{4}\cos^2(\theta)$, and the other involving $\frac{1}{8}\cos(2t)$. We deal first with the former; its evaluation requires several standard trigonometric integrals:
	\begin{align}
	&\|\mathbf{x}\|_2^2 \|\mathbf{y}\|_2^2 n^2 \int_{0}^{\pi/2}\! \!\!\!\! d\phi f(\phi) \cos^2(\phi) \int_{0}^{\pi} d\varphi_1 \sin^{n-3}(\varphi_1) I(n-3)^{-1} \int_{0}^\pi d\varphi_2 \sin^{n-4}(\varphi_2) I(n-4)^{-1}  \times \nonumber \\
	& \left( \sin^2(\phi) \cos^2(\varphi_1) + \sin^2(\varphi_1) \cos^2(\varphi_2) \right) \frac{1}{4} \cos^2(\theta) \nonumber \\
	= & \frac{\|\mathbf{x}\|_2^2 \| \mathbf{y}\|_2^2 n^2\cos^2(\theta)}{4 I(n-3) I(n-4)} \int_0^{\pi/2} d\phi f(\phi) \cos^2(\phi) \int_0^\pi d\varphi_1 \sin^{n-3}(\varphi_1) \times \nonumber \\
	& \qquad\qquad \qquad \qquad\qquad \qquad\left( \sin^2(\phi) \cos^2(\varphi_1) I(n-4) + \sin^2(\varphi_1)\left( I(n-4) - I(n-2) \right) \right) \nonumber \\
	= & \frac{\|\mathbf{x}\|_2^2 \| \mathbf{y}\|_2^2 n^2 \cos^2(\theta)}{4 I(n-3) I(n-4)} \int_0^{\pi/2} d\phi (n-2) \sin^{n-3}(\phi) \cos(\phi) \cos^2(\phi)  \times \nonumber \\
	&\qquad\qquad \qquad \qquad\qquad \left( \sin^2(\phi) (I(n-3) - I(n-1)) I(n-4) + I(n-1)\left( I(n-4) - I(n-2) \right) \right) \nonumber \\
	= & \frac{\|\mathbf{x}\|_2^2 \| \mathbf{y}\|_2^2 n^2(n-2) \cos^2(\theta)}{4 I(n-3) I(n-4)} \bigg( \left(\frac{1}{n} - \frac{1}{n+2}\right) (I(n-3) - I(n-1))I(n-4) + \nonumber \\
	& \qquad \qquad \qquad \qquad \qquad \qquad I(n-1)\left( I(n-4) - I(n-2) \right) \left(\frac{1}{n-2} - \frac{1}{n}\right) \bigg) \label{eq:cos2thetaworking} \, .
	\end{align}
	We may now turn our attention to the other constituent integral of Equation \eqref{eq:twoparts}, which involves the term $\cos(2t)$. Recall that from our earlier geometric considerations, we have $\tan(t) = \frac{\cos(\varphi_2)}{\sin(\phi)}\tan(\phi_1)$. An elementary trigonometric calculation using the tan half-angle formula yields:
	\begin{align}
	\cos(2t) & = \cos\left(2\arctan\left(\frac{\cos(\varphi_2)}{\sin(\phi)}\tan(\varphi_1)\right)\right) \nonumber \\
	& = \frac{1 - \frac{\cos^2(\varphi_2)}{\sin^2(\phi)} \tan^2(\varphi_1)}{\frac{\cos^2(\varphi_2)}{\sin^2(\phi)}\tan^2(\varphi_1) + 1} \nonumber \\
	& = \frac{\sin^2(\phi) \cos^2(\varphi_1) - \cos^2(\varphi_2)\sin^2(\varphi_1)}{\cos^2(\varphi_2)\sin^2(\varphi_1) + \sin^2(\phi) \cos^2(\varphi_1)} \, .
	\end{align}	
	This observation greatly simplifies the integral from Equation \eqref{eq:twoparts} involving the term $\cos(2t)$, as follows:
	\begin{align}
	 &\|\mathbf{x}\|_2^2 \|\mathbf{y}\|_2^2 n^2 \int_{0}^{\pi/2}\! \!\!\!\! d\phi f(\phi) \cos^2(\phi) \int_{0}^{\pi} d\varphi_1 \sin^{n-3}(\varphi_1) I(n-3)^{-1} \int_{0}^\pi d\varphi_2 \sin^{n-4}(\varphi_2) I(n-4)^{-1}  \times \nonumber \\
	 & \left( \sin^2(\phi) \cos^2(\varphi_1) + \sin^2(\varphi_1) \cos^2(\varphi_2) \right) \frac{1}{8} \cos(2t) \nonumber \\
	 = & \frac{\|\mathbf{x}\|_2^2 \|\mathbf{y}\|_2^2 n^2}{8I(n-3)I(n-4)} \int_{0}^{\pi/2}\! \!\!\!\! d\phi f(\phi) \cos^2(\phi) \int_{0}^{\pi} d\varphi_1 \sin^{n-3}(\varphi_1) \int_{0}^\pi d\varphi_2 \sin^{n-4}(\varphi_2) \times \nonumber \\
	 & \left( \sin^2(\phi) \cos^2(\varphi_1) + \sin^2(\varphi_1) \cos^2(\varphi_2) \right) \frac{\sin^2(\phi) \cos^2(\varphi_1) - \cos^2(\varphi_2)\sin^2(\varphi_1)}{\cos^2(\varphi_2)\sin^2(\varphi_1) + \sin^2(\phi) \cos^2(\varphi_1)} \nonumber\\
	  = & \frac{\|\mathbf{x}\|_2^2 \|\mathbf{y}\|_2^2 n^2}{8I(n-3)I(n-4)} \int_{0}^{\pi/2}\! \!\!\!\! d\phi f(\phi) \cos^2(\phi) \int_{0}^{\pi} d\varphi_1 \sin^{n-3}(\varphi_1) \int_{0}^\pi d\varphi_2 \sin^{n-4}(\varphi_2) \times \nonumber \\
	  & \qquad \qquad \qquad \qquad \qquad \qquad \qquad \qquad \qquad \qquad\left(\sin^2(\phi) \cos^2(\varphi_1) - \cos^2(\varphi_2)\sin^2(\varphi_1) \right) \, . 
	\end{align}
	But now observe that this integral is exactly of the form dealt with in \eqref{eq:cos2thetaworking}, hence we may immediately identify its value as:
	\begin{align}
	& \frac{\|\mathbf{x}\|_2^2 \| \mathbf{y}\|_2^2 n^2(n-2)}{8 I(n-3) I(n-4)} \bigg( \left(\frac{1}{n} - \frac{1}{n+2}\right) (I(n-3) - I(n-1))I(n-4) - \nonumber \\
	& \qquad \qquad \qquad \qquad \qquad \qquad I(n-1)\left( I(n-4) - I(n-2) \right) \left(\frac{1}{n-2} - \frac{1}{n}\right) \bigg) \, .
	\end{align}
	
	Thus substituting our calculations back into Equation \eqref{eq:twoparts}, we obtain:
	\begin{align}
	& \mathbb{E}[X_{i}X_{j}] \nonumber \\
	= & \frac{\|\mathbf{x}\|_2^2 \|\mathbf{y}\|_2^2 n^2(n-2)}{4 I(n-3)I(n-4)} \bigg( \left(\frac{1}{n} - \frac{1}{n+2}\right) (I(n-3) - I(n-1))I(n-4) \left\lbrack \cos^2(\theta) + \frac{1}{2} \right\rbrack + \nonumber \\
	& \qquad\qquad \qquad \qquad \quad I(n-1)\left( I(n-4) - I(n-2) \right) \left(\frac{1}{n-2} - \frac{1}{n}\right) \left\lbrack \cos^2(\theta) - \frac{1}{2} \right\rbrack \bigg) \, .
	\end{align}
	The covariance term is obtained by subtracting off $\mathbb{E}[X_i]\mathbb{E}[X_i] = \langle \mathbf{x}, \mathbf{y}\rangle^2$. Now we sum over $m(m-1)$ covariance terms and take into account the normalization factor $\frac{1}{\sqrt{m}}$ for the Gaussian matrix entries.
	That gives the extra multiplicative term $\frac{m(m-1)}{m^{2}} = \frac{m-1}{m}$. Substituting in the definition of the $I$ function and simplifying then yields the quantity in the statement of the theorem, completing the proof.
\end{proof}

\subsection {Proof of Theorem \ref{hd_theorem}}\label{sec:S-rademacher-proof}

We obtain Theorem \ref{hd_theorem} through a sequence of smaller propositions. Broadly, the strategy is first to show that the estimators of Theorem \ref{hd_theorem} are unbiased (Proposition \ref{prop:hd-ojlt-unbiased}). An expression for the mean squared error of the estimator $\widehat{K}_m^{(1)}$ with one matrix block is then derived (Proposition \ref{prop:hd-ojlt-mse-1block}). Finally, a straightforward recursive formula for the mean squared error of the general estimator is derived (Proposition \ref{prop:hd-ojlt-mse-recurse}), and the result of the theorem then follows. 

\begin{proposition}\label{prop:hd-ojlt-unbiased}
	The estimator $\widehat{K}^{(k)}_m(\mathbf{x}, \mathbf{y})$ is unbiased, for all $k, n \in \mathbb{N}$, $m \leq n$, and $\mathbf{x}, \mathbf{y} \in \mathbb{R}^n$.
\end{proposition}
\begin{proof}
Notice first that since rows of $\mathbf{S}=\{s_{i,j}\}$ are orthogonal and are $L_{2}$-normalized, the matrix $\mathbf{S}$ is an isometry. Thus each block $\mathbf{SD}_{i}$
is also an isometry. Therefore it suffices to prove the claim for $k=1$.

Then, denoting by $\mathbf{J} = (J_1,\ldots,J_m)$ the indices of the randomly selected rows of $\mathbf{SD}_1$, note that the estimator $\widehat{K}^{(1)}_{m}(\mathbf{x},\mathbf{y})$ may be expressed in the form
\[
\widehat{K}^{(1)}_{m}(\mathbf{x},\mathbf{y}) = \frac{1}{m} \sum_{i=1}^m \left( \sqrt{n}(\mathbf{S}\mathbf{D}_1)_{J_i}\mathbf{x} \times \sqrt{n}(\mathbf{S}\mathbf{D}_1)_{J_i}\mathbf{y} \right)\, ,
\]
where $(\mathbf{SD}_1)_i$ is the $i^\mathrm{th}$ row of $\mathbf{SD}_1$.
Since each of the rows of $\mathbf{SD}_1$ has the same marginal distribution, it suffices to demonstrate that $\mathbb{E}[\mathbf{y}^T \mathbf{D}_1 \mathbf{S}_1^\top \mathbf{S}_1 \mathbf{D}_1 \mathbf{x}] = \frac{\mathbf{x}^{\top}\mathbf{y}}{n}$, where $\mathbf{S}_1$ is the first row of $\mathbf{S}$. Now note
\begin{align*}
\mathbb{E}[\mathbf{y}^\top \mathbf{D} \mathbf{S}_1^\top \mathbf{S}_1 \mathbf{D} \mathbf{x}]\!  =\! \frac{1}{n} \mathbb{E}\left[ \sum_{i=1}^n y_i d_i \times \sum_{i=1}^n x_i d_i\right] 
 \!= \!\frac{1}{n} \mathbb{E}\left[\sum_{i=1}^n x_i y_i d_i^2\right] + \mathbb{E}\left[\sum_{i \not= j} x_i y_j d_i d_j\right]
 \!= \frac{\mathbf{x}^{\top}\mathbf{y}}{n},
\end{align*}
where $d_i = \mathbf{D}_{ii}$ are iid Rademacher random variables, for $i=1,\ldots,n$.
\end{proof}

With Proposition \ref{prop:hd-ojlt-unbiased} in place, the mean square error for the estimator $\widehat{K}^{(1)}_m$ using one matrix block can be derived.

\begin{proposition}
	\label{prop:hd-ojlt-mse-1block}
	The MSE of the single $\mathbf{SD}^\mathcal{(R)}$-block $m$-feature estimator $\widehat{K}^{(1)}_m(\mathbf{x},\mathbf{y})$ for $\langle \mathbf{x}, \mathbf{y}\rangle$ using the \textit{without replacement} row sub-sampling strategy 
	is
	\[
	\mathrm{MSE}(\widehat{K}^{(1)}_m(\mathbf{x},\mathbf{y})) = \frac{1}{m}\left(\frac{n-m}{n-1} \right)\left( \|\mathbf{x}\|^2\|\mathbf{y}\|^2 + \langle \mathbf{x}, \mathbf{y} \rangle^2 - 2\sum_{i=1}^n x_i^2 y_i^2 \right) \, .
	\]
\end{proposition}

\begin{proof}
	First note that since $\widehat{K}^{(1)}_m(\mathbf{x},\mathbf{y})$ is unbiased, the mean squared error is simply the variance of this estimator. Secondly, denoting the indices of the $m$ randomly selected rows by $\mathbf{J} = (J_1, \ldots, J_m)$, by conditioning on $\mathbf{J}$ we obtain the following:
	\begin{align*}
	 &\Var{\widehat{K}^{(1)}_m(\mathbf{x},\mathbf{y})} 
	= \\
	&\frac{n^2}{m^2} \left(\expectation{\Var{\sum_{p=1}^m (\mathbf{SDx})_{J_p} (\mathbf{SDy})_{J_p} \Bigg| \mathbf{J}}} + \Var{\expectation{\sum_{p=1}^m (\mathbf{SDx})_{J_p} (\mathbf{SDy})_{J_p} \Bigg| \mathbf{J}}} \right) .
	\end{align*}
	
	Now note that the conditional expectation in the second term is constant as a function of $J$, since conditional on whichever rows are sampled, the resulting estimator is unbiased. Taking the variance of this constant therefore causes the second term to vanish. Now consider the conditional variance that appears in the first term:
	\begin{align*}
	 \Var{\sum_{p=1}^m (\mathbf{SD}\mathbf{x})_{J_p} (\mathbf{SD}\mathbf{y})_{J_p} \Bigg| \mathbf{J}}
	= & \sum_{p=1}^m \sum_{p^\prime = 1}^m \Cov{(\mathbf{SD}\mathbf{x})_{J_m} (\mathbf{SD}\mathbf{y})_{J_p}}{(\mathbf{SD}\mathbf{x})_{J_{p^\prime}} (\mathbf{SD}\mathbf{y})_{J_{p^\prime}} \big| \mathbf{J}} \\
	= & \sum_{p, p^\prime = 1}^m \sum_{i,j,k,l = 1}^n s_{J_p i} s_{J_p j} s_{J_{p^\prime} k}s_{J_{p^\prime} l} x_i y_j x_k y_l \Cov{d_i d_j}{d_k d_l} \, ,
	\end{align*}
	where we write $\mathbf{D} = \mathrm{Diag}(d_1, \ldots, d_n)$. Now note that $\Cov{d_i d_j}{d_k d_l}$ is non-zero iff $i,j$ are distinct, and $\{i,j\} = \{k,l\}$, in which case the covariance is $1$. We therefore obtain:
	\begin{align*}
	 &\Var{\sum_{p=1}^m (\mathbf{SD}\mathbf{x})_{J_p} (\mathbf{SD}\mathbf{y})_{J_p} \Bigg| \mathbf{J}}
	= \\ & \sum_{p, p^\prime = 1}^m \sum_{i \not= j}^n \left(s_{J_p i} s_{J_p j} s_{J_{p^\prime} i}s_{J_{p^\prime} j} x^2_i y^2_j + s_{J_p i} s_{J_p j} s_{J_{p^\prime} j}s_{J_{p^\prime} i} x_i y_j x_j y_i \right) .
	\end{align*}
	Substituting this expression for the conditional variance into the decomposition of the MSE of the estimator, we obtain the result of the theorem:
	\begin{align*}
	\Var{\widehat{K}^{(1)}_m(\mathbf{x},\mathbf{y})}
	= & \frac{n^2}{m^2} \expectation{\sum_{p, p^\prime = 1}^m \sum_{i \not= j}^n \left(s_{J_p i} s_{J_p j} s_{J_{p^\prime} i}s_{J_{p^\prime} j} x^2_i y^2_j + s_{J_p i} s_{J_p j} s_{J_{p^\prime} j}s_{J_{p^\prime} i} x_i y_j x_j y_i \right)} \\
	= & \frac{n^2}{m^2} \sum_{p, p^\prime = 1}^m \sum_{i \not= j}^n \left( x^2_i y^2_j + x_i x_j y_i y_j \right) \expectation{s_{J_p i} s_{J_p j} s_{J_{p^\prime} i}s_{J_{p^\prime} j}} \, .
	\end{align*}
	
	We now consider the law on the index variables $\mathbf{J} = (J_1,\ldots,J_m)$ induced by the sub-sampling strategy without replacement 
	to evaluate the expectation in this last term. If $p=p^\prime$, the integrand of the expectation is deterministically $1/n^2$. If $p\not=p^\prime$, then we obtain:
	\begin{align*}
	\expectation{s_{J_p i} s_{J_p j} s_{J_{p^\prime} i}s_{J_{p^\prime} j}} = & \expectation{ s_{J_p i} s_{J_p j} \mathbb{E}\left\lbrack s_{J_{p^\prime} i}s_{J_{p^\prime} j} \big | J_p \right\rbrack} \\
	= & \mathbb{E}\bigg\lbrack s_{J_p i} s_{J_p j} \bigg\lbrack \left( \frac{1}{n}\left( \frac{n/2-1}{n-1}\right) - \frac{1}{n}\left( \frac{n/2}{n-1} \right) \right) \mathbbm{1}_{\{s_{J_p i} s_{J_p j} = 1/n \}  } + \\
	&\qquad \left( \frac{1}{n}\left( \frac{n/2}{n-1}\right) - \frac{1}{n}\left( \frac{n/2-1}{n-1} \right) \right) \mathbbm{1}_{ \{s_{J_p i} s_{J_p j} = -1/n\} } \bigg\rbrack \bigg\rbrack \\
	= & \frac{1}{n(n-1)} \mathbb{E}\left\lbrack s_{J_p i} s_{J_p j} \left( \mathbbm{1}_{ \{s_{J_p i} s_{J_p j} = -1/n\} }- \mathbbm{1}_{ \{s_{J_p i} s_{J_p j} = 1/n\} } \right) \right\rbrack \\
	= & \frac{1}{n^2(n-1)} \, ,
	\end{align*}
	where we have used the fact that the products $s_{J_p i} s_{J_p j}$ and $s_{J_{p^\prime} i} s_{J_{p^\prime} j}$ take values in $\{\pm1/n\}$, and because distinct rows of $\mathbf{S}$ are orthogonal, the marginal probability of each of the two values is $1/2$. A simple adjustment, using almost-sure distinctness of $J_p$ and $J_{p^\prime}$, yields the conditional probabilities needed to evaluate the conditional expectation that appears in the calculation above.
	
	Substituting the values of these expectations back into the expression for the variance of $\widehat{K}_m^{(1)}(\mathbf{x}, \mathbf{y})$ then yields
	\begin{align*}
	\mathrm{Var}(\widehat{K}^{(1)}_m(\mathbf{x}, \mathbf{y})) = & \frac{n^2}{m^2}  \sum_{i \not= j}^n \left( x^2_i y^2_j + x_i x_j y_i y_j \right) \left(m\times \frac{1}{n^2} - m(m-1) \times \frac{1}{n^2(n-1)}  \right) \\
	= & \frac{1}{m}\left(1 -\frac{m-1}{n-1}\right) \sum_{i \not= j}^n \left( x^2_i y^2_j + x_i x_j y_i y_j \right) \\
	= & \frac{1}{m}\left(1 -\frac{m-1}{n-1}\right) \left(\sum_{i, j=1}^n( x^2_i y^2_j + x_i x_j y_i y_j )- 2\sum_{i=1}^n x_i^2 y_i^2 \right) \\
	= & \frac{1}{m}\left(\frac{n-m}{n-1}\right) \left( \left\langle \mathbf{x}, \mathbf{y} \right\rangle^2 + \|\mathbf{x}\|^2 \|\mathbf{y}\|^2 - 2\sum_{i=1}^n x_i^2y_i^2 \right) \, ,
	\end{align*}
	as required.

\end{proof}

We now turn our attention to the following recursive expression for the mean squared error of a general estimator.

\begin{proposition}
	\label{prop:hd-ojlt-mse-recurse}
	Let $k \geq 2$. We have the following recursion for the MSE of $K^{(k)}_m(x,y)$:
	\[
	\mathrm{MSE}(\widehat{K}^{(k)}_m(\mathbf{x},\mathbf{y})) = \mathbb{E}\left\lbrack \mathrm{MSE}\left(\widehat{K}^{(k-1)}_m(\mathbf{SD}_1\mathbf{x},\mathbf{SD}_1\mathbf{y}) | \mathbf{D}_1 \right) \right\rbrack \, .
	\]
\end{proposition}

\begin{proof}
	The result follows from a straightforward application of the law of total variance, conditioning on the matrix $\mathbf{D}_1$. Observe that
	\begin{align*}
	\mathrm{MSE}(\widehat{K}^{(k)}_m(\mathbf{x},\mathbf{y}))
	&=  \mathrm{Var}(\widehat{K}^{(k)}_m(\mathbf{x},\mathbf{y}))\\
	&= \expectation{\Var{\widehat{K}^{(k)}_m(\mathbf{x},\mathbf{y})\Big| \mathbf{D}_1}} + \Var{\expectation{\widehat{K}^{(k)}_m(\mathbf{x},\mathbf{y}) \Big|  \mathbf{D}_1}} \\
	&\- \hspace{-5em} = \expectation{\Var{\widehat{K}^{(k-1)}_m(\mathbf{SD}_1\mathbf{x},\mathbf{SD}_1\mathbf{y}) \Big| \mathbf{D}_1}} + \Var{\expectation{\widehat{K}^{(k-1)}_m(\mathbf{SD}_1\mathbf{x},\mathbf{SD}_1\mathbf{y}) \Big| \mathbf{D}_1}}.
	\end{align*}
	But examining the conditional expectation in the second term, we observe
	\[
	\expectation{\widehat{K}^{(k-1)}_m(\mathbf{SD}_1\mathbf{x},\mathbf{SD}_1\mathbf{y}) \Big| \mathbf{D}_1} = \langle \mathbf{SD}_1\mathbf{x}, \mathbf{SD}_1\mathbf{y} \rangle \quad \text{almost surely} \, ,
	\]
	by unbiasedness of the estimator, and since $\mathbf{SD}_1$ is orthogonal almost surely, this is equal to the (constant) inner product $\langle \mathbf{x}, \mathbf{y} \rangle$ almost surely. This conditional expectation therefore has $0$ variance, and so the second term in the expression for the MSE above vanishes, which results in the statement of the proposition.
\end{proof}

With these intermediate propositions established, we are now in a position to prove Theorem \ref{hd_theorem}. In order to use the recursive result of Proposition \ref{prop:hd-ojlt-mse-recurse}, we require the following lemma.

\begin{lemma}\label{lemma:hd-mse-help}
	For all $x, y, \in \mathbb{R}^n$, we have
	\begin{align*}
	\expectation{\sum_{i=1}^n (\mathbf{SDx})^2_i (\mathbf{SDy})^2_i} = \frac{1}{n}\left( \|\mathbf{x}\|^2\|\mathbf{y}\|^2 + 2\langle \mathbf{x}, \mathbf{y} \rangle^2 - 2\sum_{i=1}^n x_i^2 y_i^2 \right) \, .
	\end{align*}
\end{lemma}
\begin{proof}
	The result follows by direct calculation. Note that
	\begin{align*}
	\expectation{\sum_{i=1}^n \left(\mathbf{SDx}\right)^2_i \left(\mathbf{SDy}\right)^2_i} & = n \expectation{\left(\sum_{a=1}^ns_{1a}d_ax_a\right)^2 \left(\sum_{a=1}s_{1a}d_ay_a\right)^2}\\
	& = n \sum_{i,j,k,l = 1}^n s_{1i}s_{1j}s_{1k}s_{1l} x_i x_j y_k y_l \expectation{d_i d_j d_k d_l} \, ,
	\end{align*}
	where the first inequality follows since the $n$ summands indexed by $i$ in the initial expectation are identically distributed. Now note that the expectation $\expectation{d_i d_j d_k d_l}$ is non-zero iff $i=j=k=l$, or $i=j \not= k =l$, or $i=k\not=j=l$, or $i=l\not=k=l$; in all such cases, the expectation takes the value $1$. Substituting this into the above expression and collecting terms, we obtain
	\begin{align*}
	\expectation{\sum_{i=1}^n \left(\mathbf{SDx}\right)^2_i \left(\mathbf{SDy}\right)^2_i} & = \frac{1}{n} \left( \sum_{i = 1}^n x^2_i y^2_i + \sum_{i \not= j} x_i^2 y_i^2 + 2 \sum_{i \not= j} x_i x_j y_i y_j \right) \\
	& = \frac{1}{n} \left( \sum_{i, j=1}^n x^2_i y^2_j + 2 \sum_{i, j=1}^n x_i x_j y_i y_j - 2\sum_{i=1}^n x_i^2 y_i^2 \right)
	\, ,
	\end{align*}
	from which the statement of the lemma follows immediately.
\end{proof}

\begin{proof}[Proof of Theorem \ref{hd_theorem}]
	Recall that we aim to establish the following general expression for $k \geq 1$:
	{\small
	\begin{align*}
	&\mathrm{MSE}(\widehat{K}^{(k)}_m(\mathbf{x},\mathbf{y}))\! = \\
	&\frac{1}{m}\!\left(\frac{n\!-\!m}{n\!-\!1}\right)\left(\! ((\mathbf{x}^{\top} \mathbf{y})^2 \!+\! \|\mathbf{x}\|^2\|\mathbf{y}\|^2) \!+\!  \sum_{r=1}^{k-1} \frac{(-1)^r 2^{r}}{n^r} (2(\mathbf{x}^{\top}\mathbf{y})^2\!+\!\|\mathbf{x}\|^2\|\mathbf{y}\|^2 ) \!+\! \frac{(-1)^{k}2^k}{n^{k-1}} \sum_{i=1}^n x_i^2 y_i^2 \!\right).
	\end{align*}}
	We proceed by induction. The case $k=1$ is verified by Proposition \ref{prop:hd-ojlt-mse-1block}. For the inductive step, suppose the result holds for some $k \in \mathbb{N}$. Then observe by Proposition \ref{prop:hd-ojlt-mse-recurse} and the induction hypothesis, we have
	\begin{align*}
	\mathrm{MSE}&(\widehat{K}^{(k+1)}_m(\mathbf{x},\mathbf{y}))  = \mathbb{E}\left\lbrack \mathrm{MSE}\left(\widehat{K}^{(k-1)}_m(\mathbf{SD}_1\mathbf{x},\mathbf{SD}_1\mathbf{y}) | \mathbf{D}_1 \right) \right\rbrack \\
	& = \frac{1}{m}\left(\frac{n-m}{n-1}\right)\bigg( ((\mathbf{x}^{\top} \mathbf{y})^2 + \|\mathbf{x}\|^2\|\mathbf{y}\|^2) +  \sum_{r=1}^{k-1} \frac{(-1)^r 2^{r}}{n^r} (2(\mathbf{x}^{\top}\mathbf{y})^2+\|\mathbf{x}\|^2\|\mathbf{y}\|^2 ) \\
	& \qquad\qquad\qquad\qquad + \frac{(-1)^{k}2^k}{n^{k-1}} \sum_{i=1}^n \mathbb{E}\left\lbrack (\mathbf{SD}_1 \mathbf{x})_i^2 (\mathbf{SD}_1\mathbf{y})_i^2 \right\rbrack \bigg),
	\end{align*}
	where we have used that $\mathbf{SD}_1$ is almost surely orthogonal, and therefore $\|\mathbf{SD}_1 \mathbf{x}\|^2 = \|\mathbf{x}\|^2$ almost surely, $\|\mathbf{SD}_1 \mathbf{y}\|^2 = \|\mathbf{y}\|^2$ almost surely, and $\langle \mathbf{SD_1}\mathbf{x}, \mathbf{SD}_1\mathbf{y} \rangle = \langle \mathbf{x}, \mathbf{y} \rangle$ almost surely. Applying Lemma \ref{lemma:hd-mse-help} to the remaining expectation and collecting terms yields the required expression for $\mathrm{MSE}(\widehat{K}^{(k+1)}_m(\mathbf{x},\mathbf{y}))$, and the proof is complete.
\end{proof}

\subsection{Proof of Lemma \ref{complexity_lemma}}

\begin{proof}
Consider the last block $\mathbf{H}$ that is sub-sampled. Notice that if rows $\mathbf{r}^{1}$ and $\mathbf{r}^{2}$ of $\mathbf{H}$ of indices $i$ and $\frac{n}{2} + i$ 
are chosen then from the recursive definition of $\mathbf{H}$ we conclude that 
$(\mathbf{r}^{2})^{\top}\mathbf{x} = (\mathbf{r}^{1}_{1})^{\top}\mathbf{x} - (\mathbf{r}^{1}_{2})^{\top}\mathbf{x}$, where $\mathbf{r}^{1}_{1},\mathbf{r}^{1}_{2}$
stand for the first and second half of $\mathbf{r}^{1}$ respectively. Thus computations of $(\mathbf{r}^{1})^{\top}\mathbf{x}$ can be reused to compute
both $(\mathbf{r}^{1})^{\top}\mathbf{x}$ and $(\mathbf{r}^{2})^{\top}\mathbf{x}$ in time $n+O(1)$ instead of $2n$.
If we denote by $r$ the expected number of pairs of rows $(i, \frac{n}{2}+i)$ that are chosen by the random sampling mechanism, then we see that by applying the trick above
for all the $r$ pairs, we obtain time complexity $O((k-1)n\log(n)+n(m-2r)+nr+r)$, where: $O((k-1)n\log(n))$ is the time required to compute first $(k-1)$ $\mathbf{HD}$ blocks
(with the use of Walsh-Hadamard Transform), $O(n(m-2r))$ stands for time complexity of the brute force computations for these rows that were not coupled in the last block 
and $O(nr+r)$ comes from the above trick applied to all $r$ aforementioned pairs of rows.
Thus, to obtain the first term in the min-expression on time complexity from the statement of the lemma, it remains to show that 

\begin{equation}
\label{simple_eq}
\mathbb{E}[r] = \frac{(m-1)m}{2(n-1)}.
\end{equation}

But this is straightforward. Note that the number of the $m$-subsets of the set of all $n$ rows that contain some fixed rows of indices $i_{1}$, $i_{2}$ ($i_{1} \neq i_{2}$)
is ${n-2 \choose m-2}$. Thus for any fixed pair of rows of indices $i$ and $\frac{n}{2}+i$ the probability that these two rows will be selected is exactly
$p_{succ} = \frac{{n-2 \choose m-2}}{{n \choose m}} = \frac{(m-1)m}{(n-1)n}$. Equation \ref{simple_eq} comes from the fact that clearly: 
$\mathbb{E}[r] = \frac{n}{2}p_{succ}$.
Thus we obtain the first term in the min-expression from the statement of the lemma.
The other one comes from the fact that one can always do all the computations by calculating $k$ times Walsh-Hadamard transformation. That completes the proof.

\end{proof}

\subsection{Proof of Theorem \ref{thm:hybrid-mse}}\label{sec:S-hybrid-proof}
The proof of Theorem \ref{thm:hybrid-mse} follows a very similar structure to that of Theorem \ref{hd_theorem}; we proceed by induction, and may use the results of Proposition \ref{prop:hd-ojlt-mse-recurse} to set up a recursion. We first show unbiasedness of the estimator (Proposition \ref{prop:hd-complex-unbiased}), and then treat the base case of the inductive argument  (Proposition \ref{prop:hd-complex-1block}). We prove slightly more general statements than needed for Theorem \ref{thm:hybrid-mse}, as this will allow us to explore the fully complex case in \S \ref{sec:proof-full-complex}.

\begin{proposition}\label{prop:hd-complex-unbiased}
	The estimator $K^{\mathcal{H}, (k)}_{m}(\mathbf{x}, \mathbf{y})$ is unbiased for all $k, n \in \mathbb{N}$, $m \leq n$, and $\mathbf{x}, \mathbf{y} \in \mathbb{C}^n$ with $\langle \overline{\mathbf{x}}, \mathbf{y} \rangle \in \mathbb{R}$; in particular, for all $\mathbf{x}, \mathbf{y} \in \mathbb{R}$.
\end{proposition}
\begin{proof}
	Following a similar argument to the proof of Proposition \ref{prop:hd-ojlt-unbiased}, note that it is sufficient to prove the claim for $k=1$, since each $\mathbf{SD}$ block is unitary, and hence preserves the Hermitian product $\langle \overline{\mathbf{x}}, \mathbf{y} \rangle$.
	
	Next, note that the estimator can be written as a sum of identically distributed terms:
	\[
	\widehat{K}^{\mathcal{H}, (1)}_{m}(\mathbf{x},\mathbf{y}) = \frac{n}{m} \sum_{i=1}^m \mathrm{Re}\left( (\mathbf{S}\overline{\mathbf{D}}_1\overline{\mathbf{x}})_{J_i} \times (\mathbf{S}\mathbf{D}_1\mathbf{y})_{J_i} \right)\, .
	\]
	The terms are identically distributed since the index variables $J_i$ are marginally identically distributed, and the rows of $\mathbf{SD}_1$ are marginally identically distributed (the elements of a row are iid $\mathrm{Unif}(S^1)/\sqrt{n}$). Now note
	\begin{align*}
	&\expectation{\mathrm{Re}\left( (\mathbf{S}\overline{\mathbf{D}}_1\overline{\mathbf{x}})_{J_i} \times (\mathbf{S}\mathbf{D}_1\mathbf{y})_{J_i} \right)}  = \frac{1}{n} \mathbb{E}\left[ \sum_{i=1}^n y_i d_i \times \sum_{i=1}^n \overline{x}_i \overline{d}_i\right] \\
&	= \frac{1}{n}\mathbb{E}\left[\sum_{i=1}^n \overline{x}_i y_i d_i \overline{d}_i \right] + \mathbb{E}\left[\sum_{i \not= j} \overline{x}_i y_j \overline{d}_i d_j\right]
	= \frac{1}{n}\langle \overline{\mathbf{x}}, \mathbf{y}\rangle \, ,
	\end{align*}
	where $d_i = \mathbf{D}_{ii} \overset{iid}{\sim} \mathrm{Unif}(S^1)$ for $i=1,\ldots,n$. This immediately yields $\expectation{\widehat{K}^{\mathcal{H},(1)}_{m}(\mathbf{x},\mathbf{y})} = \langle \overline{\mathbf{x}}, \mathbf{y} \rangle$, as required.
\end{proof}

We now derive the base case for our inductive proof, again proving a slightly more general statement then necessary for Theorem \ref{thm:hybrid-mse}.

\begin{proposition}\label{prop:hd-complex-1block}
	Let $\mathbf{x}, \mathbf{y} \in \mathbb{C}^n$ such that $\langle \overline{\mathbf{x}}, \mathbf{y} \rangle \in \mathbb{R}$. The MSE of the single complex $\mathbf{SD}$-block $m$-feature estimator $K^{\mathcal{H}, (1)}_{m}(\mathbf{x},\mathbf{y})$ for $\langle \overline{\mathbf{x}}, \mathbf{y} \rangle$ is
	\[
	\mathrm{MSE}(\widehat{K}^{\mathcal{H}, (1)}_{m}(\mathbf{x},\mathbf{y})) = \frac{1}{2m}\left(\frac{n-m}{n-1}\right)\left( \langle \overline{\mathbf{x}}, \mathbf{x} \rangle \langle \overline{\mathbf{y}}, \mathbf{y} \rangle + \langle \overline{\mathbf{x}}, \mathbf{y} \rangle^2 - \sum_{r=1}^n |x_r|^2 |y_r|^2 -\sum_{r=1}^n \mathrm{Re}(\overline{x}_r^2 y_r^2) \right) \, .
	\]
\end{proposition}

\begin{proof}
	The proof is very similar to that of Proposition \ref{prop:hd-ojlt-mse-1block}. By the unbiasedness result of Proposition \ref{prop:hd-complex-unbiased}, the mean squared error of the estimator is simply the variance. We begin by conditioning on the random index vector $\mathbf{J}$ selected by the sub-sampling procedure.
	\[
	\widehat{K}^{\mathcal{H}, (1)}_{m}(\mathbf{x},\mathbf{y})) = \frac{1}{M}\mathrm{Re}\left( \langle \sqrt{n}(\mathbf{S}\overline{\mathbf{D}}_1\overline{\mathbf{x}})_\mathbf{J}, \sqrt{n}(\mathbf{SDy})_\mathbf{J} \rangle \right) \, ,
	\]
	where again $\mathbf{J}$ is a set of uniform iid indices from $1,\ldots,n$, and the bar over $D$ represents complex conjugation. Since the estimator is again unbiased, its MSE is equal to its variance. First conditioning on the index set $\mathbf{J}$, as for Proposition \ref{prop:hd-complex-1block}, we obtain
	{\small
	\begin{align*}
	& \Var{\widehat{K}^{\mathcal{H}, (1)}_{m}(x,y)} \\
	= & \frac{n^2}{m^2} \!\left(\!\expectation{\Var{\mathrm{Re}\left(\sum_{p=1}^m (\mathbf{S}\overline{\mathbf{D}}_1\overline{\mathbf{x}})_{J_p} (\mathbf{SD}_1\mathbf{y})_{J_p}\right) \!\Bigg| \mathbf{J}}} \!+\! \Var{\expectation{\mathrm{Re}\left(\sum_{p=1}^m (\mathbf{S}\overline{\mathbf{D}}_1\overline{\mathbf{x}})_{J_p} (\mathbf{SD}_1\mathbf{y})_{J_p} \right)\Bigg| \!\mathbf{J}}} \!\right).
	\end{align*}}
	Again, the second term vanishes as the conditional expectation is constant as a function of $\mathbf{J}$, by unitarity of $\mathbf{SD}$. Turning attention to the conditional variance expression in the first term, we note
	\begin{align*}
	&\Var{\mathrm{Re}\left(\sum_{p=1}^m (\mathbf{S}\overline{\mathbf{D}}_1\overline{\mathbf{x}})_{J_p} (\mathbf{SD}_1\mathbf{y})_{J_p}\right) \Bigg| \mathbf{J}}
	= \\ &\sum_{p, p^\prime = 1}^m \sum_{i,j,k,l=1}^n s_{J_p i}s_{J_{p} j} s_{J_{p^\prime} k}s_{J_{p^\prime} l} \Cov{\mathrm{Re}(\overline{d}_{i}\overline{x}_id_{j}y_j)}{\mathrm{Re}(\overline{d}_{k} \overline{x}_k d_{l} y_l)} \, .
	\end{align*}
	Now note that the covariance term is non-zero iff $i,j$ are distinct, and $\{i,j\} = \{k,l\}$. We therefore obtain
	{\small \begin{align*}
	& \Var{\mathrm{Re}\left(\sum_{p=1}^m (\mathbf{S}\overline{\mathbf{D}}\overline{\mathbf{x}})_{J_p} (\mathbf{SDy})_{J_p}\right) \Bigg| \mathbf{J}} \\
	= & \!\!\sum_{p, p^\prime = 1}^m \sum_{i \not= j}^n s_{J_p i}s_{J_{p} j} s_{J_{p^\prime} i}s_{J_{p^\prime} j} \left( \Cov{\mathrm{Re}(\overline{d}_{i}\overline{x}_id_{j}y_j)}{\mathrm{Re}(\overline{d}_{i} \overline{x}_i d_{j} y_j)} \!+\! 
	\Cov{\mathrm{Re}(\overline{d}_{i}\overline{x}_id_{j}y_j)}{\mathrm{Re}(\overline{d}_{j} \overline{x}_j d_{i} y_i)}\right)
	\end{align*}}
	First consider the term $\Cov{\mathrm{Re}(\overline{d}_{i}\overline{x}_id_{j}y_j)}{\mathrm{Re}(\overline{d}_{i} \overline{x}_i d_{j} y_j)}$. The random variable $\overline{d}_{i}\overline{x}_id_{j}y_j$ is distributed uniformly on the circle in the complex plane centered at the origin with radius $|\overline{x}_iy_j|$. Therefore the variance of its real part is
	\[
	\Cov{\mathrm{Re}(\overline{d}_{i}\overline{x}_id_{j}y_j)}{\mathrm{Re}(\overline{d}_{i} \overline{x}_i d_{j} y_j)} = \frac{1}{2}|\overline{x}_i y_j|^2 = \frac{1}{2} x_i \overline{x}_i y_j \overline{y}_j \, .
	\]
	For the second covariance term, we perform an explicit calculation. Let $Z = e^{i\theta} = \overline{d}_{i} d_{j}$. Then we have
	\begin{align*}
	&\Cov{\mathrm{Re}(\overline{d}_{i}\overline{x}_id_{j}y_j)}{\mathrm{Re}(\overline{d}_{j} \overline{x}_j d_{i} y_i)}
	=  \Cov{\mathrm{Re}(Z\overline{x}_iy_j)}{\mathrm{Re}(\overline{Z} \overline{x}_j  y_i)} \\
	&=  \Cov{\cos(\theta)\mathrm{Re}(\overline{x}_iy_j) - \sin(\theta)\mathrm{Im}(\overline{x}_i y_j)}{\cos(\theta)\mathrm{Re}(\overline{x}_j y_i) + \sin(\theta) \mathrm{Im
		}(\overline{x}_j y_i)} \\
	&=  \frac{1}{2} \left( \mathrm{Re}(\overline{x}_i y_j)\mathrm{Re}(\overline{x}_j y_i) - \mathrm{Im}(\overline{x}_iy_j)\mathrm{Im}(\overline{x}_j y_i) \right)  \, ,
	\end{align*}
	with the final equality following since the angle $\theta$ is uniformly distributed on $[0,2\pi]$, and standard trigonometric integral identities. We recognize the bracketed terms in the final line as the real part of the product $\overline{x}_i \overline{x}_j y_i y_j$. Substituting these into the expression for the conditional variance obtained above, we have
	{\small \begin{align*}
	\Var{\mathrm{Re}\left(\sum_{p=1}^m (\mathbf{S}\overline{\mathbf{D}}\mathbf{x})_{J_p} (\mathbf{SDy})_{J_p}\right) \Bigg| \mathbf{J}}
	= \sum_{p, p^\prime = 1}^m \sum_{i\not= j}^n s_{J_p i}s_{J_{p} j} s_{J_{p^\prime} i}s_{J_{p^\prime} j} \frac{1}{2} \left( x_i \overline{x}_i y_j \overline{y}_j + \mathrm{Re}(\overline{x}_i \overline{x}_j y_i y_j) \right) .
	\end{align*}}
	Now taking the expectation over the index variables $\mathbf{J}$, we note that as in the proof of Proposition \ref{prop:hd-ojlt-mse-1block}, the expectation of the term $s_{J_p i}s_{J_{p} j} s_{J_{p^\prime} i}s_{J_{p^\prime} j}$ is $1/n^2$ when $p=p^\prime$, and $1/(n^2(n-1))$ otherwise. Therefore we obtain
	\begin{align*}
	&\Var{\widehat{K}^{\mathcal{H}, (1)}_{m}(\mathbf{x},\mathbf{y})} 
	=  \frac{n^2}{m^2} \left( \left(\frac{m}{n^2} + \frac{m(m-1)}{n^2(n-1)}\right) \frac{1}{2} \sum_{i \not= j}^n \left( x_i \overline{x}_i y_j \overline{y}_j + \mathrm{Re}( \overline{x}_i \overline{x}_j y_i y_j) \right) \right) \\
	&=  \frac{1}{2m}\left(\frac{n-m}{n-1} \right) \left(\sum_{i \not= j}^n \left( x_i \overline{x}_i y_j \overline{y}_j + \mathrm{Re}(\overline{x}_i \overline{x}_j y_i y_j) \right) \right) \\
	&=  \frac{1}{2m} \left(\frac{n-m}{n-1} \right) \left(\sum_{i, j = 1}^n \left( x_i \overline{x}_i y_j \overline{y}_j + \mathrm{Re}(\overline{x}_i \overline{x}_j y_i y_j) \right) - \sum_{i=1}^n (x_i \overline{x}_i y_i \overline{y}_i + \mathrm{Re}(\overline{x}_i \overline{x}_i y_i y_i)) \right) \\
	&=  \frac{1}{2m} \left(\frac{n-m}{n-1} \right) \left(\langle \overline{\mathbf{x}}, \mathbf{x} \rangle \langle \overline{\mathbf{y}}, \mathbf{y} \rangle + \langle \overline{\mathbf{x}}, \mathbf{y} \rangle^2 - \sum_{i=1}^n (x_i \overline{x}_i y_i \overline{y}_i + \mathrm{Re}(\overline{x}_i \overline{x}_i y_i y_i)) \right) \, ,
	\end{align*}
	where in the final equality we have used the assumption that $\langle \overline{\mathbf{x}}, \mathbf{y} \rangle \in \mathbb{R}$.
\end{proof}

We are now in a position to prove Theorem \ref{thm:hybrid-mse} by induction, using Proposition \ref{prop:hd-complex-1block} as a base case, and Proposition \ref{prop:hd-ojlt-mse-recurse} for the inductive step.

\begin{proof}[Proof of Theorem \ref{thm:hybrid-mse}]
	
		Recall that we aim to establish the following general expression for $k \geq 1$:
		{\small
		\begin{align*}
		\mathrm{MSE}(\widehat{K}^{\mathcal{H}, (k)}_m(\mathbf{x},\mathbf{y})) = \frac{1}{2m}\left(\frac{n-m}{n-1}\right)\bigg( & ((\mathbf{x}^{\top} \mathbf{y})^2 + \|\mathbf{x}\|^2\|\mathbf{y}\|^2) + \\
		 & \sum_{r=1}^{k-1} \frac{(-1)^r 2^{r}}{n^r} (2(\mathbf{x}^{\top}\mathbf{y})^2+\|\mathbf{x}\|^2\|\mathbf{y}\|^2 ) + \frac{(-1)^{k}2^k}{n^{k-1}} \sum_{i=1}^n x_i^2 y_i^2 \bigg) \, .
		\end{align*}}
		We proceed by induction. The case $k=1$ is verified by Proposition \ref{prop:hd-complex-1block}, and by noting that in the expression obtained in Proposition \ref{prop:hd-complex-1block}, we have
		\[
		\sum_{i=1}^n x_i \overline{x}_i y_i \overline{y}_i = \mathrm{Re}(\overline{x}_i \overline{x}_i y_i y_i) = \sum_{i=1}^n x_i^2 y_i^2 \, .
		\]
		For the inductive step, suppose the result holds for some $k \in \mathbb{N}$. Then observe by Proposition \ref{prop:hd-ojlt-mse-recurse} and the induction hypothesis, we have, for $\mathbf{x}, \mathbf{y} \in \mathbb{R}^n$:
		\begin{align*}
		&\mathrm{MSE}(\widehat{K}^{\mathcal{H}, (k+1)}_m(\mathbf{x},\mathbf{y}))  = \mathbb{E}\left\lbrack \mathrm{MSE}\left(\widehat{K}^{(k-1)}_m(\mathbf{SD}_1\mathbf{x},\mathbf{SD}_1\mathbf{y}) | \mathbf{D}_1 \right) \right\rbrack \\
		& = \frac{1}{2m}\left(\frac{n-m}{n-1}\right)\bigg( ((\mathbf{x}^{\top} \mathbf{y})^2 + \|\mathbf{x}\|^2\|\mathbf{y}\|^2) +  \sum_{r=1}^{k-1} \frac{(-1)^r 2^{r}}{n^r} (2(\mathbf{x}^{\top}\mathbf{y})^2+\|\mathbf{x}\|^2\|\mathbf{y}\|^2 ) \\
		& \qquad\qquad\qquad\qquad + \frac{(-1)^{k}2^k}{n^{k-1}} \sum_{i=1}^n \mathbb{E}\left\lbrack (\mathbf{SD}_1 \mathbf{x})_i^2 (\mathbf{SD}_1\mathbf{y})_i^2 \right\rbrack \bigg),
		\end{align*}
		where we have used that $\mathbf{SD}_1$ is almost surely orthogonal, and therefore $\|\mathbf{SD}_1 \mathbf{x}\|^2 = \|\mathbf{x}\|^2$ almost surely, $\|\mathbf{SD}_1 \mathbf{y}\|^2 = \|\mathbf{y}\|^2$ almost surely, and $\langle \mathbf{SD_1}\mathbf{x}, \mathbf{SD}_1\mathbf{y} \rangle = \langle \mathbf{x}, \mathbf{y} \rangle$ almost surely. Applying Lemma \ref{lemma:hd-mse-help} to the remaining expectation and collecting terms yields the required expression for $\mathrm{MSE}(\widehat{K}^{\mathcal{H}, (k+1)}_m(\mathbf{x},\mathbf{y}))$, and the proof is complete.

\end{proof}

\subsection{Proof of Corollary \ref{corr:4-pt-complex}}

The proof follows simply by following the inductive strategy of the proof of Theorem \ref{thm:hybrid-mse}, replacing the base case in Proposition \ref{prop:hd-complex-1block} with the following.

\begin{proposition}\label{prop:hybrid-1block}
	Let $\mathbf{x}, \mathbf{y} \in \mathbb{R}^n$. The MSE of the single hybrid $\mathbf{SD}$-block $m$-feature estimator $K^{\mathcal{H}, (1)}_{m}(\mathbf{x},\mathbf{y})$ using a diagonal matrix with entries $\mathrm{Unif}(\{1,-1,i,-i\})$, rather than $\mathrm{Unif}(S^1)$ for $\langle \mathbf{x}, \mathbf{y} \rangle$ is
	\[
	\mathrm{MSE}(\widehat{K}^{\mathcal{H}, (1)}_{m}(\mathbf{x},\mathbf{y})) = \frac{1}{2m}\left( \langle \overline{\mathbf{x}}, \mathbf{x} \rangle \langle \overline{\mathbf{y}}, \mathbf{y} \rangle + \langle \overline{\mathbf{x}}, \mathbf{y} \rangle^2 - 2\sum_{r=1}^n x_r^2 y_r^2 \right)\, .
	\]
\end{proposition}
\begin{proof}
	The proof of this proposition proceeds exactly as for Proposition \ref{prop:hd-complex-1block}; by following the same chain of reasoning, conditioning on the index set $\mathbf{J}$ of the sub-sampled rows, we arrive at
	\begin{align*}
	&\Var{\mathrm{Re}\left(\sum_{p=1}^m (\mathbf{S}\overline{\mathbf{D}}_1\overline{\mathbf{x}})_{J_p} (\mathbf{SD}_1\mathbf{y})_{J_p}\right) \Bigg| \mathbf{J}}
	= \\ &\sum_{p, p^\prime = 1}^m \sum_{i,j,k,l=1}^n s_{J_p i}s_{J_{p} j} s_{J_{p^\prime} k}s_{J_{p^\prime} l} \Cov{\mathrm{Re}(\overline{d}_{i}\overline{x}_id_{j}y_j)}{\mathrm{Re}(\overline{d}_{k} \overline{x}_k d_{l} y_l)} \, .
	\end{align*}
	Since we are dealing strictly with the case $\mathbf{x}, \mathbf{y} \in \mathbb{R}^n$, we may simplify this further to obtain
	\begin{align*}
	&\Var{\mathrm{Re}\left(\sum_{p=1}^m (\mathbf{S}\overline{\mathbf{D}}_1\overline{\mathbf{x}})_{J_p} (\mathbf{SD}_1\mathbf{y})_{J_p}\right) \Bigg| \mathbf{J}}
	= \\ &\sum_{p, p^\prime = 1}^m \sum_{i,j,k,l=1}^n s_{J_p i}s_{J_{p} j} s_{J_{p^\prime} k}s_{J_{p^\prime} l} x_i x_k y_i y_l \Cov{\mathrm{Re}(\overline{d}_{i}d_{j})}{\mathrm{Re}(\overline{d}_{k} d_{l})} \, .
	\end{align*}
	By calculating directly with the $d_i, d_j, d_k, d_l \sim \mathrm{Unif}(\{1, -1, i, -i \})$, we obtain
	\begin{align*}
	&\Var{\mathrm{Re}\left(\sum_{p=1}^m (\mathbf{S}\overline{\mathbf{D}}_1\overline{\mathbf{x}})_{J_p} (\mathbf{SD}_1\mathbf{y})_{J_p}\right) \Bigg| \mathbf{J}}
	= \\ &\frac{1}{2} \sum_{p, p^\prime = 1}^m \sum_{i\not= j}^n s_{J_p i}s_{J_{p} j} s_{J_{p^\prime} k}s_{J_{p^\prime} l} (x_i^2 y_j^2 + x_i x_j y_i y_j) \, ,
	\end{align*}
	exactly as in Proposition \ref{prop:hd-complex-1block}; following the rest of the argument of Proposition \ref{prop:hd-complex-1block} yields the result.
\end{proof}

The proof of the corollary now follows by applying the steps of the proof of Theorem \ref{thm:hybrid-mse}.

\subsection{Exploring Dimensionality Reduction with Fully-complex Random Matrices}\label{sec:proof-full-complex}

In this section, we briefly explore the possibility of using $\mathbf{SD}$-product matrices in which all the random diagonal matrices are complex-valued. Following on from the ROMs introduced in Definition \ref{def:ROMs}, we define the $\mathbf{S}$-\textit{Uniform} random matrix with $k \in \mathbb{N}$ blocks to be given by
\begin{equation*}
\mathbf{M}_\mathrm{\mathbf{S}\mathcal{U}}^{(k)} = \prod_{i=1}^k \mathbf{S}\mathbf{D}_i^{(\mathcal{U})} \, ,
\end{equation*}
where $(\mathbf{D}^{(\mathcal{U})}_i)_{i=1}^k$ are iid diagonal matrices with iid $\mathrm{Unif}(S^1)$ random variables on the diagonals, and $S^1$ is the unit circle of $\mathbb{C}$.

As alluded to in \S \ref{sec:ojlt}, we will see that introducing this increased number of complex parameters does not lead to significant increases in statistical performance relative to the estimator $\widehat{K}_m^{\mathcal{H}, (k)}$ for dimensionality reduction.

We consider the estimator $\widehat{K}^{\mathcal{U}, (k)}_m$ below, based on the sub-sampled $\mathbf{SD}$-product matrix $\mathbf{M}_{\mathrm{\mathbf{S}\mathcal{U}}}^{(k),\mathrm{sub}}$:
\begin{align*}
\widehat{K}^{\mathcal{U}, (k)}_m(\mathbf{x}, \mathbf{y}) & = \frac{1}{m}\mathrm{Re}\left\lbrack\left(\overline{\mathbf{M}_{\mathrm{\mathbf{S}\mathcal{U}}}^{(k),\mathrm{sub}}\mathbf{x}}\right)^\top \left(\mathbf{M}_{\mathrm{\mathbf{S}\mathcal{U}}}^{(k),\mathrm{sub}}\mathbf{y}\right) \right\rbrack \, ,
\end{align*}
and show that it does not yield a significant improvement over the estimator $\widehat{K}_m^{\mathcal{H}, (k)}$ of Theorem \ref{thm:hybrid-mse}:

\begin{theorem}\label{thm:hd-full-complex}
	For $\mathbf{x}, \mathbf{y} \in \mathbb{R}^n$, the estimator $\widehat{K}^{\mathcal{U}, (k)}_m(\mathbf{x},\mathbf{y})$, applying random sub-sampling strategy without replacement  
	is unbiased and satisfies:
	{\small 
	\begin{align*}
	&\mathrm{MSE}(\widehat{K}^{\mathcal{U}, (k)}_{m}(\mathbf{x},\mathbf{y}))\! = \\ & \frac{1}{2m}\left(\frac{n-m}{n-1}\right)\left( \left( (\mathbf{x}^\top\mathbf{y})^2 \!+\! 
	\|\mathbf{x}\|^2\|\mathbf{y}\|^2 \right) \!+ 
	\sum_{r=1}^{k-1} \frac{(-1)^{r}}{n^{r}}(3(\mathbf{x}^\top\mathbf{y})^2 + \|\mathbf{x}\|^2\|\mathbf{y}\|^2) + \frac{(-1)^k 2}{n^{k-1}} \sum_{i=1}^n x_i^2 y_i^2\right). 
	\end{align*}}
\end{theorem}

The structure of the proof of Theorem \ref{thm:hd-full-complex} is broadly the same as that of Theorem \ref{hd_theorem}. We begin by remarking that the proof that the estimator is unbiased is exactly the same as that of Proposition \ref{prop:hd-complex-unbiased}. We then note that in the case of $k=1$ block, the estimators $\widehat{K}_m^{\mathcal{H}, (1)}$ and $\widehat{K}_m^{\mathcal{U}, (1)}$, coincide so Proposition \ref{prop:hd-complex-1block} establishes the MSE of the estimator $\widehat{K}_m^{\mathcal{U}, (k)}$ in the base case $k=1$. We then obtain a recursion formula for the MSE (Proposition \ref{prop:hd-complex-recurse}), and finally prove the theorem by induction.

\begin{proposition}\label{prop:hd-complex-recurse}
	Let $k \geq 2$, $n \in \mathbb{N}$, $m\leq n$, and $\mathbf{x}, \mathbf{y} \in \mathbb{C}^n$ such that $\langle \overline{\mathbf{x}}, \mathbf{y} \rangle \in \mathbb{R}$; in particular, this includes $\mathbf{x}, \mathbf{y} \in \mathbb{R}^n$. Then we have the following recursion for the MSE of $\widehat{K}^{\mathcal{U}, (k)}_{M}(\mathbf{x},\mathbf{y})$:
	\[
	\mathrm{MSE}(\widehat{K}^{\mathcal{U}, (k)}_{m}(\mathbf{x},\mathbf{y})) = \expectation{\mathrm{MSE}(\widehat{K}^{\mathcal{U}, (k-1)}_{m}(\mathbf{SD}_1\mathbf{x},\mathbf{SD}_1\mathbf{y}) \big| \mathbf{D}_1)}
	\]
\end{proposition}

\begin{proof}
	The proof is exactly analogous to that of Proposition \ref{prop:hd-ojlt-mse-recurse}, and is therefore omitted.
\end{proof}

Before we complete the proof by induction, we will need the following auxiliary result, to deal with the expectations that arise during the recursion due to the terms in the MSE expression of Proposition \ref{prop:hd-complex-1block}.

\begin{lemma}\label{lemma:hd-complex-help}
		Under the assumptions of Theorem \ref{thm:hd-full-complex}, we have the following expectations:
	\begin{align*}
	\expectation{|(\mathbf{SDx})_r|^2 |(\mathbf{SDy})_r|^2} = \frac{1}{n^2} \left( \langle \overline{\mathbf{x}}, \mathbf{x} \rangle \langle \overline{\mathbf{y}}, \mathbf{y} \rangle + \langle \overline{\mathbf{x}}, \mathbf{y} \rangle^2 - \sum_{i=1}^n |x_i|^2 |y_i|^2 \right)
	\end{align*}
	\begin{align*}
	\expectation{\mathrm{Re}((\mathbf{S}\overline{\mathbf{D}}\overline{\mathbf{x}})_r^2 (\mathbf{SDy})_r^2)} = \frac{1}{n^2} \left(2 \langle \overline{\mathbf{x}}, \mathbf{y} \rangle^2 - \sum_{i=1}^n\mathrm{Re}(\overline{x}_i^2 y_i^2)  \right)
	\end{align*}
\end{lemma}
\begin{proof}
	For the first claim, we note that
	\begin{align*}
		\expectation{|(\mathbf{SDx})_r|^2 |(\mathbf{SDy})_r|^2} = & \sum_{i,j,k,l}^n s_{ri} s_{rj} s_{rk} s_{rl} \overline{x}_i x_j \overline{y}_k y_l \mathbb{E}\left\lbrack \overline{d}_i d_j \overline{d}_k d_l \right\rbrack \\
		= & \frac{1}{n^2} \left( \sum_{i \not= j} \overline{x}_i x_i \overline{y}_j y_j + \sum_{i \not= j} \overline{x}_i x_j \overline{y}_j y_i + \sum_{i=1}^n \overline{x}_i x_i \overline{y}_i y_i \right) \\
		= & \frac{1}{n^2} \left( \sum_{i, j = 1}^n \overline{x}_i x_i \overline{y}_j y_j + \sum_{i, j=1}^n \overline{x}_i x_j \overline{y}_j y_i - \sum_{i=1}^n \overline{x}_i x_i \overline{y}_i y_i \right) \\
		= & \frac{1}{n^2} \left( \langle \overline{\mathbf{x}}, \mathbf{x} \rangle \langle \overline{\mathbf{y}} , \mathbf{y} \rangle + \langle \overline{\mathbf{x}}, \mathbf{y} \rangle ^2 - \sum_{i=1}^n |x_i|^2 |y_i|^2 \right) \, ,
	\end{align*}
	as required, where in the final equality we have use the assumption that $\langle \overline{\mathbf{x}}, \mathbf{y} \rangle \in \mathbb{R}$. For the second claim, we observe that
	\begin{align*}
	\expectation{\mathrm{Re}((\overline{\mathbf{SDx}})_r^2 (\mathbf{SDy})_r^2} = & \mathrm{Re}\left(\sum_{i,j,k,l}^n s_{ri} s_{rj} s_{rk} s_{rl} \overline{x}_i \overline{x}_j y_k y_l \mathbb{E}\left\lbrack \overline{d}_i \overline{d}_j d_k d_l \right\rbrack \right)\\
	= & \mathrm{Re}\left(\frac{1}{n^2} \left( 2\sum_{i \not= j} \overline{x}_i \overline{x_j} y_i y_j + \sum_{i=1}^n \overline{x}_i \overline{x}_i y_i y_i \right)\right) \\
	= & \frac{1}{n^2} \left(2 \langle \overline{\mathbf{x}}, \mathbf{y} \rangle^2 - \sum_{i=1}^n\mathrm{Re}\left( \overline{x}_i^2 y_i^2\right) \right) \, ,
	\end{align*}
	where again we have used the assumption that $\langle \overline{\mathbf{x}}, \mathbf{y} \rangle \in \mathbb{R}$.
\end{proof}

\begin{proof}[Proof of Theorem \ref{thm:hd-full-complex}]
	The proof now proceeds by induction. We in fact prove the stronger result that for any $\mathbf{x}, \mathbf{y} \in \mathbb{C}^n$ for which $\langle \overline{\mathbf{x}}, \mathbf{y} \rangle \in \mathbb{R}$, we have
	{\small 
	\begin{align*}
	\mathrm{MSE}(\widehat{K}^{\mathcal{U}, (k)}_{m}(\mathbf{x},\mathbf{y}))\! =\! \frac{1}{2m}\left(\frac{n-m}{n-1}\right)\Bigg(\!\! &\left( \langle \overline{\mathbf{x}}, \mathbf{y} \rangle^2 \!+\! 
	\langle\overline{\mathbf{x}}, \mathbf{x} \rangle\langle \overline{\mathbf{y}}, \mathbf{y} \rangle \right) \!+ 
	\sum_{r=1}^{k-1} \frac{(-1)^{r}}{n^{r}}(3\langle\overline{\mathbf{x}}, \mathbf{y} \rangle^2 \!+\! 
	\langle\overline{\mathbf{x}}, \mathbf{x} \rangle\langle \overline{\mathbf{y}}, \mathbf{y} \rangle) + \\
	& \frac{(-1)^k}{n^{k-1}} \left( \sum_{i=1}^n \left(|x_i|^2 |y_i|^2 + \mathrm{Re}\left(\overline{x}_i^2 y_i^2 \right) \right)\right) \Bigg) \, .
	\end{align*}}
from which Theorem \ref{thm:hd-full-complex} clearly follows.
	Proposition \ref{prop:hd-complex-1block} yields the base case $k=1$ for this claim. For the recursive step, suppose that the result holds for some number $k \in \mathbb{N}$ of blocks. Recalling the recursion of Proposition \ref{prop:hd-complex-recurse}, we then obtain
	{\small
	\begin{align*}
	\mathrm{MSE}(\widehat{K}^{\mathcal{U}, (k+1)}_{m}(\mathbf{x},\mathbf{y})) \!= &\frac{1}{2m}\left(\frac{n-m}{n-1}\right) \!\Bigg(\! \left( \langle \overline{\mathbf{x}}, \mathbf{y} \rangle^2 \!+\! 
	\langle\overline{\mathbf{x}}, \mathbf{x} \rangle\langle \overline{\mathbf{y}}, \mathbf{y} \rangle \right) \!+ 
	\sum_{r=1}^{k-1} \frac{(-1)^{r}}{n^{r}}(3\langle\overline{\mathbf{x}}, \mathbf{y} \rangle^2 \!+\! 
	\langle\overline{\mathbf{x}}, \mathbf{x} \rangle\langle \overline{\mathbf{y}}, \mathbf{y} \rangle) + \\
	& \frac{(-1)^k}{n^{k-1}} \left( \sum_{i=1}^n \left(\mathbb{E}\left\lbrack|\mathbf{SD}_1\mathbf{x}|_i^2 |\mathbf{SD}_1 \mathbf{y}|_i^2 \right\rbrack + \mathbb{E}\left\lbrack \mathrm{Re}\left((\overline{\mathbf{SD}_1\mathbf{x}})_i^2 (\mathbf{SD}_1\mathbf{y})_i^2 \right) \right\rbrack \right)\right) \Bigg),
	\end{align*}}
	where we have used the fact that $\mathbf{SD}_1$ is a unitary isometry almost surely, and thus preserves Hermitian products. Applying Lemma \ref{lemma:hd-complex-help} to the remaining expectations and collecting terms proves the inductive step, which concludes the proof of the theorem.
\end{proof}

\subsection{Proof of Theorem \ref{thm:no-replacement-subsampling}}

\begin{proof}
	The proof of this result is reasonably straightforward with the proofs of Theorems \ref{hd_theorem} and \ref{thm:hybrid-mse} in hand; we simply recognize where in these proofs the assumption of the sampling strategy without replacement 
	was used. We deal first with Theorem \ref{hd_theorem}, which deals with the MSE associated with $\widehat{K}^{(k)}_m(\mathbf{x}, \mathbf{y})$. The only place in which the assumption of the sub-sampling strategy without replacement 
	is used is mid-way through the proof of Proposition \ref{prop:hd-ojlt-mse-1block}, which quantifies $\mathrm{MSE}(\widehat{K}^{(1)}_m(\mathbf{x}, \mathbf{y}))$. Picking up the proof at the point the sub-sampling strategy is used, we have
	\begin{align*}
	\mathrm{MSE}(\widehat{K}^{(1)}_m(\mathbf{x}, \mathbf{y})) = \frac{n^2}{m^2} \sum_{p, p^\prime = 1}^m \sum_{i \not= j}^n \left( x^2_i y^2_j + x_i x_j y_i y_j \right) \expectation{s_{J_p i} s_{J_p j} s_{J_{p^\prime} i}s_{J_{p^\prime} j}} \, .
	\end{align*}
	Now instead using sub-sampling strategy with replacement, 
	note that each pair of sub-sampled indices $J_p$ and $J_{p^\prime}$ are independent. Recalling that the columns of $\mathbf{S}$ are orthogonal, we obtain for distinct $p$ and $p^\prime$ that
	\[
	 \expectation{s_{J_p i} s_{J_p j} s_{J_{p^\prime} i}s_{J_{p^\prime} j}} =  \expectation{s_{J_p i} s_{J_p j}} \expectation{s_{J_{p^\prime} i}s_{J_{p^\prime} j}} = 0 \, .
	\]
	Again, for $p=p^\prime$, we have $ \expectation{s_{J_p i} s_{J_p j} s_{J_{p^\prime} i}s_{J_{p^\prime} j}} = 1/n^2$.
	Substituting the values of these expectations back into the expression for the MSE of $\widehat{K}_m^{(k)}(\mathbf{x}, \mathbf{y})$ then yields
	\begin{align*}
	\mathrm{MSE}(\widehat{K}^{(1)}_m(\mathbf{x}, \mathbf{y})) = & \frac{n^2}{m^2}  \sum_{i \not= j}^n \left( x^2_i y^2_j + x_i x_j y_i y_j \right) \left(m\times \frac{1}{n^2} \right) \\
	= & \frac{1}{m}\left(1 -\frac{m-1}{n-1}\right) \sum_{i \not= j}^n \left( x^2_i y^2_j + x_i x_j y_i y_j \right) \\
	= & \frac{1}{m} \left( \left\langle \mathbf{x}, \mathbf{y} \right\rangle^2 + \|\mathbf{x}\|^2 \|\mathbf{y}\|^2 - 2\sum_{i=1}^n x_i^2y_i^2 \right)
	\end{align*}
	as required.
	
	For the estimator $\widehat{K}^{\mathcal{H}, (k)}_m(\mathbf{x}, \mathbf{y})$, the result also immediately follows with the above calculation, as the only point in the proof of the MSE expressions for these estimators that is influenced by the sub-sampling strategy is in the calculation of the quantities $\expectation{s_{J_p i} s_{J_p j} s_{J_{p^\prime} i}s_{J_{p^\prime} j}}$; therefore, exactly the same multiplicative factor is incurred for MSE as for  $\widehat{K}^{(k)}_m(\mathbf{x}, \mathbf{y})$.
	
\end{proof}

\section{Proofs of results in \texorpdfstring{\S \ref{sec:kernels}}{Section \ref{sec:kernels}}}

\subsection{Proof of Lemma \ref{simple_lemma}}

\begin{proof}
Follows immediately from the proof of Theorem \ref{gen_theorem} (see: the proof below).
\end{proof}

\subsection{Proof of Theorem \ref{angle_theorem}}\label{sec:ang-ort-proof}

Recall that the angular kernel estimator based on $\mathbf{G}_\mathrm{ort}$ is given by
\[
\widehat{K}^{\mathrm{ang, ort}}_m(\mathbf{x},\mathbf{y}) = \frac{1}{m} \mathrm{sign}(\mathbf{G}_\mathrm{ort} \mathbf{x})^\top \mathrm{sign}(\mathbf{G}_\mathrm{ort} \mathbf{y})
\]
where the function $\mathrm{sign}$ acts on vectors element-wise. In what follows, we write $\mathbf{G}_\mathrm{ort}^i$ for the $i$th row of $\mathbf{G}_\mathrm{ort}$, and $\mathbf{G}_i$ for the $i$th row of $\mathbf{G}$.

Since each $\mathbf{G}_{\mathrm{ort}}^i$ has the same marginal distribution as $\mathrm{R}_m$ in the unstructured Gaussian case covered by Theorem \ref{gen_theorem}, unbiasedness of $\widehat{K}^{\mathrm{ang, ort}}(x,y)$ follows immediately from this result, and so we obtain:

\begin{lemma}
	$\widehat{K}^{\mathrm{ang, ort}}_m(\mathbf{x}, \mathbf{y})$ is an unbiased estimator of $K^{\mathrm{ang}}(\mathbf{x},\mathbf{y})$.
\end{lemma}

We now turn our attention to the variance of $\widehat{K}^{\mathrm{ang, ort}}_m(\mathbf{x}, \mathbf{y})$.

\begin{theorem}\label{thm:angular-orth-var}
	The variance of the estimator $\widehat{K}^{\mathrm{ang, ort}}_m(x, y)$ is strictly smaller than the variance of $ \widehat{K}^\mathrm{ang,\ base}_m(\mathbf{x}, \mathbf{y})$
\end{theorem}
\begin{proof}
	Denote by $\theta$ the angle between $\mathbf{x}$ and $\mathbf{y}$, and for notational ease, let $S_i = \takesign{\innerprod{\mathbf{G}^i}{\mathbf{x}}}\takesign{\innerprod{\mathbf{G}^i}{\mathbf{y}}}$, and $S^\mathrm{ort}_i = \takesign{\innerprod{\mathbf{G}_{\mathrm{ort}}^i}{\mathbf{x}}}\takesign{\innerprod{\mathbf{G}_{\mathrm{ort}}^i}{\mathbf{y}}}$. Now observe that as $\widehat{K}^{\mathrm{ang, ort}}_m(\mathbf{x}, \mathbf{y})$ is unbiased, we have
	\begin{align*}
	& \phantom{=} \Var{\widehat{K}^{\mathrm{ang, ort}}_m(\mathbf{x}, \mathbf{y})}\\
	& = \Var{\frac{1}{m} \sum_{i=1}^m  S^\mathrm{ort}_i} \\
	& = \frac{1}{m^2} \left( \sum_{i=1}^m \Var{S_i^\mathrm{ort}} + \sum_{i \not= i^\prime}^m \Cov{S_i^\mathrm{ort}}{S_{i^\prime}^\mathrm{ort}} \right) \, .
	\end{align*}
	By a similar argument, we have
	\begin{align}\label{eq:unstruct-approx-var}
	\Var{\widehat{K}^\mathrm{base}_m(\mathbf{x}, \mathbf{y})} = \frac{1}{m^2} \left( \sum_{i=1}^m \Var{S_i} + \sum_{i \not= i^\prime}^m \Cov{S_i}{S_{i^\prime}} \right) \, .
	\end{align}
	Note that the covariance terms in \eqref{eq:unstruct-approx-var} evaluate to $0$, by independence of $S_i$ and $S_{i^\prime}$ for $i \not= i^\prime$ (which is inherited from the independence of $\mathbf{G}^i$ and $\mathbf{G}^{i^\prime}$). Also observe that since $\mathbf{G}^i \overset{d}{=} \mathbf{G}_{\mathrm{ort}}^i$, we have 
	\[
	\Var{S_i^\mathrm{ort}} = \Var{S_i} \, .
	\]
	Therefore, demonstrating the theorem is equivalent to showing, for $i\not= i^\prime$, that
	\[
	\Cov{S_i^\mathrm{ort}}{S_{i^\prime}^\mathrm{ort}} < 0 \, ,
	\]
	 which is itself equivalent to showing
    \begin{align}\label{eq:angthm:main-objective} 
	\expectation{S_i^\mathrm{ort} S_{i^\prime}^\mathrm{ort}} < \expectation{S_i^\mathrm{ort}}\expectation{S_{i^\prime}^\mathrm{ort}} \, .
	\end{align}
	Note that the variables $(S^\mathrm{ort}_i)_{i=1}^m$ take values in $\{ \pm 1\}$. Denoting $\mathcal{A}_i = \{S^\mathrm{ort}_i = -1 \}$ for $i=1,\ldots,m$, we can rewrite \eqref{eq:angthm:main-objective} as
	\[
	\probability{\mathcal{A}_i^c \cap \mathcal{A}_{i^\prime}^c} + \probability{\mathcal{A}_i \cap \mathcal{A}_{i^\prime}} - \probability{\mathcal{A}_i \cap \mathcal{A}_{i^\prime}^c} - \probability{\mathcal{A}_i^c \cap \mathcal{A}_{i^\prime}} < \left(\frac{\pi - 2\theta}{\pi}\right)^2 \, .
	\]
	Note that the left-hand side is equal to 
	\[
	2(\probability{\mathcal{A}_i^c \cap \mathcal{A}_{i^\prime}^c} + \probability{\mathcal{A}_i \cap \mathcal{A}_{i^\prime}})  - 1 \, .
	\]
	Plugging in the bounds of Proposition \ref{prop:angkernel:prob-bounds}, and using the fact that the pair of indicators $(\mathbbm{1}_{\mathcal{A}_i}, \mathbbm{1}_{\mathcal{A}_{i^\prime}})$ is identically distributed for all pairs of distinct indices $i, i^\prime \in \{1, \ldots, m\}$,  thus yields the result.
\end{proof}

\begin{proposition}\label{prop:angkernel:prob-bounds}
	We then have the following inequalities:
	\begin{align*}
	\probability{\mathcal{A}_1 \cap \mathcal{A}_2} < \left( \frac{\theta}{\pi} \right)^2 \qquad \mathrm{and} \qquad \probability{\mathcal{A}_1^c \cap \mathcal{A}_2^c} < \left(1 - \frac{ \theta}{\pi} \right)^2
	\end{align*}
\end{proposition}

Before providing the proof of this proposition, we describe some coordinate choices we will make in order to obtain the bounds in Proposition \ref{prop:angkernel:prob-bounds}.

We pick an orthonormal basis for $\mathbb{R}^n$ so that the first two coordinates span the $\mathbf{x}$-$\mathbf{y}$ plane, and further so that $(\mathbf{G}_\mathrm{ort}^{1})_2$, the coordinate of $\mathbf{G}_{\mathrm{ort}}^1$ in the second dimension, is $0$. We extend this to an orthonormal basis of $\mathbb{R}^n$ so that $(\mathbf{G}_\mathrm{ort}^1)_3 \geq 0$, and $(\mathbf{G}_\mathrm{ort}^1)_{i} = 0$ for $i \geq 4$. Thus, in this basis, we have coordinates
\begin{align*}
\mathbf{G}_{\mathrm{ort}}^1 = ((\mathbf{G}_\mathrm{ort}^1)_1, 0, (\mathbf{G}_\mathrm{ort}^1)_3, 0, \ldots, 0) \, ,
\end{align*}
with $(\mathbf{G}_\mathrm{ort}^1)_1 \sim \chi_2$ and $(\mathbf{G}_\mathrm{ort}^1)_3 \sim \chi_{N-2}$ (by elementary calculations with multivariate Gaussian distributions). Note that the angle, $\phi$,  that $\mathbf{G}_{\mathrm{ort}}^1$ makes with the $\mathbf{x}$-$\mathbf{y}$ plane is then $\phi = \arctan((\mathbf{G}_\mathrm{ort}^1)_3/(\mathbf{G}_\mathrm{ort}^1)_1)$. Having fixed our coordinate system relative to the random variable $\mathbf{G}_{\mathrm{ort}}^1$, the coordinates of $\mathbf{x}$ and $\mathbf{y}$ in this frame are now themselves random variables; we introduce the angle $\psi$ to describe the angle between $\mathbf{x}$ and the positive first coordinate axis in this basis.

Now consider $\mathbf{G}_{\mathrm{ort}}^2$. We are concerned with the direction of $((\mathbf{G}_\mathrm{ort}^2)_1, (\mathbf{G}_\mathrm{ort}^2)_2)$ in the $\mathbf{x}$-$\mathbf{y}$ plane. Conditional on $\mathbf{G}_{\mathrm{ort}}^1$, the direction of the full vector $\mathbf{G}_{\mathrm{ort}}^2$ is distributed uniformly on $S^{n-2}(\langle \mathbf{G}_{\mathrm{ort}}^1 \rangle^\perp)$, the set of unit vectors orthogonal to $\mathbf{G}_{\mathrm{ort}}^1$. Because of our particular choice of coordinates, we can therefore write
\[
\mathbf{G}_{\mathrm{ort}}^2 = (r \sin(\phi), (\mathbf{G}_\mathrm{ort}^2)_2, r \cos(\phi), (\mathbf{G}_\mathrm{ort}^2)_4, (\mathbf{G}_\mathrm{ort}^2)_5,\ldots, (\mathbf{G}_\mathrm{ort}^2)_n) \, ,
\]
where the $(N-1)$-dimensional vector $(r, (\mathbf{G}_\mathrm{ort}^2)_2, (\mathbf{G}_\mathrm{ort}^2)_4, (\mathbf{G}_\mathrm{ort}^2)_5, \ldots, (\mathbf{G}_\mathrm{ort}^2)_n)$ has an isotropic distribution.

So the direction of $((\mathbf{G}_\mathrm{ort}^2)_1, (\mathbf{G}_\mathrm{ort}^2)_2)$ in the $\mathbf{x}$-$\mathbf{y}$ plane follows an angular Gaussian distribution, with covariance matrix
\[
\begin{pmatrix}
\sin^2(\phi) & 0 \\ 
0 & 1
\end{pmatrix} \, .
\]

With these geometrical considerations in place, we are ready to give the proof of Proposition \ref{prop:angkernel:prob-bounds}.

\begin{proof}[Proof of Proposition \ref{prop:angkernel:prob-bounds}]
	Dealing with the first inequality, we decompose the event as
	\begin{align*}
	\mathcal{A}_1 \cap \mathcal{A}_2 = & \{\innerprod{\mathbf{G}_{\mathrm{ort}}^1}{\mathbf{x}} > 0 , \, \innerprod{\mathbf{G}_{\mathrm{ort}}^1}{\mathbf{y}} < 0 , \, \innerprod{\mathbf{G}_{\mathrm{ort}}^2}{\mathbf{x}} > 0 , \, \innerprod{\mathbf{G}_{\mathrm{ort}}^2}{\mathbf{y}} < 0   \} \\
	& \cup \{\innerprod{\mathbf{G}_{\mathrm{ort}}^1}{\mathbf{x}} > 0 , \, \innerprod{\mathbf{G}_{\mathrm{ort}}^1}{\mathbf{y}} < 0 , \, \innerprod{\mathbf{G}_{\mathrm{ort}}^2}{\mathbf{x}} < 0 , \, \innerprod{\mathbf{G}_{\mathrm{ort}}^2}{\mathbf{y}} > 0   \} \\
	& \cup \{\innerprod{\mathbf{G}_{\mathrm{ort}}^1}{\mathbf{x}} < 0 , \, \innerprod{\mathbf{G}_{\mathrm{ort}}^1}{\mathbf{y}} > 0 , \, \innerprod{\mathbf{G}_{\mathrm{ort}}^2}{\mathbf{x}} > 0 , \, \innerprod{\mathbf{G}_{\mathrm{ort}}^2}{\mathbf{y}} < 0   \} \\
	& \cup \{\innerprod{\mathbf{G}_{\mathrm{ort}}^1}{\mathbf{x}} < 0 , \, \innerprod{\mathbf{G}_{\mathrm{ort}}^1}{\mathbf{y}} > 0 , \, \innerprod{\mathbf{G}_{\mathrm{ort}}^2}{\mathbf{x}} < 0 , \, \innerprod{\mathbf{G}_{\mathrm{ort}}^2}{\mathbf{y}} > 0   \} \, .
	\end{align*}
	As the law of $(\mathbf{G}_{\mathrm{ort}}^1, \mathbf{G}_{\mathrm{ort}}^2)$ is the same as that of $(\mathbf{G}_{\mathrm{ort}}^2, \mathbf{G}_{\mathrm{ort}}^1)$ and that of $(-\mathbf{G}_{\mathrm{ort}}^1, \mathbf{G}_{\mathrm{ort}}^2)$, it follows that all four events in the above expression have the same probability. The statement of the theorem is therefore equivalent to demonstrating the following inequality:
	\begin{align*}
	\probability{\innerprod{\mathbf{G}_{\mathrm{ort}}^1}{x} > 0 , \, \innerprod{\mathbf{G}_{\mathrm{ort}}^1}{\mathbf{y}} < 0 , \, \innerprod{\mathbf{G}_{\mathrm{ort}}^2}{\mathbf{x}} > 0 , \, \innerprod{\mathbf{G}_{\mathrm{ort}}^2}{\mathbf{y}} < 0} < \left(\frac{\theta}{2\pi}\right)^2 \, .
	\end{align*}
	We now proceed according to the coordinate choices described above. We first condition on the random angles $\phi$ and $\psi$ to obtain
	\begin{align*}
	&\probability{\innerprod{\mathbf{G}_{\mathrm{ort}}^1}{\mathbf{x}} > 0 , \, \innerprod{\mathbf{G}_{\mathrm{ort}}^1}{\mathbf{y}} < 0 , \, \innerprod{\mathbf{G}_{\mathrm{ort}}^2}{\mathbf{x}} > 0 , \, \innerprod{\mathbf{G}_{\mathrm{ort}}^2}{\mathbf{y}} < 0}\\
	 = & \int_{0}^{2\pi} \frac{\calcd\psi}{2\pi} \int_0^{\pi/2} f(\phi)\calcd \phi \	\probability{\innerprod{\mathbf{G}_{\mathrm{ort}}^1}{\mathbf{x}} > 0 , \, \innerprod{\mathbf{G}_{\mathrm{ort}}^1}{\mathbf{y}} < 0 , \, \innerprod{\mathbf{G}_{\mathrm{ort}}^2}{\mathbf{x}} > 0 , \, \innerprod{\mathbf{G}_{\mathrm{ort}}^2}{\mathbf{y}} < 0 | \psi, \phi}\\
	 = & \int_{0}^{2\pi} \frac{\calcd\psi}{2\pi} \int_0^{\pi/2} f(\phi) \calcd \phi \ \mathbbm{1}_{ \{ 0 \in [\psi - \pi/2, \psi - \pi/2 + \theta] \} }	\probability{\innerprod{\mathbf{G}_{\mathrm{ort}}^2}{\mathbf{x}} > 0 , \, \innerprod{\mathbf{G}_{\mathrm{ort}}^2}{\mathbf{y}} < 0 | \psi, \phi} \, ,
	\end{align*}
	where $f$ is the density of the random angle $\phi$. The final equality above follows as $\mathbf{G}_{\mathrm{ort}}^1$ and $\mathbf{G}_{\mathrm{ort}}^2$ are independent conditional on $\psi$ and $\phi$, and since the event $\{ \innerprod{\mathbf{G}_{\mathrm{ort}}^1}{\mathbf{x}} > 0 , \, \innerprod{\mathbf{G}_{\mathrm{ort}}^1}{\mathbf{y}} < 0 \}$ is exactly the event $\{ 0 \in [\psi - \pi/2, \psi - \pi/2 + \theta] \}$, by considering the geometry of the situation in the $\mathbf{x}$-$\mathbf{y}$ plane. We can remove the indicator function from the integrand by adjusting the limits of integration, obtaining
	\begin{align*}
	&\probability{\innerprod{\mathbf{G}_{\mathrm{ort}}^1}{\mathbf{x}} > 0 , \, \innerprod{\mathbf{G}_{\mathrm{ort}}^1}{\mathbf{y}} < 0 , \, \innerprod{\mathbf{G}_{\mathrm{ort}}^2}{\mathbf{x}} > 0 , \, \innerprod{\mathbf{G}_{\mathrm{ort}}^2}{\mathbf{y}} < 0}\\
	 = & \int_{\pi/2 - \theta}^{\pi/2} \frac{\calcd\psi}{2\pi} \int_0^{\pi/2} f(\phi) \calcd \phi \ 	\probability{\innerprod{\mathbf{G}_{\mathrm{ort}}^2}{\mathbf{x}} > 0 , \, \innerprod{\mathbf{G}_{\mathrm{ort}}^2}{\mathbf{y}} < 0 | \psi, \phi} \, .
	\end{align*}
	We now turn our attention to the conditional probability
	\[
	\probability{\innerprod{\mathbf{G}_{\mathrm{ort}}^2}{\mathbf{x}} > 0 , \, \innerprod{\mathbf{G}_{\mathrm{ort}}^2}{\mathbf{y}} < 0 | \psi, \phi} \, .
	\]
	The event $\{ \innerprod{\mathbf{G}_{\mathrm{ort}}^2}{\mathbf{x}} > 0 , \, \innerprod{\mathbf{G}_{\mathrm{ort}}^2}{\mathbf{y}} < 0 \}$ is equivalent to the angle $t$ of the projection of $\mathbf{G}_{\mathrm{ort}}^2$ into the $\mathbf{x}$-$\mathbf{y}$ plane with the first coordinate axis lying in the interval $[\psi - \pi/2, \psi - \pi/2 + \theta]$. Recalling the distribution of the angle $t$ from the geometric considerations described immediately before this proof, we obtain
	\begin{align*}
	&\probability{\innerprod{\mathbf{G}_{\mathrm{ort}}^1}{\mathbf{x}} > 0 , \, \innerprod{\mathbf{G}_{\mathrm{ort}}^1}{\mathbf{y}} < 0 , \, \innerprod{\mathbf{G}_{\mathrm{ort}}^2}{\mathbf{x}} > 0 , \, \innerprod{\mathbf{G}_{\mathrm{ort}}^2}{\mathbf{y}} < 0}\\
	= & \int_{\pi/2 - \theta}^{\pi/2} \frac{\calcd\psi}{2\pi} \int_0^{\pi/2} f(\phi) \calcd \phi \ 	\int_{\psi - \pi/2}^{\psi - \pi/2 + \theta} (2 \pi \sin(\phi))^{-1} (\cos^2(t)/\sin^2(\phi) + \sin^2(t))^{-1} dt \, .
	\end{align*}
	With $\theta \in [0, \pi/2]$, we note that the integral with respect to $t$ can be evaluated analytically, leading us to
	{\small 
	\begin{align*}
	&\probability{\innerprod{\mathbf{G}_{\mathrm{ort}}^1}{\mathbf{x}} > 0 , \, \innerprod{\mathbf{G}_{\mathrm{ort}}^1}{\mathbf{y}} < 0 , \, \innerprod{\mathbf{G}_{\mathrm{ort}}^2}{\mathbf{x}} > 0 , \, \innerprod{\mathbf{G}_{\mathrm{ort}}^2}{\mathbf{y}} < 0}\\
	= & \int_{\pi/2 - \theta}^{\pi/2} \frac{\calcd\psi}{2\pi} \int_0^{\pi/2} f(\phi) \calcd \phi \ \frac{1}{2\pi} \left( \arctan(\tan(\psi - \pi/2 + \theta)\sin(\phi)) - \arctan(\tan(\psi - \pi/2)\sin(\phi)) \right) \\
	\leq & \int_{\pi/2 - \theta}^{\pi/2} \frac{\calcd\psi}{2\pi} \int_0^{\pi/2} f(\phi) \calcd \phi \ \frac{\theta}{2\pi}\\
	& = \left(\frac{\theta}{2\pi}\right)^2 . 
	\end{align*}}
	To deal with $\theta \in [\pi/2, \pi]$, we note that if the angle $\theta$ between $\mathbf{x}$ and $\mathbf{y}$ is obtuse, then the angle between $\mathbf{x}$ and $-\mathbf{y}$ is $\pi - \theta$ and therefore acute. Recalling from our definition that $\mathcal{A}_m = \{ \takesign{\innerprod{\mathbf{G}_{\mathrm{ort}}^i}{\mathbf{x}}}\takesign{\innerprod{\mathbf{G}_{\mathrm{ort}}^i}{\mathbf{y}}} = -1 \}$, if we denote the corresponding quantity for the pair of vecors $\mathbf{x}$, $-\mathbf{y}$ by $\bar{\mathcal{A}}_m = \{ \takesign{\innerprod{\mathbf{G}_{\mathrm{ort}}^i}{\mathbf{x}}}\takesign{\innerprod{\mathbf{G}_{\mathrm{ort}}^i}{-\mathbf{y}}} = -1 \}$, then we in fact have $\bar{\mathcal{A}}_m = \mathcal{A}_m^c$. Therefore, applying the result to the pair of vectors $\mathbf{x}$ and $-\mathbf{y}$ (which have acute angle $\pi -\theta$ between them) and using the inclusion-exclusion principle, we obtain:
	\begin{align*}
	\mathbb{P}(\mathcal{A}_1 \cap \mathcal{A}_2) & = 1 - \mathbb{P}(\mathcal{A}_1^c) - \mathbb{P}(\mathcal{A}_2^c) + \mathbb{P}(\mathcal{A}_1^c \cap \mathcal{A}_2^c) \\
	& < 1 - \mathbb{P}(\mathcal{A}_1^c) - \mathbb{P}(\mathcal{A}_2^c) + \left(\frac{\pi - \theta}{\pi}\right)^2 \\
	& = 1 - 2\left(\frac{\pi - \theta}{\pi}\right) + \left(\frac{\pi - \theta}{\pi}\right)^2 \\
	& = \left(\frac{\theta}{\pi}\right)^2
	\end{align*}
	as required.
	
	The second inequality of Proposition \ref{prop:angkernel:prob-bounds} follows from the inclusion-exclusion principle and the first inequality:
	\begin{align*}
	\probability{\mathcal{A}_1^c \cap \mathcal{A}_2^c} & = 1 - \probability{\mathcal{A}_1} - \probability{\mathcal{A}_2} + \probability{\mathcal{A}_1 \cap \mathcal{A}_2} \\
	& < 1 - \probability{\mathcal{A}_1} - \probability{\mathcal{A}_2} + \left(\frac{\theta}{\pi}\right)^2\\
	& = (1 - \probability{\mathcal{A}_1})(1 - \probability{\mathcal{A}_2}) \\
	& = \left(1 - \frac{\theta}{\pi}\right)^2 \, .
	\end{align*}		
\end{proof}

\subsection{Proof of Theorem \ref{gen_theorem}}

\begin{proof}
We will consider the following setting. Given two vectors $\mathbf{x}, \mathbf{y} \in \mathbb{R}^{n}$, each of them is transformed by the nonlinear mapping: 
$\phi^{\mathbf{M}}: \mathbf{z} \rightarrow \frac{1}{\sqrt{k}} \mathrm{sgn}(\mathbf{M}\mathbf{z})$, where $\mathbf{M} \in \mathbb{R}^{m \times n}$ is some linear transformation 
and $\mathrm{sgn}(\mathbf{v})$ stands for a vector obtained from $\mathbf{v}$ by applying pointwise nonlinear mapping $\mathrm{sgn}:\mathbb{R} \rightarrow \mathbb{R}$ defined 
as follows: $\mathrm{sgn}(x) = +1$ if $x > 0$ and $\mathrm{sgn}(x) = -1$ otherwise.
The angular distance $\theta$ between $\mathbf{x}$ and  $\mathbf{y}$ is estimated by:
$\hat{\theta}^{\mathbf{M}} = \frac{\pi}{2}(1-\phi^{\mathbf{M}}(\mathbf{x})^{\top}\phi^{\mathbf{M}}(\mathbf{y}))$.
We will derive the formula for the $\mathrm{MSE}(\hat{\theta}^{\mathbf{M}}(\mathbf{x},\mathbf{y}))$. One can easily see that the $\mathrm{MSE}$ of the considered in the
statement of the theorem angular kernel on vectors $\mathbf{x}$ and $\mathbf{y}$ can be obtained from this one by multiplying by $\frac{4}{\pi^{2}}$.

Denote by $\mathbf{r}^{i}$ the $i^{th}$ row of $\mathbf{M}$.
Notice first that for any two vectors $\mathbf{x},\mathbf{y} \in \mathbb{R}^{n}$ with angular distance $\theta$, the event 
$E_{i} = \{\mathrm{sgn}((\mathbf{r}^{i})^{\top}\mathbf{x}) \neq \mathrm{sgn}((\mathbf{r}^{i})^{\top}\mathbf{y})\}$
is equivalent to the event $\{\mathbf{r}^{i}_{proj} \in \mathcal{R}\}$, where $\mathbf{r}^{i}_{proj}$ stands for the projection of $\mathbf{r}^{i}$ 
into the $\mathbf{x}-\mathbf{y}$ plane and $\mathcal{R}$ is a union of two cones in the $\mathbf{x}$-$\mathbf{y}$ plane obtained by rotating vectors $\mathbf{x}$ 
and $\mathbf{y}$ by $\frac{\pi}{2}$.
Denote $\mathcal{A}^{i} = \{\mathbf{r}^{i}_{proj} \in \mathcal{R}\}$ for $i=1,...,k$
and $\delta_{i,j} = \mathbb{P}[\mathcal{A}^{i} \cap \mathcal{A}^{j}] - \mathbb{P}[\mathcal{A}^{i}]\mathbb{P}[\mathcal{A}^{j}]$.

For a warmup, let us start our analysis for the standard unstructured Gaussian estimator case. 
It is a well known fact that this is an unbiased estimator of $\theta$. Thus 
\begin{align}
\begin{split}
\mathrm{MSE}(\hat{\theta}^{\mathbf{G}}(\mathbf{x},\mathbf{y})) = Var(\frac{\pi}{2}(1-\phi^{\mathbf{M}}(\mathbf{x})^{\top}\phi^{\mathbf{M}}(\mathbf{y}))) = \frac{\pi^{2}}{4}Var(\phi^{\mathbf{M}}(\mathbf{x})^{\top}\phi^{\mathbf{M}}(\mathbf{y})))\\=
\frac{\pi^{2}}{4} \frac{1}{m^{2}}Var(\sum_{i=1}^{m} X_{i}),
\end{split}
\end{align}
where $X_{i} = \mathrm{sgn}((\mathbf{r}^{i})^{\top}\mathbf{x})\mathrm{sgn}((\mathbf{r}^{i})^{\top}\mathbf{y})$.

Since the rows of $\mathbf{G}$ are independent, we get 
\begin{equation}
Var(\sum_{i=1}^{m}X_{i}) = 
\sum_{i=1}^{m} Var(X_{i}) = \sum_{i=1}^{m} (\mathbb{E}[X_{i}^{2}] - \mathbb{E}[X_{i}]^{2}).
\end{equation}

From the unbiasedness of the estimator, we have: $\mathbb{E}[X_{i}] = (-1) \cdot \frac{\theta}{\pi} + 1 \cdot (1-\frac{\theta}{\pi})$.
Thus we get:
\begin{equation}
\mathrm{MSE}(\hat{\theta}^{\mathbf{G}}(\mathbf{x},\mathbf{y}))= \frac{\pi^{2}}{4}\frac{1}{m^{2}}
\sum_{i=1}^{m}(1 - (1-\frac{2\theta}{\pi})^{2}) = \frac{\theta(\pi - \theta)}{m}.
\end{equation}

Multiplying by $\frac{4}{\pi^{2}}$, we obtain the proof of Lemma \ref{simple_lemma}.

Now let us switch to the general case. We first compute the variance of the general estimator $\mathcal{E}$ using matrices $\mathbf{M}$ 
(note that in this setting we do not assume that the estimator is necessarily unbiased).

By the same analysis as before, we get:
\begin{align}
\begin{split}
Var(\mathcal{E}) = Var(\frac{\pi}{2}(1-\phi(\mathbf{x})^{\top}\phi(\mathbf{y}))) = \frac{\pi^{2}}{4}Var(\phi(\mathbf{x})^{\top}\phi(\mathbf{y})))=
\frac{\pi^{2}}{4} \frac{1}{m^{2}}Var(\sum_{i=1}^{m} X_{i}),
\end{split}
\end{align}
This time however different $X_{i}$s are not uncorrelated.
We get
\begin{align}
\begin{split}
Var(\sum_{i=1}^{m} X_{i}) = \sum_{i=1}^{m}Var(X_{i}) + \sum_{i \neq j} Cov(X_{i},X_{j})=\\
\sum_{i=1}^{m} \mathbb{E}[X_{i}^{2}] - \sum_{i=1}^{m} \mathbb{E}[X_{i}]^{2} + 
\sum_{i \neq j} \mathbb{E}[X_{i}X_{j}] - \sum_{i \neq j} \mathbb{E}[X_{i}]\mathbb{E}[X_{j}]=\\
m + \sum_{i \neq j} \mathbb{E}[X_{i}X_{j}] - \sum_{i,j} \mathbb{E}[X_{i}]\mathbb{E}[X_{j}]
\end{split}
\end{align}

Now, notice that from our previous observations and the definition of $\mathcal{A}^{i}$, we have
\begin{equation}
\mathbb{E}[X_{i}] = -\mathbb{P}[\mathcal{A}^{i}] + \mathbb{P}[\mathcal{A}_{c}^{i}],
\end{equation}
where $\mathcal{A}^{i}_{c}$ stands for the complement of $\mathcal{A}^{i}$.

By the similar analysis, we also get:
\begin{equation}
\mathbb{E}[X_{i}X_{j}] = \mathbb{P}[\mathcal{A}^{i} \cap \mathcal{A}^{j}] + \mathbb{P}[\mathcal{A}^{i}_{c} \cap \mathcal{A}^{j}_{c}] - \mathbb{P}[\mathcal{A}^{i}_{c} \cap \mathcal{A}^{j}] - \mathbb{P}[\mathcal{A}^{i} \cap \mathcal{A}^{j}_{c}]
\end{equation}

Thus we obtain
\begin{align}
\begin{split}
Var(\sum_{i=1}^{m} X_{i}) = m + \sum_{i \neq j} (\mathbb{P}[\mathcal{A}^{i} \cap \mathcal{A}^{j}] + \mathbb{P}[\mathcal{A}^{i}_{c} \cap \mathcal{A}^{j}_{c}] - \mathbb{P}[\mathcal{A}^{i}_{c} \cap \mathcal{A}^{j}] - \mathbb{P}[\mathcal{A}^{i} \cap \mathcal{A}^{j}_{c}]\\-
(\mathbb{P}[\mathcal{A}_{c}^{i}]-\mathbb{P}[\mathcal{A}^{i}])
(\mathbb{P}[\mathcal{A}_{c}^{j}]-\mathbb{P}[\mathcal{A}^{j}]))\\-
\sum_{i}(\mathbb{P}[\mathcal{A}_{c}^{i}]-\mathbb{P}[\mathcal{A}^{i}])^{2}=
m - \sum_{i}(1 -2\mathbb{P}[\mathcal{A}^{i}])^{2} \\+ \sum_{i \neq j} (\mathbb{P}[\mathcal{A}^{i} \cap \mathcal{A}^{j}] + \mathbb{P}[\mathcal{A}^{i}_{c} \cap \mathcal{A}^{j}_{c}] - \mathbb{P}[\mathcal{A}^{i}_{c} \cap \mathcal{A}^{j}] - \mathbb{P}[\mathcal{A}^{i} \cap \mathcal{A}^{j}_{c}]+\\ \mathbb{P}[\mathcal{A}_{c}^{i}]\mathbb{P}[\mathcal{A}^{j}]
+\mathbb{P}[\mathcal{A}^{i}]\mathbb{P}[\mathcal{A}_{c}^{j}]-
\mathbb{P}[\mathcal{A}_{c}^{i}]\mathbb{P}[\mathcal{A}_{c}^{j}]-
\mathbb{P}[\mathcal{A}^{i}]\mathbb{P}[\mathcal{A}^{j}])\\=
m-\sum_{i}(1 -2\mathbb{P}[\mathcal{A}^{i}])^{2}+\sum_{i \neq j} (\delta_{1}(i,j)+\delta_{2}(i,j)+\delta_{3}(i,j) + \delta_{4}(i,j)),
\end{split}
\end{align}
where
\begin{itemize}
	\item $\delta_{1}(i,j) = \mathbb{P}[\mathcal{A}^{i} \cap \mathcal{A}^{j}] - \mathbb{P}[\mathcal{A}^{i}]\mathbb{P}[\mathcal{A}^{j}]$, 
	\item $\delta_{2}(i,j) = \mathbb{P}[\mathcal{A}_{c}^{i} \cap \mathcal{A}_{c}^{j}] - \mathbb{P}[\mathcal{A}_{c}^{i}]\mathbb{P}[\mathcal{A}_{c}^{j}]$,
	\item $\delta_{3}(i,j) = \mathbb{P}[\mathcal{A}_{c}^{i}]\mathbb{P}[\mathcal{A}^{j}] - \mathbb{P}[\mathcal{A}_{c}^{i} \cap \mathcal{A}^{j}]$, 
	\item $\delta_{4}(i,j) = \mathbb{P}[\mathcal{A}^{i}]\mathbb{P}[\mathcal{A}_{c}^{j}] - \mathbb{P}[\mathcal{A}^{i} \cap \mathcal{A}_{c}^{j}]$.
\end{itemize}

Now note that
\begin{align}
\begin{split}
-\delta_{4}(i,j) = \mathbb{P}[\mathcal{A}^{i}] - \mathbb{P}[\mathcal{A}^{i} \cap \mathcal{A}^{j}] - \mathbb{P}[\mathcal{A}^{i}]\mathbb{P}[\mathcal{A}^{j}_{c}]\\=
\mathbb{P}[\mathcal{A}^{i}] - \mathbb{P}[\mathcal{A}^{i}](1-\mathbb{P}[\mathcal{A}^{j}])-\mathbb{P}[\mathcal{A}^{i} \cap \mathcal{A}^{j}] \\= \mathbb{P}[\mathcal{A}^{i}]\mathbb{P}[\mathcal{A}^{j}]-\mathbb{P}[\mathcal{A}^{i} \cap \mathcal{A}^{j}]
=-\delta_{1}(i,j)
\end{split}
\end{align}
Thus we have $\delta_{4}(i,j)=\delta_{1}(i,j)$.
Similarly, $\delta_{3}(i,j) = \delta_{1}(i,j)$.
Notice also that
\begin{align}
\begin{split}
-\delta_{2}(i,j) = (1-\mathbb{P}[\mathcal{A}^{i}])(1-\mathbb{P}[A^{j}])-
(\mathbb{P}[\mathcal{A}^{i}_{c}]-\mathbb{P}[\mathcal{A}^{i}_{c} \cap \mathcal{A}^{j}])\\=
1-\mathbb{P}[\mathcal{A}^{i}]-\mathbb{P}[\mathcal{A}^{j}]+\mathbb{P}[\mathcal{A}^{i}]
\mathbb{P}[\mathcal{A}^{j}]-1+
\mathbb{P}[\mathcal{A}^{i}]+\mathbb{P}[\mathcal{A}^{i}_{c} \cap \mathcal{A}^{j}] \\= 
\mathbb{P}[\mathcal{A}^{i}]\mathbb{P}[\mathcal{A}^{j}]-\mathbb{P}[\mathcal{A}^{i} \cap \mathcal{A}^{j}]
=-\delta_{1}(i,j),
\end{split}
\end{align}
therefore $\delta_{2}(i,j)=\delta_{1}(i,j)$.

Thus, if we denote $\delta_{i,j} = \delta_{1}(i,j) = \mathbb{P}[\mathcal{A}^{i} \cap \mathcal{A}^{j}] - \mathbb{P}[\mathcal{A}^{i}]\mathbb{P}[\mathcal{A}^{j}]$, then we get
\begin{align}
\begin{split} 
Var(\sum_{i=1}^{m}X_{i}) = m - \sum_{i}(1-2\mathbb{P}[A^{i}])^{2}+4\sum_{i \neq j}\delta_{i,j}.
\end{split}
\end{align}

Thus we obtain

\begin{equation}
Var(\mathcal{E}) = \frac{\pi^{2}}{4m^{2}}[m - \sum_{i}(1-2\mathbb{P}[A^{i}])^{2}+4\sum_{i \neq j}\delta_{i,j}].
\end{equation}

Note that $Var(\mathcal{E}) = \mathbb{E}[(\mathcal{E}-\mathbb{E}[\mathcal{E}])^{2}]$.
We have:
\begin{align}
\begin{split}
\mathrm{MSE}(\hat{\theta}^{\mathbf{M}}(\mathbf{x},\mathbf{y})) = 
\mathbb{E}[(\mathcal{E}-\theta)^{2}] = \mathbb{E}[(\mathcal{E}-\mathbb{E}[\mathcal{E}])^{2}] + \mathbb{E}[(\mathcal{E}-\theta)^{2}] - \mathbb{E}[(\mathcal{E}-\mathbb{E}[\mathcal{E}])^{2}] \\=
Var(\mathcal{E}) + \mathbb{E}[(\mathcal{E}-\theta)^{2}-(\mathcal{E}-\mathbb{E}[\mathcal{E}])^{2}] \\= Var(\mathcal{E}) + (\mathbb{E}[\mathcal{E}]-\theta)^{2}
\end{split}
\end{align}

Notice that $\mathcal{E} = \frac{\pi}{2}(1-\frac{1}{m}\sum_{i=1}^{m}X_{i})$.
Thus we get:
\begin{align}
\begin{split}
\mathrm{MSE}(\hat{\theta}^{\mathbf{M}}(\mathbf{x},\mathbf{y})) = \frac{\pi^{2}}{4m^{2}}[m - \sum_{i}(1-2\mathbb{P}[A^{i}])^{2}+4\sum_{i \neq j}\delta_{i,j}] + \frac{\pi^{2}}{m^{2}}\sum_{i}(\mathbb{P}(\mathcal{A}^{i})-\frac{\theta}{\pi})^{2}.
\end{split}
\end{align}

Now it remains to multiply the expression above by $\frac{4}{\pi^{2}}$ and that completes the proof.
\end{proof}

\begin{remark}
	Notice that if $\mathbb{P}(\mathcal{A}^{i}) = \frac{\theta}{\pi}$ 
	(this is the case for the standard unstructured estimator as well as for the considered by us estimator using orthogonalized version of Gaussian vectors) 
	and if rows of matrix $\mathbf{M}$ are independent then the general formula for $\mathrm{MSE}$ for the estimator of an angle reduces to $\frac{(\pi-\theta)\theta}{m}$. If the first property is 
	satisfied but the rows are not necessarily independent (as it is the case for the estimator using orthogonalized version of Gaussian vectors) then whether the $\mathrm{MSE}$ 
	is larger or smaller than for the standard unstructured case is determined by the sign of the sum $\sum_{i \neq j} \delta_{i,j}$. 
	For the estimator using orthogonalized version of Gaussian vectors we have already showed that for every $i \neq j$ we have: $\delta_{i,j} > 0$ thus we obtain 
	estimator with smaller $\mathrm{MSE}$. If $\mathbf{M}$ is a product of blocks $\mathbf{HD}$ then we both have: an estimator with dependent rows and with bias. 
	In that case it is also easy to see that
	$\mathbb{P}(\mathcal{A}^{i})$ does not depend on the choice of $i$. Thus there exists some $\epsilon$ such that $\epsilon = \mathbb{P}(\mathcal{A}^{i})-\frac{\theta}{\pi}$.
	Thus the estimator based on the $\mathbf{HD}$ blocks gives smaller $\mathrm{MSE}$ iff:
	$$\sum_{i \neq j} \delta_{i,j} + m\epsilon^{2} < 0.$$
\end{remark}

\section{Further comparison of variants of OJLT based on \texorpdfstring{$\mathbf{SD}$-product}{SD-product} matrices}\label{sec:comparison}
In this section we give details of further experiments complementing the theoretical results of the main paper. In particular, we explore the various parameters associated with the $\mathbf{SD}$-product matrices introduced in \S \ref{sec:model}. In all cases, as in the experiments of \S \ref{sec:experiments}, we take the structured matrix $\mathbf{S}$ to be the normalized Hadamard matrix $\mathbf{H}$. All experiments presented in this section measure the MSE of the OJLT inner product estimator for two randomly selected data points in the $\texttt{g50c}$ data set. The MSE figures are estimated on the basis of $1,000$ repetitions. All results are displayed in Figure \ref{fig:hd-comparisons}.

\begin{figure}[h]
	\centering
	\subfigure[Comparison of estimators based on $\mathbf{S}$-Rademacher matrices with a varying number of $\mathbf{SD}$ matrix blocks, using the with replacement 
	sub-sampling strategy.]{
		\includegraphics[keepaspectratio, width=0.4\textwidth]{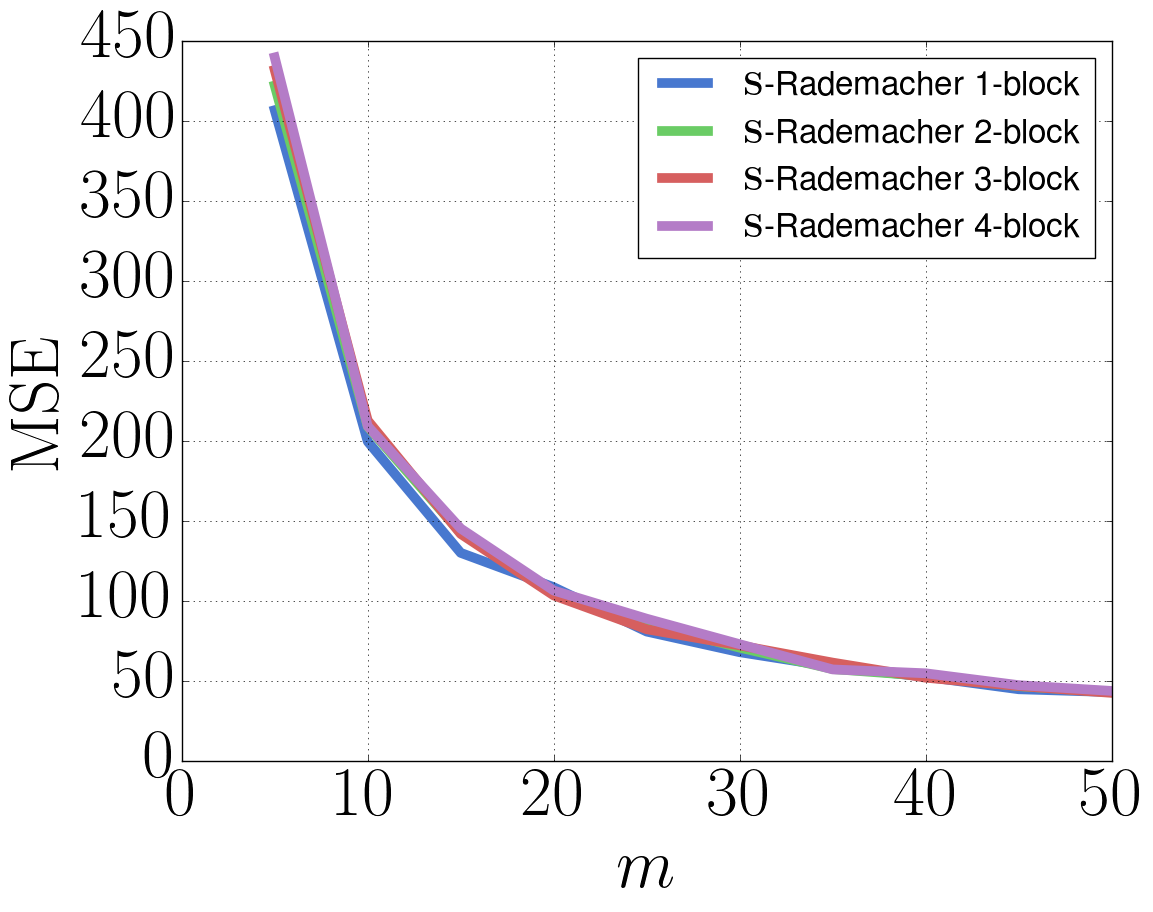}}
	
	\subfigure[Comparison of estimators based on $\mathbf{S}$-Rademacher matrices with a varying number of $\mathbf{SD}$ matrix blocks, using the 
	sub-sampling strategy without replacement.]{
		\includegraphics[keepaspectratio, width=0.4\textwidth]{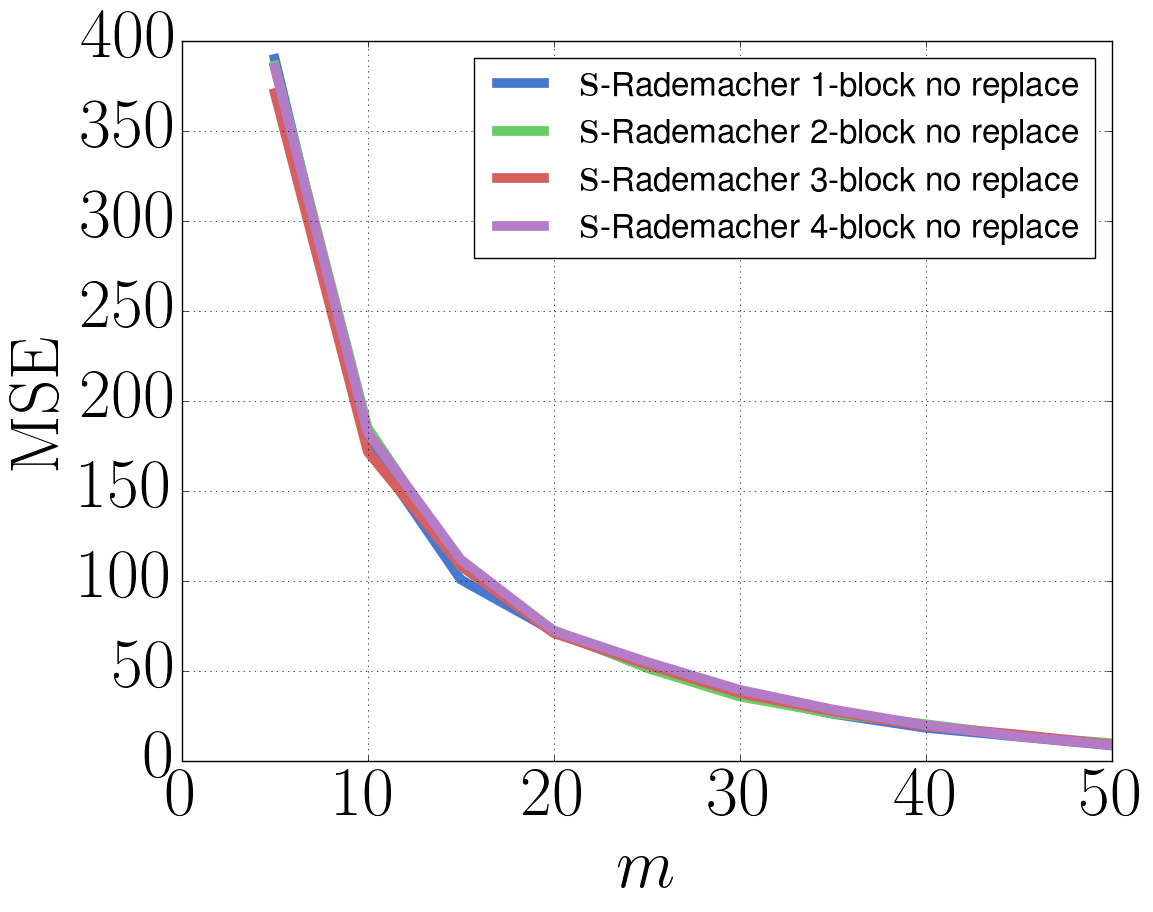}}\qquad
	\subfigure[Comparison of the use of $\mathbf{M}^{(3)}_{\mathbf{S}\mathcal{R}}$, $\mathbf{M}^{(3)}_{\mathbf{S}\mathcal{H}}$, and $\mathbf{M}^{(3)}_{\mathbf{S}\mathcal{U}}$ (introduced in \S \ref{sec:proof-full-complex}) for dimensionality reduction. All use 
	sub-sampling without replacement.
	The curves corresponding to the latter two random matrices are indistinguishable.]{
		\includegraphics[keepaspectratio, width=0.4\textwidth]{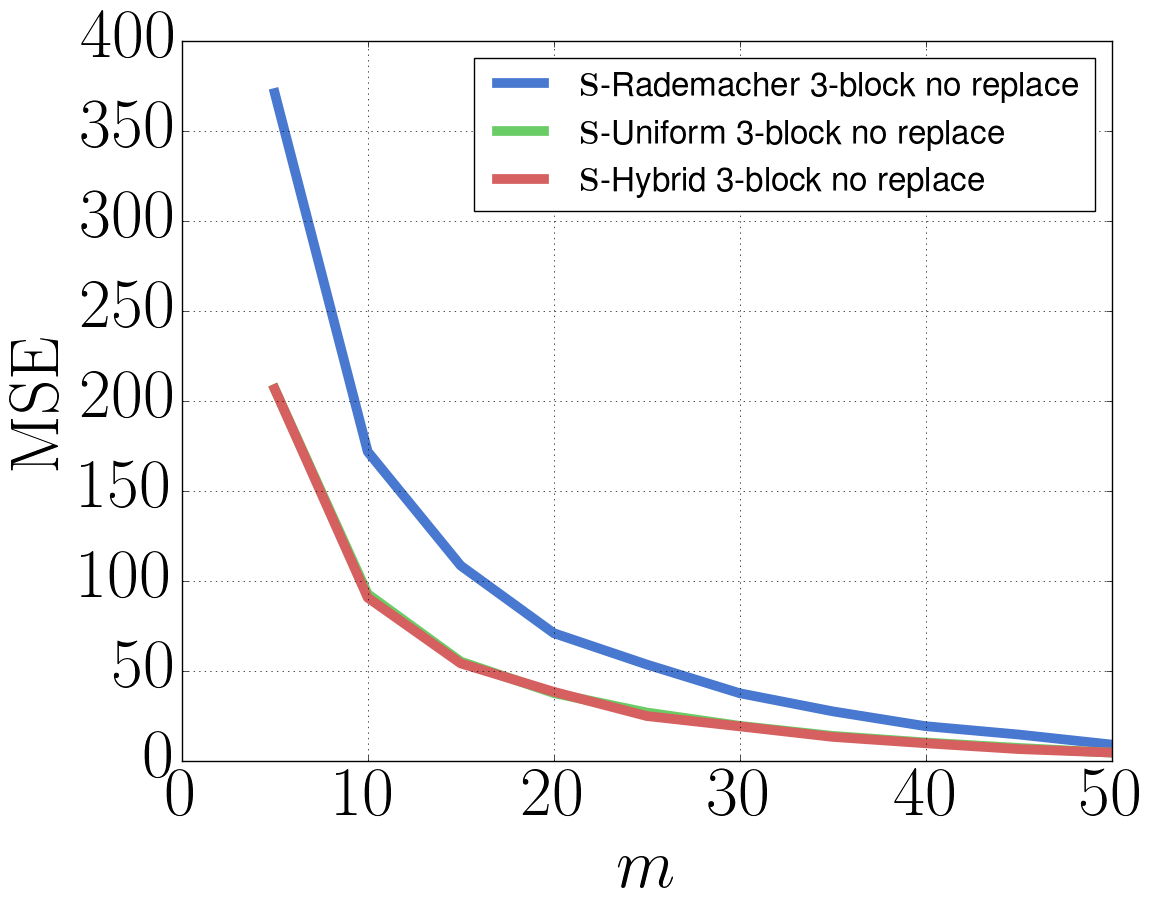}}
	\caption{Results of experiments comparing OJLTs for a variety of $\mathbf{SD}$-matrices.}
	\label{fig:hd-comparisons}
\end{figure}	
	
\end{document}